\documentclass[11pt]{article}
\usepackage[a4paper,top=3cm,bottom=3cm,left=3cm,right=3cm,marginparwidth=1.75cm]{geometry}
\usepackage{amsmath}
\usepackage{mathrsfs}
\usepackage{graphicx} 
\usepackage{subfigure}
\usepackage{amssymb}
\usepackage{cite}  
\usepackage{appendix}
\usepackage{enumitem}
\usepackage{bm}
\usepackage{natbib}
\usepackage[dvipsnames,table,dvipsnames*, svgnames*]{xcolor}
\usepackage{algpseudocode,float,algorithm}
\usepackage{hyperref}
\hypersetup{
    colorlinks=true,
    linkcolor=blue,
    filecolor=blue,
    urlcolor=blue,
    citecolor=blue,
}

\newtheorem{theorem}{Theorem}
\newtheorem{lemma}{Lemma}
\newtheorem{proof}{Proof}
\newtheorem{proposition}{Proposition}
\newtheorem{assumption}{Assumption}
\newtheorem{remark}{Remark}
\newtheorem{definition}{Definition}
\newtheorem{example}{Example}
\newtheorem{corollary}{Corollary}

\DeclareMathOperator{\Cov}{\mathrm{Cov}}
\newcommand{\argmin}{\mathop{\rm arg\min}}

\newcommand{\Acal}{{\mathcal{A}}}
\newcommand{\Zcal}{\mathcal{Z}}
\newcommand{\Prob}{\mathbb{P}}
\newcommand{\D}{{\mathcal{D}}}

\newcommand{\R}{{\mathbb{R}}}
\newcommand{\E}{{\mathbb{E}}}

\newcommand{\indep}{\rotatebox[origin=c]{90}{$\models$}}

\usepackage{authblk}

\title{Unified Inference Framework for Single and Multi-Player Performative Prediction: Method and Asymptotic Optimality}

 \author[1]{Zhixian Zhang}
 \author[2]{Xiaotian Hou}
 \author[1]{Linjun Zhang}

 \affil[1]{Rutgers University}
 \affil[2]{University of Pennsylvania}

\begin{document}

\maketitle

\begin{abstract}

Performative prediction characterizes environments where predictive models alter the very data distributions they aim to forecast, triggering complex feedback loops. While prior research treats single-agent and multi-agent performativity as distinct phenomena, this paper introduces a unified statistical inference framework that bridges these contexts, treating the former as a special case of the latter. Our contribution is two-fold. First, we put forward the Repeated Risk Minimization (RRM) procedure for estimating the performative stability, and establish a rigorous inferential theory for admitting its asymptotic normality and confirming its asymptotic efficiency. Second, for the performative optimality, we introduce a novel two-step plug-in estimator that integrates the idea of Recalibrated Prediction Powered Inference (RePPI) with Importance Sampling, and further provide formal derivations for the Central Limit Theorems of both the underlying distributional parameters and the plug-in results. The theoretical analysis demonstrates that our estimator achieves the semiparametric efficiency bound and maintains robustness under mild distributional misspecification. This work provides a principled toolkit for reliable estimation and decision-making in dynamic, performative environments.
\end{abstract}

\section{Introduction}
\textit{Performative prediction} refers to a class of predictive modeling problems where the act of prediction itself influences the distribution of the data 
it uses to predict \cite{perdomo2020performative}. Unlike traditional supervised learning settings where data distributions remain fixed, performative predictions induce distributional shifts through their deployment, particularly when supporting consequential decisions, including loan approvals \cite{bartlett2022consumer}, criminal sentencing \cite{courtland2018bias}, and public policy design \cite{lum2016predict}. \

Consider a bank that builds a predictive model for loan approval. If the model predicts that an applicant has a high risk of default, the bank may respond by offering a higher interest rate. This decision, however, induces an inverse behavioral response from applicants: in order to qualify for better loan terms, they may actively modify their financial behaviors to meet the model’s approval criteria. Consequently, the bank’s predictive model becomes miscalibrated with respect to the outcomes that arise once its decisions are implemented, as the related distribution depends on the current model.

Suppose a prediction model $f_\theta$ is parametrized by $\theta$. The primary goal of performative prediction is to find a prediction model that minimizes the performative risk function, which leads to the definition of the performatively optimal point. We have its mathematical description as:
\begin{equation*}
    \theta_{PO} = \arg\min _{\theta \in \Theta}\mathbb{E}_{Z \sim \mathcal{D}(\theta)} \ell(\theta,Z),
\end{equation*}
where $\ell(\theta,Z)$ is the loss function, and $Z=(X,Y)$ is a input-output pair. The underlying distribution $\mathcal{D}(\theta)$ is not static but rather a distribution mapping that depends on the model parameter $\theta \in \Theta$, which, in the loan approval example, represents the interest rate policy offered by the bank, through which applicant behavior is altered.
By incorporating feedback from current predictions, it can demonstrate distributional shifts in future observations. However, since the distribution $\D(\theta)$ is typically unknown, the performative risk function is difficult to calculate directly, which makes the minimization problem intractable.

Besides the performative optimality, we also have the concept of performative stability. While performative optimality corresponds to the model that minimizes the performative risk over all possible predictors, performative stability refers to a fixed point where the prediction model $f_\theta$, given as a basis for predictions, is also simultaneously optimal for the very distribution that its deployment induces. The performatively stable point defined as the solution to the following fixed-point equation: 
\begin{equation*}
    \theta_{PS} = \arg\min _{\theta \in \Theta}\mathbb{E}_{Z \sim \mathcal{D}(\theta_{PS})} \ell(\theta,Z),
\end{equation*}
where the performatively stable model $f_{\theta_{PS}}$ minimizes the risk with respect to the distribution $\D(\theta_{PS})$, which itself arises from the model $f_{\theta_{PS}}$. Since this model already accounts for the distributional shift caused by deployment, it eliminates the need for further retraining.

In real-world applications, learners are often deployed alongside others, either in cooperation or competition. Considering several banks simultaneously building loan approval models, each bank trains its own model to predict defaults, yet these predictions influence future applicant distributions. For instance, if one bank tightens its approval threshold, more applicants may turn to other banks, shifting the overall distribution. Thus, each bank’s predictive strategy shapes not only its own data environment but also those of others. Such interactions naturally evolve toward an equilibrium, giving rise to the concept of \textit{multiplayer performative prediction}. Similar to the performative optimality, the Nash equilibria aim to find a set of prediction models that for each player $i$, its performative risk function based on other players is the minimum. Analogous to performative optimality, the Nash equilibrium corresponds to a set of prediction models where, for each player $i$, the performative risk conditioned on the strategies of other players attains its minimum (the exact definition is given in Section~\ref{subsec:PP setting}). Besides, the performative stable equilibria characterize situations in which the prediction model employed for decision-making is also optimal with respect to the distribution it induces. Therefore, each player $i$ has no intention to deviate from the stable equilibria while it only has access to the distribution generated by it.

While prior work has primarily focused on developing algorithms to identify performative stability and optimality, our work focuses on constructing a statistical inference framework for performative prediction in both the single-player and multi-player settings, offering critical capabilities such as efficiency analysis, uncertainty quantification, and decision making. For instance, we can construct the confidence interval for the estimates to quantify the variability arising from both data randomness and the distributional feedback induced by model deployment. This enables rigorous assessment of the model’s stability and reliability when its predictions influence the underlying data-generating process. More importantly, we improve the performance of inference under performativity to its optimal level among all estimation procedures for finding the performative stability and optimality in both settings. The efficiency of our estimators ensures that the subsequent analysis will be as accurate and reliable as possible. By integrating our inferential framework into performative predictions, we go beyond merely identifying target models, as we provide formal guarantees of their reliability and interoperability at the highest achievable level.

\subsection{Overview of our results}

{Throughout this work, the term ``performative prediction'' specifically refers to the single-player setting, where only one decision maker interacts with the environment.}
As the multiplayer performative prediction is a general case of performative prediction, and the methods for finding the stable and Nash equilibria described in \cite{narang2023multiplayer} are built upon their single-player counterparts, we initiate a systematic study of statistical inference for performatively stable and Nash equilibria in the multiplayer setting, and it naturally encompasses the single-agent case. For both equilibria, we develop feasible estimation procedures and establish their asymptotic normality and efficiency. The resulting framework offers a unified statistical inference approach applicable to both classical performative prediction and its multiplayer generalizations.

\subsubsection{Performative Stability}
The estimation procedure for the performative stable equilibria builds on the model update scheme known as \textit{Repeated Retraining (RR)}, introduced in the work \cite{narang2023multiplayer}, wherein the model parameter $\theta_t$ is updated iteratively by minimizing the risk function set evaluated on the distribution induced by the previous model. Inspired by the estimation method for the performative stability based on the repeated risk minimization in the work \cite{li2025statisticalinferenceperformativity} and the structure of RR, we first construct an estimation procedure for $\theta_t$ by replacing the risk function with the empirical risk function in the RR scheme at each iteration for every player. We refer to this method as \textit{Empirical Repeated Retraining (ERR)}. We show that under certain conditions on the underlying distribution map $\D(\theta)$ and the loss functions for each player $i$, the ERR-based estimators $\hat \theta_t$ follow the central limit theorem for $\theta_t$, that is, the deviation $\sqrt{N}(\hat\theta_t - \theta_t)$ at every time $t \in \mathbb{T}$ is asymptotically normal with cumulated asymptotic covariance, which is related to the covariance at all previous iterations.

To establish the optimality of this method, we derive a lower bound on the asymptotic covariance for any estimation procedure targeting performative stability along a sequence of small perturbations of the original performative problem. We then show that, under suitable regularity conditions, our ERR-based estimator attains this bound, demonstrating its asymptotic efficiency.

\begin{theorem}[Stability, informal]
\label{thm:sum of stable, normality}
    Suppose that for each player $i$, the distribution map is $\epsilon_i$-Lipschitz in Wasserstein-1 distance, the loss function is $\beta_i$-jointly smooth, and the gradient function is $\alpha_i$-strongly monotone on $\theta^i$, and locally Lipschitz on $\theta^i$. Suppose $\sum_{k=1}^m (\frac{\beta_i\epsilon_i}{\alpha})^2 < 1$ holds and $\{\theta_t\}_{t=1}$ lie in the interior of $\Theta$, so the estimators $\hat{\theta}_t$ generated from the ERR method follow the central limit theorem for $\theta_t$ at each iteration $t$ with cumulated asymptotic covariance:
    $$
    \sqrt{N}(\hat \theta_t - \theta_t) \xrightarrow{d} N(0, \Sigma_t),
    $$ 
    where the covariance $\Sigma_t$ is related to the covariance at all previous iterations.

    Suppose $\theta_{t-1} \neq \theta_{PS}$ and its ERR-based estimator $\hat \theta_t$ satisfies the condition of regularity, then the ERR-based estimator is semiparametrically efficient.
\end{theorem}

Intuitively, the statistical inference results derived in the multiplayer setting can be seamlessly reduced to their single-player counterparts, reflecting that the latter can be viewed as a special case of our more general framework. This unifying perspective highlights the flexibility of our approach and its capacity to encompass both individual and interactive performative learning scenarios. In the single-player setting, our ERR estimation method will reduce to the \textit{Repeated Empirical Risk Minimization (RERM)} method, simply based on the repeated risk minimization. Under the single-player version of the required assumptions, the deviation still converges to a normal distribution with a covariance related to the previous ones in distribution. Similarly, we establish the local asymptotic optimality of our RERM-based estimator, showing that it attains the semiparametric efficiency bound.

\subsubsection{Performative Optimality}

\textit{Plug-in performative optimization} is a useful technique for finding the performative optimal point for the single-player performative prediction introduced in the work \cite{lin2023plug}, in which the optimum based on a $\beta$-misspecified yet known distribution map can help with learning the true performative optimal point $\theta_{PO}$, with a bounded error between their performative risk. In this paper, we extend the algorithm to the more general multiplayer setting. In this context, we first estimate the distributional parameter $\beta_i$ for each player $i$, and then compute the plug-in optimum based on the distribution map induced by the estimator $\hat{\beta}_i$.

Beginning with fitting the distributional parameter, rather than relying only on empirical risk minimization, as the previous work \cite{lin2023plug} did, we establish the estimation for $\beta_i$ by a three-fold cross-fitting procedure based on the recalibrated prediction-powered inference (RePPI) method, as it ensures the efficiency under certain conditions. We demonstrate the asymptotic normality of our recalibrated estimation $\hat \beta_i$, and further prove its efficiency by identifying the efficiency influence functions in this setting. Based on the $\hat \beta_i$, we can use the empirical plug-in optimization to generate the plug-in estimator for Nash equilibria. However, since the fitted parametric model is still related to $\theta$, drawing samples directly is still hard here. To solve this problem, we combine the plug-in optimization with importance sampling to enable the collection of samples. Let $\hat{\theta}_{PO}^{\hat{\beta}}$ denote the plug-in estimator based on the distributional estimator, and we similarly establish the central limit theorem, that is, the asymptotic normality of the deviation with asymptotic covariance related to that of the distributional estimator, and further prove that it attains the lower bound. 

\begin{theorem}[Optimality, informal]
\label{thm:informal original optimal beta}
    Suppose that for each player, the distribution atlas is smooth and misspecified in total-variation distance, and the loss functions follow the conditions of local Lipschitzness, differentiability, and convexity. Suppose the solution map to the plug-in optimization is differentiable in $\beta$ at $\beta^*$. Denote $s_i^*(\theta)= \E\big[\nabla_{\beta_i} r_i(\theta,Z^i;\beta_i^*)|\theta\big]$. If the estimation for $s_i^*(\theta)$ is consistent at each fold and sample sizes satisfy certain conditions, we have:
    $$
    \sqrt{N}(\hat{\beta}_i - \beta_i^*) \xrightarrow{P} N(0,\Sigma_{\beta_i}),
    $$ 
    $$
    \sqrt{N}(\hat{\theta}^{\hat{\beta}}_{PO} - \theta^{\beta^*}_{PO})  \xrightarrow{d}  N(0,\Sigma_\theta),
    $$
    where the covariance $\Sigma_\theta$ is related to the covariance of the estimator of the distributional parameter $\Sigma_{\beta_i}$
    Under certain conditions of regularity, $\Sigma_{\beta_i}$ and $\Sigma_\theta$ reach the lower bounds for all deviations for $\beta$ and $\theta_{PO}^\beta$.
\end{theorem}

Analogous to the case of performative stability, the inference framework developed for multiplayer performative prediction naturally reduces to its single-player counterpart under the single version of corresponding conditions, built upon the same ideas.

\subsection{Related Work}

\paragraph{Performative prediction} 
The performative prediction was first introduced in \cite{perdomo2020performative}, which focuses on the concepts of performative stability and performative optimality, and is further refined by a line of works \cite{mendler2020stochastic, mofakhami2023performative, miller2021outside, izzo2021learn, drusvyatskiy2020stochasticoptimizationdecisiondependentdistributions, jagadeesan2022regretminimizationperformativefeedback}. 
Some studies proposed various algorithmic variants to achieve performatively stable points. For example, \cite{perdomo2020performative} proposed two algorithms, RRM and RGD, for finding stable points at the population level, while \cite{mendler2020stochastic} developed two variants of the stochastic gradient method for performative predictions based on the RGD algorithm. Moreover, \cite{drusvyatskiy2020stochasticoptimizationdecisiondependentdistributions} demonstrated that many gradient-based algorithms in the decision-dependent setting can be viewed as standard algorithms on a static problem, with only a vanishing bias. Though \cite{perdomo2020performative} proves that under certain conditions, the performatively stable point is close to the performatively optimal point, stability can be far from optimality when evaluated in terms of the performative risk.  Therefore, numerous works \cite{miller2021outside, izzo2021learn, jagadeesan2022regretminimizationperformativefeedback, lin2023plug} focusing on obtaining performative optimality have emerged. For example, \cite{miller2021outside} proposes a two-stage algorithm for optimizing the performative risk and proves its efficiency in location families. \cite{izzo2021learn} introduces the PerfGD algorithm for computing performatively optimal points and proves its convergence, and \cite{lin2023plug} presents a distributional-plug-in algorithm to effectively approximate the true optimality.

All the works mentioned above focus on the single-player performative setting, in which the interaction exists solely between a single model and agents that respond to its actions. However, performative prediction can also involve an interconnected set of models, where each is implemented together with others. This scenario was formalized in \cite{narang2023multiplayer} as multiplayer performative prediction. The study defines performatively stable equilibria and Nash equilibria, which are aligned with performative stability and performative optimality in the single-player setting, and proposes several algorithms for finding them based on the algorithms designed for the single-player setting.

Most existing works focus on finding the two target equilibria, while only a few investigate statistical inference under performativity. In particular, \cite{li2025statisticalinferenceperformativity} introduces a framework for statistical inference at the performatively stable point, based on the RRM algorithm from \cite{perdomo2020performative}, whereas \cite{cutler2024stochasticapproximationdecisiondependentdistributions} proposes a more general framework for all stable equilibria in decision-dependent settings based on the stochastic gradient-based algorithms.

\paragraph{Recalibrated Prediction-Powered Inference} 
RePPI is developed in the work \cite{ji2025predictions}, mainly based on the concepts of surrogate outcome models and prediction-powered inference.
Surrogate outcomes, also known as auxiliary or proxy variables, are frequently collected to facilitate faster data analysis and enhance statistical efficiency, and surrogate outcome models are widely applied in the application field of clinical trials \cite{prentice1989surrogate,wittes1989surrogate,pepe1992inference,post2010analysis, fleming1994surrogate} and marketing and business \cite{chen2005measurement, athey2019surrogate, kallus2025role,zhang2023evaluating}. The form of the loss function in the optimal surrogate model is given in \cite{robins1994estimation}, leading to the property of efficiency of surrogate outcome models, and therefore leading to the efficiency of RePPI. In all the studies mentioned above, surrogates are still required to be collected by the researcher, though typically at a lower cost than the outcome of primary interest. Also, surrogates may be subject to missingness, arising from survey non-response, dropout, or unexpected measurement failures, which leads to various problems \cite{prentice1989surrogate, frangakis2002principal, chen2007criteria}.

Prediction-Powered Inference (PPI) \cite{angelopoulos2023prediction} is a semi-supervised statistical framework related to inference with missing data and semi-supervised inference \cite{azriel2022semi,chernozhukov2018double,robins1995semiparametric,zhang2019semi,song2024general,robins1994estimation,rubin1976inference}. Unlike the surrogate outcomes which are collected manually by the researcher, PPI leverages black-box machine learning predictions as proxy variables to enhance the efficiency and validity of classical inferential procedures. In this framework, the researcher has access to a small labeled dataset, a large unlabeled dataset and its machine learning predictions generated by a pre-trained model, and constructs a bias-corrected estimator for target parameters by decomposing the estimation error into two components: a model-based prediction term and a debiasing term derived from gold-standard measurements. There are plenty of extensions of PPI, including PPI++ \cite{angelopoulos2023ppi++}, Stratified PPI \cite{fisch2024stratified}, Cross PPI \cite{zrnic2024cross}, etc. \cite{gan2023prediction,miao2023assumption,gronsbell2024another}

The work \cite{ji2025predictions} connects the Surrogate Outcome Model and Prediction-Powered Inference to construct the Recalibrated Prediction-Powered Inference (RePPI), which generates more efficient estimators than existing PPI proposals. To make the procedure practical, they present a three-fold cross-fitting algorithm for RePPI, which allows learning the intractable integral by flexible machine learning methods. Specifically, the estimator will achieve the smallest asymptotic variance if the integral is estimated consistently.

\subsection{Notation and Definitions}
We clarify the notations we use in this paper. 
Throughout, we denote a standard $d$-dimensional Euclidean space as $\R^d$, with inner product $\langle x,y \rangle = x^Ty$ and induced norm $\|x\| = \sqrt{\langle x,x \rangle}$. For any set $\Theta \subset \R^d$, the projection of a point $x \in \R^d$ onto the set is denoted by $\Pi_\Theta(x) = \argmin_{\theta \in \Theta}\|x - \theta\|$, meaning the nearest points of $\Theta$ to $x$. The normal cone $\mathcal{N}_\Theta(x)$ to a convex set $\Theta$ at $\theta \in \Theta$ is the set $\mathcal{N}_\Theta(x) = \{v \in \R^d \mid \langle v, \theta-x \rangle \leq 0 \text{ for all } \theta \in \Theta\}$.

\section{Preliminaries}

\subsection{Problem Setup}
\label{subsec:PP setting}

\paragraph{Multi-player Performative Prediction}
Suppose we have $m$ players in our prediction, and the model parameter for every player $i$ is denoted as $\theta^i$. Fix an index set $[m] = \{1,...,m\}$, the dimension of the model parameter $\theta^i$ for each player $i$ as $d_i$, and let $d = \sum_{i=1}^m d_i$. Let $\Theta_i \subset \mathbb{R}^{d_i}$ denote the model parameter space for each player $i$, and $\Zcal_i$ denote the variate space, both of which are convex and closed. 
The parameter vector $\theta \in \R^d$ at the population level is decomposed by $\theta^i \in \R^d$ with $\theta = (\theta^1, ..., \theta^m)$. For each player $i$, we separate the parameter vector as $\theta = (\theta^i, \theta^{-i})$, where $\theta^{-i}$ denotes the parameter vector of all other players. 
According to the definition of multiplayer performative prediction \cite{narang2023multiplayer}, we have a collection of functions $\ell_i: \R^{d_i} \rightarrow \R$ for each player $i$, and they seek to solve the decision-dependent optimization problems interconnected with others:
$$
\min_{\theta^i \in \Theta_i} \mathcal{L}_i(\theta^i, \theta^{-i}) = \min_{\theta^i \in \Theta_i} \E_{Z_i \sim \D_i(\theta)} \ell_i(\theta^i, \theta^{-i}, Z^i),
$$
where the random variable $Z^i$ for each player $i$ is governed by the distribution map $\D_i(\theta)$, which is related to all the players as $\theta = (\theta^1,...,\theta^m)$.
In our work, the \textit{Nash equilibrium} is defined as a vector $\theta_{PO} \in \R^d$ if the following condition holds:
\begin{equation*}
    \theta_{PO}^i = \argmin_{\theta^i \in \Theta_i} \mathcal{L}_i(\theta^i, \theta_{PO}^{-i})  = \argmin_{\theta^i \in \Theta_i} \E_{Z^i \sim \D_i(\theta^i, \theta^{-i}_{PO})} \ell_i(\theta^i, \theta^{-i}_{PO}, Z^i), \quad \forall i \in [m],
\end{equation*}
where each player $i$ selects its model parameter $\theta^i_{PO}$ to minimize its own performative risk, assuming that all other players simultaneously adopt their respective best-response strategies under the same rationale.

Denote $\D = \D_1 \times..\times \D_m$. We can rewrite the prediction problems above into the generalized first-order condition form.
Denote the gradient of the function $\ell(\cdot)$ with respect to $\theta^i$ as $\nabla_i \ell(\cdot)$, then we have  a vector of gradient functions as follows:
\begin{equation*}
\begin{split}
    G(\theta,Z) & = (G_1(\theta,Z^1),...,G_m(\theta,Z^m)) \\
    & =(\nabla_1 \ell_1(\theta,Z^1), ..., \nabla_m \ell_m(\theta,Z^m)).
\end{split}
\end{equation*}
For simplicity, we refer to $G(\theta,Z)$ as the Jacobian matrix throughout this paper. Define the joint space as $\Theta = \Theta_1 \times ... \times \Theta_m$ and $\Zcal = \Zcal_1 \times ... \times \Zcal_m$, we have Nash equilibria characterized by the generalized first-order condition:
\begin{equation}
\label{equ:VI}
    0 \in G(\theta_{PO},Z) + \mathcal{N}_\Theta(\theta_{PO}).
\end{equation}
At the Nash equilibrium $\theta_{PO}$, each player $i$ has no intention to deviate from $\theta^i_{PO}$ when actions of all other players remain at $\theta_{PO}^{-i}$. Note that performative stable equilibria $\theta_{PS}$ can be seen as the Nash equilibria of a static problem set, where the underlying distribution is fixed at $\theta_{PS}$, the definitions here are also valid for it.

It is worth noting that the notation introduced above subsumes the single-player performative prediction as a special case. When $m = 1$, the Nash equilibrium condition degenerates to the performative optimality problem:
\begin{equation*}
    \theta_{\mathrm{PO}} = \argmin_{\theta \in \Theta} \mathcal{L}(\theta) = \argmin_{\theta \in \Theta} \E_{Z \sim \D(\theta)} \ell(\theta, Z),
\end{equation*}
where the optimal parameter $\theta_{PO}$ minimizes the expected loss evaluated under the distribution $\D(\theta)$ induced by its own deployment. Moreover, the performative optimality condition can be equivalently expressed in the same generalized first-order condition form as in (\ref{equ:VI}), with the operator defined by $G(\theta, Z) = \nabla_\theta \ell(\theta, Z)$.

\paragraph{Strong Monotonicity}

A map $g:\R^d \rightarrow \R$ is called $\alpha$-strongly monotone on $\Theta \subset \R^d$ for $\alpha >0$ if for every $\theta_1, \theta_2 \in \R^d$:
$$
\langle g(\theta_1)- g(\theta_2), \theta_1 - \theta_2 \rangle \geq \alpha\|\theta_1 - \theta_2\|^2.
$$
If $g = G(\theta,Z) = \nabla \ell(\theta,Z)$, then the $\alpha$-strong monotonicity of the gradient function $G(\theta,Z)$ is equivalent to the $\alpha$-strong convexity of the loss function $\ell(\theta,Z)$. As the work \cite{narang2023multiplayer} has claimed, finding global Nash equilibria is only possible for the monotone game, strong monotonicity is an important assumption throughout our analysis.

\paragraph{Probability Measures}

For notational simplicity, we will assume that all expectations with respect to a measure exist and that integration and differentiation can be interchanged whenever they appear. These assumptions are standard and can be rigorously justified under uniform integrability conditions.
Given a metric space $\Zcal$ with a Borel $\sigma$-algebra, let $\mathbb{P}(\Zcal)$ denote the set of probability measures on $\Zcal$ with finite first moment. We can measure the deviation between two measures $P,Q \in \mathbb{P}(\Zcal)$ by the Wasserstein-1 distance:
$$
W_1(P, Q) = \sup_{f \in \mathrm{Lip}_1} \left\{ \mathbb{E}_{X \sim P}[f(X)] - \mathbb{E}_{Y \sim Q}[f(Y)] \right\},
$$
where the supremum is taken over all 1-Lipschitz functions $f: \Zcal \to \mathbb{R}$.
Alternatively, one can also measure the deviation between two measures $P,Q \in \mathbb{P}(\Zcal)$ using the total variation distance:
\begin{equation*}
\begin{split}
    d_{\text{TV}}(P,Q) &= \sup_{A \subset \Zcal} |P(A) - Q(A)| \\
    & = \frac{1}{2} \sum_{x \in \Zcal} |P(x) - Q(x)| \quad \text{(discrete case)} \\
    & = \frac{1}{2} \int_{\Zcal} |p(x) - q(x)| \, dx, \quad \text{(continuous case)}
\end{split}
\end{equation*}
where $p,q$ are the probability density functions of the measures $P,Q$.

\subsection{Performatively Stable Equilibria and Repeated Retraining}

Recall that the performative stable model set generates the optimal prediction for a performative problem set based on the distribution induced by the model itself, and the stable equilibrium is a vector $\theta_{PS} \in \R^d$ satisfies:
\begin{equation}
\label{equ:stable equilibria}
    \theta_{PS}^i = \arg\min _{\theta^i \in \Theta_i}\mathbb{E}_{Z^i \sim \mathcal{D}_i(\theta_{PS})} \ell_i(\theta^i,\theta_{PS}^{-i},Z^i), \quad \forall i \in [m].
\end{equation}
Therefore, there is no need for performative stable models to retrain. \textit{Repeated Retraining} is an effective algorithm for finding the performative stable equilibria based on a model update procedure described in the work \cite{narang2023multiplayer}, which is similar to the \textit{repeated risk minimization} for the single-player setting in the work \cite{perdomo2020performative}. 

Suppose we have $m$ players. The procedure begins with an initial vector of model parameters $\theta_0 = (\theta_0^1, ..., \theta_0^m)$ chosen arbitrarily. Based on the distribution induced by the previous $\theta_{t}$, the new parameter vector $\theta_{t+1} = (\theta_{t+1}^1,..., \theta_{t+1}^m)$ is iteratively updated by minimizing the risk function $\ell_i$ evaluated on the distribution induced by the previous model with parameter $\theta_{t}$, according to the update rule for $t \in \mathbb{T}$:
\begin{equation*}
    \theta_{t+1}^i = \arg\min _{\theta^i \in \Theta_i}\mathbb{E}_{Z^i \sim \mathcal{D}_i(\theta_{t})} \ell_i(\theta^i,\theta_{t+1}^{-i},Z^i), \quad \forall i \in [m].
\end{equation*}
To ensure the convergence of the updated sequence $\{\theta_t\}_{t=1}$ towards the stable equilibria, we are required to define additional regularity assumptions for the distribution map and the loss function for each player $i$. According to \cite{perdomo2020performative} and \cite{narang2023multiplayer}, the following assumptions should hold.

\begin{assumption}
\label{asm:existence and convergence}
    Let $G_{i,\theta} (y) = \E_{Z^i \sim \D_i(\theta)}G_i(y, Z^i)$ for each player $i$ and $G_\theta (y) = \E_{Z \sim \D(\theta)}G(y, Z) = (G_{1,\theta} (y),..., G_{m,\theta} (y))$. The following assumptions are required for convergence of $\theta_t$:

    \begin{enumerate}
        \item ($\epsilon_i$-sensitivity) For every player $i \in [m]$, there exist a $\epsilon_i>0$ such that for all $\theta,\theta' \in \Theta$:
        $$
        W_1(\D_i(\theta),\D_i(\theta')) \leq \epsilon_i \|\theta - \theta'\|.
        $$ 
        
        \item ($\alpha$-strong monotonicity) For all $y \in \Theta$, the map $G_\theta (y)$ is $\alpha$-strongly monotone in $y$.
        
        \item (Lipschitz continuity) The loss function $\ell(\theta,Z)$ is $\beta_i$-jointly smooth, that is, for all $\theta, \theta' \in \Theta$ and $Z, Z' \in \mathcal{Z}$, the gradient function $G(\theta, Z)$ is $\beta_i$-Lipschitz continue in $Z^i$ and $\theta^i$ for each $i$:
        $$
        \left\| G_i(\theta^i,\theta^{-i},Z^i) - G_i(\theta^{'i},\theta^{-i},Z^i) \right\| \leq \beta_i \cdot\left\| \theta^i - \theta^{'i} \right\|,
        $$
        $$
        \left\| G_i(\theta, Z^i) - G_i(\theta,Z^{'i}) \right\| \leq \beta_i \left\| Z^i - Z^{'i} \right\|.
        $$

        \item (Compatibility) The coefficients follow $\sum_{i=1}^m (\frac{\beta_i\epsilon_i}{\alpha})^2 < 1$.
    \end{enumerate}
\end{assumption}

By the \cite[Theorem 2]{narang2023multiplayer}, we have the convergence of $\{\theta_t\}_{t=1}$ towards a unique stable equilibrium $\theta_{PS}$ at a linear rate. We summarize this result into Proposition \ref{prop:existence and convergence}.
\begin{proposition}[Existence and convergence  \citep{narang2023multiplayer}]
\label{prop:existence and convergence}
Suppose that the Assumption \ref{asm:existence and convergence} holds for the gradient function $G(\theta,Z)$ and the distribution map $\D(\theta)$, so there exist an unique equilibrium point $\theta_{PS}$, and the iterates $\theta_t$ of our update algorithm converge to $\theta_{PS}$ at a linear rate:
$$
\|\theta_t - \theta_{PS}\| \leq \delta \text{ for }t \geq (1-C)^{-1} \log \left(\frac{\|\theta_0 - \theta_{PS}\|}{\delta}\right),
$$
where $C = \sqrt{\sum_{k=1}^m (\frac{\beta_i\epsilon_i}{\alpha})^2}$.
\end{proposition}

\begin{remark}
\label{rmk:single existence and convergence}
    For the performative prediction, we set $m=1$, so the Assumption \ref{asm:existence and convergence} reduces to the $\epsilon$-sensitivity, $\beta$-joint smoothness, $\alpha$-strong convexity, and compatibility $\epsilon < \frac{\alpha}{\beta}$, which aligns with the minimal conditions required for $\{\theta_t\}_{t=1}$ convergence. Besides, $C = \frac{\beta \epsilon}{\alpha}$ in Proposition \ref{prop:existence and convergence}, which matches the result in the work \cite{perdomo2020performative}. 
\end{remark}

An effective estimation algorithm for the performative stable point in the single-player setting is introduced in the work \cite{li2025statisticalinferenceperformativity}, inspired by the repeated risk minimization. Initiated by a chosen $\theta_0$, the estimator for each iteration $\theta_t$ for time $t \geq 0$ is given by a dynamic update:
\begin{equation*}
     \hat{\theta}_{t+1}  = \arg\min _{\theta \in \Theta}\frac{1}{N}\sum_{i = 1}^{N} \ell(\theta,Z_{t,i}), \quad  Z_{t,i} = (X_{t,i},Y_{t,i}) \sim  \mathcal{D}(\hat{\theta}_t).
\end{equation*}
Under certain conditions, the estimation $\hat \theta_t$ at time $t$ is asymptotically normal, with its covariance correlated to that from the previous steps.

\subsection{Nash Equilibria and Plug-in Optimization}
\label{subsec:Nash equilibria bg}
As we have introduced above, Nash equilibrium is the point where the performative risk functions for all the players are jointly minimized, that is, the Nash equilibrium is a vector $\theta_{PO} \in \R^d$ such that 
\begin{equation}
\label{equ:Nash equilibria}
    \theta_{PO}^i = \argmin_{\theta^i \in \Theta_i} \E_{Z^i \sim \D_i(\theta^i, \theta^{-i}_{PO})} \ell_i(\theta^i, \theta^{-i}_{PO}, Z^i), \quad \forall i \in [m].
\end{equation}
While the distribution map $\D_i(\theta)$ is usually unknown, it is often intractable to find the optimal point directly and accurately.

\textit{Plug-in performative optimization} is a technique for finding the performative optimal point for the single-player case described in the work \cite{lin2023plug}. They initiate a study of the benefits of modeling feedback in performative prediction, and efficiently learn the true performative optimal point $\theta_{PO}$ by a plug-in optimal point based on a misspecified yet known distribution map. To be more specific, since the unknown distribution map $\D(\cdot)$ makes optimizing the performative risk directly a hard problem, the plug-in optimization considers using a distribution atlas $\mathcal{D}_\mathcal{B} = \{\mathcal{D}_\beta \}_{\beta \in \mathcal{B}}$ for modeling the true distribution map $\D$ with parameter $\beta$. Note that $\D \in \mathcal{D}_\mathcal{B}$ is not required. Based on the sample set ${(\theta_i, Z_i)}_{i=1}^n$ drawn from $\theta_i \sim D_\theta$ and $Z_i \sim D(\theta_i)$ with $D_\theta$ being a user-specified distribution, the best parametric model can be estimated by fitting $\hat{\beta}$ as follows:
\begin{equation}
\label{equ:beta map}
\hat{\beta} = \widehat{Map}\big((\theta_1, Z_1), \ldots, (\theta_n, Z_n)\big),
\end{equation}
where $\widehat{Map}$ is a model-fitting function. Thus, the plug-in performatively optimal point is obtained based on the fitted parametric model:
\begin{equation*}
\theta_{PO}^{\hat{\beta}} = \arg\min_\theta \mathbb{E}_{Z \sim D_{\hat{\beta}}(\theta)} \ell(\theta, Z).
\end{equation*}

The excess risk between the true optimum $\theta_{PO}$ and the plug-in optimum $\theta_{PO}^{\hat \beta}$ arises from two sources of error: the misspecification error, due to $\D \notin \D_{\mathcal{B}}$, and the statistical error, resulting from the discrepancy between $\hat{\beta}$ and $\beta$. According to \cite[Corollary 1, Theorem 3]{lin2023plug}, under certain conditions, the error bound can be further characterized from the perspective of total variation distance, which is specified in Proposition \ref{thm:original error gap}.

\begin{assumption}
\label{assumtion:origin optimal}
Assume the distribution atlas satisfies:
\begin{enumerate}
    \item ($\eta$-misspecification)
    The distribution atlas $\mathcal{D}_\mathcal{B}$ is $\eta$-misspecified:  for all $\theta \in \Theta$, it holds that 
    $$
    dist(\mathcal{D}_{\beta}(\theta) - \mathcal{D}(\theta)) \leq \eta.
    $$
    \item ($\epsilon$-smoothness)
    The distribution atlas $\mathcal{D}_\mathcal{B}$ is $\epsilon$-smooth: for all $\beta_1,\beta_2 \in \mathcal{B}$ and $\theta \in \Theta$, it holds that
    $$
    dist(\mathcal{D}_{\beta_1}(\theta) - \mathcal{D}_{\beta_2}(\theta)) \leq \epsilon \|\beta_1-\beta_2\|_2.
    $$
\end{enumerate}
\end{assumption}

\begin{proposition}[\citep{lin2023plug}]
    \label{thm:original error gap}
    Denote $PR(\theta) = \E_{Z \sim \D(\theta)}\ell(\theta,Z)$. Suppose the Assumption \ref{assumtion:origin optimal} holds in total-variation distance, the loss function is uniformly bounded as $|\ell(\theta,Z)| \leq a$, and the estimation gap between $\hat \beta$ and $\beta$ is bounded as $\|\hat \beta - \beta^*\| \leq b$, then we have 
    $$
    PR(\theta_{PO}^{\hat{\beta}}) - PR(\theta_{\text{PO}}) \leq 4a\eta + 4a\gamma b.
    $$
    If the model-fitting procedure is the empirical risk minimization, and the loss function for model fitting satisfies additional regularity conditions, the excess risk can be further characterized as follows:
    \begin{equation*}
    PR(\theta_{PO}^{\hat{\beta}}) - PR(\theta_{\text{PO}}) \leq 4a\eta + \tilde{O}\left(\frac{1}{\sqrt{n}}\right).
    \end{equation*}
\end{proposition}
Therefore, as long as the misspecification is small enough, the plug-in performative optimization is asymptotically efficient to be an accurate estimator for the true optimum.

\section{Stable Equilibria}

This section focuses on the multi-player setting. We formally characterize the iterative estimation procedure for reaching the stable equilibria, leveraging the intrinsic structure of repeated retraining and the inference framework established under single-player performativity. Furthermore, we provide theoretical guarantees that our estimators achieve the semiparametric efficiency bound across a class of perturbed problems, thereby ensuring the asymptotic efficiency of our methodology. Additionally, we establish the asymptotic normality of our estimators to facilitate principled uncertainty quantification. The corresponding statistical properties for the single-player case are discussed in detail in Section \ref{sec:single-player}.


\subsection{Empirical Repeated Retraining}

Finding performatively stability is one of the most important problems in performative prediction, as it eliminates the need for model retraining and approximately minimizes the performative risk under certain conditions, according to the \cite[Theorem 4.3]{perdomo2020performative}. We have introduced the repeated retraining in the sections above. In this section, we begin by providing a detailed description of our estimation procedures, empirical repeated retraining, for iterations $\theta_t$ for each $t$, and further the performatively stable equilibria. We then develop corresponding inference results and highlight the interconnections between each estimator. 

As mentioned before, the repeated retraining is a multi-player 
update procedure where the target model set is generated based on the distribution induced by the preceding model set. The parameter vector $\theta_{t+1}$ is a vector in $\R^d$ wherein model parameter for each player $i \in [m]$ satisfies:
\begin{equation*}
    \theta_{t+1}^i = \arg\min _{\theta^i \in \Theta_i}\mathbb{E}_{Z^i \sim \mathcal{D}_i(\theta_{t})} \ell_i(\theta^i,\theta_{t+1}^{-i},Z^i), \quad \forall i \in [m],
\end{equation*}
where the iterate $\theta_{t}$ converges to the stable equilibria $\theta_{PS}$ at a linear rate under certain conditions. Therefore, the stable equilibria can be seen as the fixed point of the game. Inspired by this, a natural next step is to extend the update algorithm within the estimation framework for constructing the estimation for each iteration $\theta_t$. We call this estimation procedure \textit{empirical repeated retraining (ERR)}, and it is summarized in Algorithm \ref{alg:ERR}.

\begin{algorithm}
    \caption{Empirical Repeated Retraining}
    \label{alg:ERR}
    \begin{algorithmic}
        \State{\bf Input:} Initial parameter vector $\theta_0 = (\theta_0^1,...,\theta_0^m)$.
        \State{\bf Output:} Iterated estimators $\{\hat \theta_t\}_{t=1}$ for $t \in \mathbb{T}$.
        \State{\bf Step 1:} At the initial step $t=1$, randomly draw $N_0^i$ samples $\{Z_{0,k}^i\}_{k=1}^{N_0^i} = \{(X_{0,k}^i, Y_{0,k}^i)\}_{k=1}^{N_0^i}$ from the initial distribution map $\D(\theta_0)$ for each player $i$.
        \State{\bf Step 2:} Construct the estimator $\hat \theta_1$ such that the following equation hold for every $i \in [m]$:
        $$
        \hat \theta_{1}^i = \arg\min _{\theta^i \in \Theta_i}\frac{1}{N^i_{1}} \sum_{k=1}^{N^i_{1}}\ell_i(\theta^i,\hat \theta_{1}^{-i},Z^i_{0,k}).
        $$
        \State{\bf Step 3:} For all $t > 1$, we randomly draw $N_t^i$ samples $\{Z_{t,k}^i\}_{k=1}^{N_t^i} = \{(X_{t,k}^i, Y_{t,k}^i)\}_{k=1}^{N_t^i}$ from the plug-in distribution map $\D(\hat \theta_{t-1})$ for each player $i$.
        \State{\bf Step 4:} Construct the estimator $\hat \theta_t$ by the similar update procedure, where the following equation holds for every $i \in [m]$:
        $$
        \hat \theta_{t}^i = \arg\min _{\theta^i \in \Theta_i}\frac{1}{N^i_{t}} \sum_{k=1}^{N^i_{t}}\ell_i(\theta^i,\hat \theta_{t}^{-i},Z^i_{k}).
        $$
    \end{algorithmic}
\end{algorithm}

For simplicity, we assume that $N_t^i = N$ for all $t \in \mathbb{T}$ and $i \in [m]$. 
We can also rewrite the minimization problems above into the variational inequality form at the population level with the gradient $G(\theta,Z)$ of all loss functions. 
With $G(\theta,Z) = (\nabla_1\ell_1, ..., \nabla_m \ell_m)$, the stable equilibria $\theta_{PS}$ solve the first-order conditions as follows:
$$
0 \in \E_{Z \sim \D(\theta_{PS})}G(\theta_{PS}, Z) + \mathcal{N}_\Theta(\theta_{PS}),
$$
and the iteration for finding the stable equilibria based on the RR method follows:
$$
0 \in \E_{Z \sim \D(\theta_t)}G(\theta_{t+1}, Z) + \mathcal{N}_\Theta(\theta_{t+1}).
$$
In this problem, we assume that the stable equilibria $\theta_{PS}$ and all the model parameter iterates $\{\theta_t\}_{t=1}$ at every time $t$ lie in the interior of their action spaces. Therefore, the normal cone will reduce to zero, and the first-order condition can be simplified as
$$
0 \in \E_{Z \sim \D(\theta_t)}G(\theta_{t+1}, Z) + {0} \implies 0 = \E_{Z \sim \D(\theta_t)}G(\theta_{t+1}, Z).
$$
Define solution map for the RR-based parameter $\theta_{t+1}$ at $t \in \mathbb{T}$ as
\begin{align*}
\label{equ:RRM-based}
    \theta_{t+1} = \mathrm{sol}(\theta_t) &= \Pi_\Theta\{y \mid \E_{Z \sim \D(\theta_t)}G(y, Z) = 0\},
\end{align*}
which is a vector such that the equation set $\left(\Pi_{\Theta_i}\{y \mid \E_{Z^i \sim \D_i(\theta_t)}G_i(y,\theta_{t+1}^{-i}, Z^i) = 0\}\right)_{i=1}^m$ holds for every player $i \in [m]$. The update algorithm for the parameter estimation $\{\hat \theta_t\}_{t=1}$ is similarly defined as:
\begin{equation}
\label{equ:ERRM-based}
    \hat \theta_{t+1} = \mathrm{\widehat{sol}}(\hat \theta_t) = \Pi_\Theta \left\{y \mid \frac{1}{N}\sum_{k=1}^N G(y, Z_{k})=0, Z_{k} \sim \D(\hat \theta_t) \right\},
\end{equation}
where the estimator $\hat \theta_{t+1}$ is a vector such that the equation $\Pi_{\Theta_i}\{y \mid \frac{1}{N}\sum_{k=1}^N G_i(y,\hat \theta_{t+1}^{-i}, Z_{k}^i)=0, Z_{k}^i \sim \D_i(\hat \theta_t)\}$ holds for every $i \in [m]$. By our definition, we know that $\mathrm{sol}(\theta_t) = \theta_{t+1}$ and $\mathrm{\widehat{sol}}(\hat\theta_t) = \hat \theta_{t+1}$ for $t = 0,1,2...$.

\subsection{Consistency and Asymptotic Normality}
The main result of this section is the asymptotic normality of our ERR-based estimators $\{\hat \theta_t\}_{t=1}$. We begin by proving the consistency of $\hat{\theta}_t$ toward $\theta_t$ for each $t \in \mathbb{T}$, and then establish a central limit theorem for the sequence. We first introduce the additional assumptions required.

\begin{assumption}
\label{asm:CLT stable}
    Here are additional assumptions required for consistency and asymptotic normality:

    \begin{enumerate}

        \item\label{it CLT stable:local lipschitz} (Local Lipschitzness) Assume the function $G(\theta, Z)$ is locally Lipschitz at each $\tilde \theta_t$ at every iteration $t$, that is, for each iteration $t \in \mathbb{T}$, there exists a neighborhood $U$ of $\tilde \theta_t$ and $L_{U}(Z) > 0$ such that for all $\theta, \theta' \in U$:
        $$
        \|G(\theta, Z) - G(\theta', Z)\| \leq L_{U}(Z) \|\theta - \theta'\|,
        $$
        with $\E \|L_{U}(Z)\|^2 < \infty$.
        
        \item\label{it CLT stable:limit jacobian} (Bounded Jacobian) The Jacobian matrix has bounded second moment:
        $$
        H_{\theta_{t-1}}(\theta) = \E_{Z\sim \D(\theta_{t-1})}\|G(\theta,Z)\|^2 < \infty.
        $$


        \item\label{it CLT stable:differentiable} (Differentiable) The map $G_{\theta}(y)$ is differentiable on $y$.
        
        \item\label{it CLT stable:pdf} (Strongly smooth distribution) The estimator $\hat{\theta}_t$ admits a Lebesgue-measurable probability density function and a characteristic function that is absolutely integrable for every $t \in \mathbb{T}$.
    \end{enumerate}
\end{assumption}
{The first three conditions are standard requirements for establishing asymptotic normality according to \cite{van2000asymptotic}. Specifically, the Local Lipschitzness condition \ref{it CLT stable:local lipschitz} is instrumental in proving both consistency and asymptotic normality. The boundedness of the Jacobian condition \ref{it CLT stable:limit jacobian} guarantees the existence of a well-defined asymptotic covariance matrix and the validity of the central limit theorem. Furthermore, the differentiability condition \ref{it CLT stable:differentiable} permits a valid first-order Taylor expansion of the estimating function around the true parameter. Following the framework of \cite{li2025statisticalinferenceperformativity}, we impose the final condition \ref{it CLT stable:pdf} to control the propagation of stochastic fluctuations across recursive iterations. Given that the estimator at time $t$ is constructed from its predecessor, the proof proceeds via a two-step decomposition: we first establish the conditional asymptotic distribution of the current estimator given the previous iterate, and then incorporate the randomness of the prior estimate to derive the marginal limiting distribution.}

\begin{remark}
    Note that here we do not impose any additional assumptions on the invertibility of the Hessian matrix that
    $$
    V_{\theta_{t-1}}(\theta) = \E_{Z\sim \D(\theta_{t-1})}\left[\frac{\partial G(\theta,Z)}{\partial \theta^\top}\right] \quad \text{is nonsingular},
    $$
    though it is a key assumption for asymptotic normality, as it is already ensured by the strong monotonicity of the gradient function.
\end{remark}

The condition of strongly smooth distribution is necessary in this setting, as the proof of asymptotic normality begins by constructing the conditional distribution of $\hat{\theta}_t$, and then requires the characteristic function of $\hat{\theta}_{t-1}$ to recover the marginal distribution. 

\begin{theorem}[Consistency and Asymptotic Normality]
\label{thm:CLT of theta, stable}
    Suppose the Assumption \ref{asm:existence and convergence} and Assumption \ref{asm:CLT stable} hold. Denote $J_{sol}(\theta_{t-1})$ as the Jacobian matrix of the map $\mathrm{sol}(\theta)$, then for all $t \in \mathbb{T}$, we have: 
    $$
    \sqrt{N}(\hat \theta_t - \theta_t) \xrightarrow{d} N(0, \Sigma_t),
    $$
    where the covariance matrix satisfies:
    \begin{align*}
        \Sigma_t &= V_{\theta_{t-1}}(\theta_t)^{-1} \E_{Z \sim \D(\theta_{t-1})}\left(G(\theta_t, Z)G(\theta_t, Z)^\top\right) V_{\theta_{t-1}}(\theta_t)^{-1} +J_{sol}(\theta_{t-1})\Sigma_{t-1}J_{sol}(\theta_{t-1})^T \\
        & = \sum_{k=1}^t \left[\prod_{j=k}^{t-1} J_{sol}(\theta_j)\right] V_{\theta_{k-1}}(\theta_k)^{-1} \E_{Z \sim \D(\theta_{k-1})}\left(G(\theta_k, Z)G(\theta_k, Z)^\top\right)V_{\theta_{k-1}}(\theta_k)^{-1}  \left[\prod_{j=k}^{t-1}J_{sol}(\theta_j)\right].
    \end{align*}

\end{theorem}

Theorem \ref{thm:CLT of theta, stable} shows that the covariance of the asymptotic distribution at time $t-1$ constitutes a component of the covariance structure at time $t$. It is intuitive since ERR has a nature of recursion, where the estimators are inherently interconnected as each $\hat{\theta}_t$ is computed based on the distribution induced by the previous estimate $\hat{\theta}_{t-1}$. Consequently, both the consistency and the asymptotic normality of $\hat{\theta}_t$ are also closely tied to that of earlier estimators. 

\subsubsection{Numerical Estimation of Covariance}

From Theorem \ref{thm:CLT of theta, stable}, we have the result of asymptotic covariance that 
$$
\Sigma_t = \Sigma +J_{sol}(\theta_{t-1})\Sigma_{t-1}J_{sol}(\theta_{t-1})^T ,
$$
$$
\Sigma = V_{\theta_{t-1}}(\theta_t)^{-1} \E_{Z \sim \D(\theta_{t-1})}\left(G(\theta_t, Z)G(\theta_t, Z)^\top\right) V_{\theta_{t-1}}(\theta_t)^{-1},
$$
where $V_{\theta_{t-1}}(\theta) = \E_{Z\sim \D(\theta_{t-1})}\left[\frac{\partial G(\theta,Z)}{\partial \theta^\top}\right]$. Since the form of the underlying distribution is unknown, the expectation based on the distribution map is impossible to calculate. Here we further provide estimations of $\Sigma_t$ for confidence interval construction, and explain their validity by showing their consistency.


The main problem is to construct the estimation for the Jacobian matrix $J_{sol}(\theta_{t-1})$. We denote the derivative as a bivariate function $F(\theta,\gamma)$, and according to the minimization procedure, we know that $\mathrm{sol}(\theta)$ is the minimizer, leading to the equality:
$$
F(\theta,\mathrm{sol}(\theta)) = \mathbb{E}_{Z \sim \mathcal{D}(\theta)} G(Z;\gamma)|_{\gamma=\mathrm{sol}(\theta)}  = 0.
$$
Denote $p(\theta,Z) = \D(\theta)$ as the joint distribution for $(\theta,Z)$, so by the theorem of implicit function, we have:
\begin{equation*}
\begin{split}
    J_{sol}(\theta_{t-1}) = \frac{\partial \mathrm{sol}(\theta_{t-1})}{\partial\theta^\top} &= -\left[\frac{\partial F(\theta_{t-1},\mathrm{sol}(\theta_{t-1}))}{\partial \gamma^\top}\right]^{-1}\left[\frac{\partial F(\theta_{t-1},\mathrm{sol}(\theta_{t-1}))}{\partial \theta^\top}\right]\\
    &= -\left[\frac{\partial \mathbb{E}_{Z \sim \mathcal{D}(\theta_{t-1})} G(Z;\gamma)}{\partial \gamma^\top}\big|_{\gamma=\mathrm{sol}(\theta_{t-1})} \right]^{-1}\left[\frac{\partial \mathbb{E}_{Z \sim \mathcal{D}(\theta_{t-1})} G(Z;\gamma)}{\partial \theta^\top}\big|_{\gamma=\mathrm{sol}(\theta_{t-1})} \right] \\
    &=-\left[\mathbb{E}_{Z \sim \mathcal{D}(\theta_{t-1})} \frac{\partial G(Z;\theta_{t})}{\partial \gamma^\top} \right]^{-1}\left[\frac{\partial \mathbb{E}_{Z \sim \mathcal{D}(\theta_{t-1})} G(Z;\theta_{t})}{\partial \theta^\top}\right] \\
    &=-V_{\theta_{t-1}}(\theta_{t})^{-1} \cdot \mathbb{E}_{Z \sim \mathcal{D}(\theta_{t-1})}\left[G(Z;\theta_{t})  \cdot \nabla^\top_\theta \log p(\theta_{t-1},Z) \right].
\end{split}
\end{equation*}
Since we do not know the form of the distribution map, it is impossible to calculate $\nabla_\theta \log p(\theta,Z)|_{\theta=\theta_{t-1}}$. However, inspired by the plug-in method in the optimal point part, we can first estimate the form of the distribution map by a distribution atlas, and then construct a plug-in estimator for $\nabla^\top_\theta \mathrm{sol}(\theta_{t-1})$. Similarly, we use a collection of parametric models $\mathcal{D}_{\mathcal{B}} = \{\mathcal{D}_{\beta} \}_{\beta \in \mathcal{B}}$ to model the unknown distribution map, and estimate the parameter $\hat{\beta}$. Then we substitute the true distribution map with a plug-in one, and now the derivative function is
\begin{equation*}
    \frac{\partial \mathrm{sol}_{\hat \beta}(\theta_{t-1})}{\partial\theta^\top} = - V_{\theta_{t-1}}(\theta_{t})^{-1} \cdot \mathbb{E}_{Z \sim \D_{\hat{\beta}}(\theta_{t-1})}\left[G(Z;\theta_{t})  \cdot \nabla^\top_\theta \log p_{\hat{\beta}}(\theta_{t-1},Z) \right],
\end{equation*}
where the sample estimation is available. If the distribution atlas contains the true distribution map, our estimation will converge to the true variance precisely. The estimation method and its consistency are specified in Theorem \ref{thm:consistenct of estimated cov, stable theta}.

\begin{theorem}
\label{thm:consistenct of estimated cov, stable theta}
    Suppose that $\E\left\lVert\frac{\partial G(\theta,Z)}{\partial \theta^\top}\right\rVert^2 \leq \infty$, $\E \left\lVert G(\theta, Z)\right\rVert^2 \leq \infty$ hold.  
    Denote the classical sample estimators as follows:
    \begin{align*}
        \widehat V_{\hat \theta_{t-1}}(\hat \theta_t) &= \frac{1}{N}\sum_{k=1}^N \frac{\partial G(\hat \theta_t,Z_k)}{\partial \theta^\top}, \\
        \widehat H(\hat \theta_t) &= \frac{1}{N}\sum_{k=1}^N (G(\hat{\theta}_t, Z_k)-L)(G(\hat{\theta}_t, Z_k)-L)^T ,\\
        \widehat M_{\hat \beta}(\hat \theta_t)&= \frac{1}{N}\sum_{j=1}^N \left[G(Z_j;\hat \theta_{t})  \cdot \nabla^\top_\theta \log p_{\hat{\beta}}(\hat \theta_{t-1},Z)\right],
    \end{align*}
    where $L = \frac{1}{N}\sum_{k=1}^N G(\hat{\theta}_t, Z_k) $ with $Z_k \sim D(\hat{\theta}_{t-1})$, and $Z_j \sim D_{\hat \beta}(\hat{\theta}_{t-1})$.
    Let the estimated jacobian term with fitted $\hat \beta$ be $\hat J_{sol}^{\hat \beta}(\hat \theta_{t-1}) = - \widehat V_{\hat \theta_{t-1}}(\hat \theta_{t})^{-1} \widehat M_{\hat \beta}(\hat \theta_t)$, and the estimated covariance with fitted $\hat \beta$: 
    $$
    \hat \Sigma_t^{\hat \beta} = \hat \Sigma_1 + \hat \Sigma_2 = \widehat V_{\theta_{t-1}}(\hat \theta_t)^{-1} \widehat H(\hat \theta_t) \widehat V_{\hat \theta_{t-1}}(\hat \theta_t)^{-1} + \hat J_{sol}^{\hat \beta}(\hat \theta_{t-1}) \hat \Sigma_{t-1} \hat J_{sol}^{\hat \beta}(\hat \theta_{t-1})^\top.
    $$
    Our estimated covariance is consistency:
    $$
    \hat \Sigma_t^{\hat \beta} \xrightarrow{P} \Sigma_t^{\hat \beta}.
    $$
    If the distribution atlas $\mathcal{D}_{\mathcal{B}} = \{\mathcal{D}_{\beta} \}_{\beta \in \mathcal{B}}$ contains the true distribution map, which is parametrized by $\beta^*$, and the fitted parameter $\hat \beta$ from our modeling procedure converges to $\beta^*$, the result reduces to
    $$
    \hat \Sigma_t^{\hat \beta} \xrightarrow{P} \Sigma_t.
    $$
\end{theorem}

\subsection{Efficiency}

Recall that $\theta_t=\mathrm{sol}(\theta_{t-1})$ is recursively defined based on merely the initial point $\theta_0$ and the maps $\D_{[m]}=\{\D_i:i\in[m]\}$. Therefore, $\theta_t$ can be viewed as a functional $\theta_t = f_t(\D_{[m]})$ of $\D_{[m]}$. We call the problem of estimating $\theta_t$ "semiparametric" because instead of the full map $\D_{[m]}$, we are only interested in the functional $\theta_t=f_t(\D_{[m]})$, treating the remaining information in $\D_{[m]}$ as nuisance components. In this section, our goal is to study the semiparametric efficiency for the recursively defined parameter $\theta_t$. 
Results in this section build upon the classical work of H\'{a}jek and Le Cam \citep{van2000asymptotic}, as well as the more recent work of \cite{cutler2024stochasticapproximationdecisiondependentdistributions} on the lower bound for the stable point $\theta_{PS}$. Our focus, however, is on the recursively defined $\theta_t$, whereas the stable point $\theta_{PS}$, although may be estimated via recursive procedures, is not recursively defined. Due to its recursive definition, the efficiency analysis of $\theta_t$ is more intricate than that of $\theta_{PS}$.

To study the efficiency, it is important to specify a distribution space that reflects our prior knowledge of the underlying distribution. To this end, we define the admissible distribution space as the set of all maps $\tilde \D_{[m]}$ that satisfy Assumptions \ref{asm:existence and convergence} and \ref{asm:CLT stable},
\[\mathscr{D}=\big\{\tilde\D_{[m]}=\{\tilde\D_i:i\in[m]\}:\text{$\tilde\D_{[m]}$ satisfies Assumptions \ref{asm:existence and convergence} and \ref{asm:CLT stable} for some $\tilde\epsilon_i$ and $\tilde\alpha$}\big\}.\]
Similar to \cite{cutler2024stochasticapproximationdecisiondependentdistributions}, we make the following assumptions to guarantee the existence of local parametric sub-models.
\begin{assumption}\label{asm:compact_stable}
    Suppose the following assumptions hold:
    \begin{enumerate}
        \item The space $\Theta\times\Zcal_1\times\ldots\times\Zcal_m$ is compact.
        \item The loss functions $\ell_i(\theta,Z^i)$ are twice continuously differentiable in $\theta$ on $\Theta\times\Zcal_1$.
        \item The parameter $\theta_{t-1}$ and the stable point $\theta_{PS}$ are different.
    \end{enumerate}
\end{assumption}
The first condition is imposed for simplicity. Under the first condition, the second condition is satisfied by many losses, such as the squared loss or logistic loss. The last condition is to ensure the behavior of $\theta_t$ is different from that of $\theta_{PS}$.

We denote $\bm S_j^i=\{Z_{j,k}^i:k\in[N_j^i]\}$ as the collection of all the $N_j^i$ samples observed by the $i$th player under $\D_i(\hat\theta_{j-1})$ at time $j$, and let $\bm S_{[t]}=\cup_{j\in[t],i\in[m]}\bm S_j^i$. We denote $N_t=\frac{1}{m}\sum_{i\in[m]}N_t^i$ as the player-averaged sample size in the last round under $\hat\theta_{t-1}$. Since the observed samples are drawn from the estimated distributions $\{\D_i(\hat\theta_{j-1}):i\in[m],j\in[t]\}$ rather than from the true distributions $\D_i(\theta_{j-1})$, we only consider consistent algorithms for which $\hat\theta_j\rightarrow\theta_j$ almost surely for $j\in[t-1]$. Note that this consistency constraint is mild and is satisfied by the ERR algorithm. Moreover, we impose the classical regularity conditions \citep{van2000asymptotic} that ensure the limiting distribution of $\hat\theta_t$ is invariant under smooth local perturbations.
\begin{definition}(\textbf{Regularity})\label{def:regular_stable}
    Denote $\D_{[m]}^u$ to be any smooth parametric sub-model in $\mathscr{D}$ indexed by $u\in\R^d$, we assume $\D_{[m]}^u=\D_{[m]}$ when $u=0$. Then we let $\hat\theta_j$ be estimators generated by a sequence of algorithms $\Acal_j$ under $\D_{[m]}^u$ as
    \[\hat\theta_j=\Acal_j(\bm S_{[j]}),\quad \bm S_j^i\overset{\rm i.i.d.}{\sim}\D_i^u(\hat\theta_{j-1}), \quad j\in[t], i\in[m],\quad \hat\theta_0=\theta_0.\]
    Denote $P_t^u=\prod_{i\in[m]} P_{\bm S_1^i}^u\prod_{j=2}^t P_{\bm S_j^i\mid\bm S_{j-1}^i}^u=\prod_{i\in[m],j\in[t]}\D_i^u(\hat\theta_{j-1})^{\otimes N_j^i}$ as the joint distribution of all the samples $\bm S_{[t]}$.
    We assume $\frac{N_t}{N_j^i}\rightarrow \mu_{t,j}^i$, $\hat\theta_j\rightarrow \theta_j$ $P_t$-almost surely for $j\in[t-1]$ and the estimator $\hat\theta_t$ is regular, i.e., 
    \[\sqrt{N_t}\big(\hat\theta_t-\theta^{(1/\sqrt{N_t})}_t\big)\overset{P_t^{1/\sqrt{N_t}}}{\rightsquigarrow} L,\]
    {where $\theta^{(1/\sqrt{N_t})}_t$ is the solution under the local sub-model indexed by $u=1/\sqrt{N_t}$, and $\overset{P_t^{1/\sqrt{N_t}}}{\rightsquigarrow}$ denotes weak convergence along the sequence of probability measures $P_t^{1/\sqrt{N_t}}$.}
    The limiting law $L$ does not depend on the parametric sub-model.
\end{definition}

Based on Definition \ref{def:regular_stable}, the following theorem presents a semiparametric lower bound for all regular algorithms and verifies the optimality of the ERR algorithm.
\begin{theorem}[\textbf{Convolution Theorem}]\label{thm:lower_stable}
    Suppose that Assumptions \ref{asm:existence and convergence}, \ref{asm:CLT stable} and \ref{asm:compact_stable} hold, then for any regular estimator $\hat\theta_t$ as defined in Definition \ref{def:regular_stable}, we have
    \[\sqrt{N_t}\big(\hat\theta_t-\theta_t\big)\overset{P_t^0}{\rightsquigarrow} W+R,\]
    where $R\indep W$, $W\sim N(0,\Sigma_t)$, and 
    \begin{align*}
        \Sigma_t=\sum_{j\in[t]}&\bigg(\prod_{k=j}^{t-1}J_\mathrm{sol}(\theta_l)\bigg)\bigg\{\E_{\D(\theta_{j-1})}\nabla_\theta^\top G\big(\theta_j,Z\big)\bigg\}^{-1}\mathrm{diag}\bigg\{\mu_{t,j}^i\Cov_{\D_i(\theta_{j-1})}\big(G_i(\theta_j,Z^i)\big):i\in[m]\bigg\}\\
        &\cdot\bigg\{\E_{\D(\theta_{j-1})}\nabla_\theta^\top G\big(\theta_j,Z\big)\bigg\}^{-\top}\bigg(\prod_{k=j}^{t-1}J_\mathrm{sol}(\theta_l)\bigg)^\top.
    \end{align*}
\end{theorem}

\begin{remark} Note that we have the following equivalent as for each player $i \in [m]$, the processes of data collection are independent.
    $$
    \E_{\D(\theta_{j-1})}\bigg\{G(\theta_j,Z)G(\theta_j,Z)^\top\bigg\} = \mathrm{diag}\bigg\{\mu_{t,j}^i\Cov_{\D_i(\theta_{j-1})}\big(G_i(\theta_j,Z^i)\big):i\in[m]\bigg\}.
    $$
    This independence is intuitive: in a competitive setting, each agent designs its strategy autonomously. Consequently, the data supporting each agent’s decision-making should be collected independently and not shared with others. 
\end{remark}

Since we have set $N_t^i = N$ for all $t$, so $N_t=\frac{1}{m}\sum_{i\in[m]}N_t^i = N$ and $\frac{N_t}{N_j^i}\rightarrow \mu_{t,j}=1$ at all $t$ and $j$ for each player $i$, in terms of Louwner's ordering \cite[Definition 7.13]{li2019graduate}, we have $\operatorname{Var}(W+R) \succeq \Sigma_t$, so the asymptotic covariance of $\sqrt{N}\big(\hat\theta_t-\theta_t\big)$ is lower bounded by the covariance of the limiting Gaussian variable $W$. From Theorem \ref{thm:CLT of theta, stable}, we see that if the sequence of algorithms $\Acal_j$ is the repeated retraining, and the iterated estimations $\hat{\theta_t}$ are generated from the empirical repeated retraining, 
the asymptotic covariance $\Sigma_t$ exactly attains this lower bound $\Sigma_t'$. Therefore, the ERR estimation procedure is asymptotically optimal for estimating the sequence of repeated risk minimizers $\{\theta_t\}_{t=1}$.

\section{Nash Equilibria}

Although several effective solutions for finding performative optimality under both the single-player and multi-player case have been proposed, no algorithm has yet been developed for constructing its estimator. As discussed above, plug-in optimization provides an effective approach for locating the performative optimum, since the underlying distribution becomes known once the parameter is fitted. Motivated by this insight, we propose a general estimation procedure, called \textit{recalibrated plug-in estimation}, that integrates the plug-in optimization framework with the construction idea of RePPI.

\subsection{Recalibrated Plug-in}

In this section, we first construct the estimation procedure for the distributional parameter $\beta^*$, and then build the estimation procedure for the plug-in optimum based on the fitted distribution map $\D_{\hat \beta}$. We present the asymptotic properties of both estimators separately and then demonstrate how they are interlinked in the resulting asymptotic guarantees. This two-stage analysis highlights the layered structure of plug-in performative optimization and clarifies the dependencies between the two estimations.

\subsubsection{Estimation for $\beta$: Recalibrated Estimation}
We first describe the estimation procedure for the distributional parameter $\beta$, motivated by the recent work of \cite{ji2025predictions}. Their proposed recalibrated prediction-powered inference (RePPI) method targets a similar statistical quantity as ours and has been shown to achieve efficiency among all comparable algorithms. Therefore, we expect that our estimation for $\beta$ can reach the lower bound by leveraging their insights. However, our setting differs fundamentally: it does not involve labeled versus unlabeled data, nor does it rely on predictions from a pre-trained model. As a result, following the results from surrogate outcomes literatures \cite{robins1994estimation, chen2005measurement, chen2007criteria}, we construct our loss function using a specially designed \textit{imputed loss}. As we show in Section~\ref{subsec:EIF}, this imputed loss is closely related to the efficient influence function of the target distributional parameter.



Denote the loss function for fitting $\hat \beta_i$ for player $i$ as $r_i(\theta,Z^i;\beta_i)$, and the joint distribution for $\theta$ and $Z^i$ as $p_i(\theta,Z^i)$.
For player $i$, denote $r_{i,\theta}(\theta;\beta_i)$ to be the conditional expectation of $r_i(\theta, Z^i,\beta_i)$ given $\theta$
$$
r_{i,\theta}(\theta;\beta_i)=\E_{Z^i \mid \theta \sim\D_i(\theta)}r_i(\theta,Z^i;\beta_i).
$$
To adapt our problem, we construct a primary risk function (\ref{equ:loss1}), which can be seen as a modified variant of the PPI estimator, and aim to minimize it over the distributional parameter for each $i \in [m]$:
\begin{equation}
\label{equ:loss1}
    \argmin_{\beta_i \in \mathcal{B}_i}\frac{1}{N_i}\sum_{k\in[N_i]}\bigg\{r_i(\theta_k,Z^i_k;\beta_i)-\nabla_{\beta_i}^\top r_{i,\theta}(\theta_k;\beta_i^*)\beta_i\bigg\}+\E_{D_\theta}\nabla_{\beta_i}^\top r_{i,\theta}(\theta;\beta_i^*)\beta_i.
\end{equation}
Note that the structure of (\ref{equ:loss1}) ensures its unbiasedness for the original risk $\E_{p_i(\theta,Z^i)}r_i(\theta,Z^i;\beta_i)$.

As the joint distribution $p_i(\theta,Z^i)$ is unknown, the derivative $\nabla_{\beta_i}^\top r_{i,\theta}(\theta,\beta_i^*)$ is unable to calculate. To address this challenge, \cite{ji2025predictions} suggests applying a flexible machine learning algorithm to estimate the conditional expectation  
$$
s_i(\theta) \triangleq \E_{Z^i}\big[\nabla_{\beta_i} r_i(\theta,Z^i;\tilde\beta_i)|\theta\big]
,$$
where $\tilde\beta_i$ is an initial estimator of the target parameter. The resulting estimator is denoted by $\hat{s}_i(\theta)$. The key insight is that if $\hat s_i(\theta)$ consistently estimates $s_i(\theta)$, then the final estimator $\hat{\beta}_i$, constructed using $\hat{s}_i(\theta)$ in place of the true conditional expectation, remains consistent with the ideal estimator that is generated by using $s_i(\theta)$ directly. Furthermore, the consistency of $\hat s_i(\theta)$ is an essential condition for achieving the semiparametric efficiency of our estimation under suitable regularity conditions, meaning that the estimator attains the lowest possible asymptotic variance among all regular estimators.

This step involves estimating a conditional expectation under the distribution $\D_i(\theta)$, which is considerably easier than estimating the full distribution $\D_i(\theta)$ itself. However, due to the computational complexity, the resulting estimator $\hat{s}_i(\theta)$ may be asymptotically biased without further assumptions. As a consequence, a naive plug-in of $\hat{s}_i(\theta)$ into the objective function may not definitely improve the estimation accuracy. In fact, when such bias is present, the asymptotic variance of the resulting estimator may even exceed that of an estimator based solely on empirical risk minimization. To mitigate this issue, we draw inspiration from the idea of optimal control variates introduced in \cite{gan2023prediction}. Specifically, we apply a matrix to de-correlate the loss gradient $\nabla_{\beta_i} r_i(\theta,Z_i;\beta_i^*)$ and the estimated correction term $\hat s_i(\theta)$, defined as follows:
$$
\hat M_i=\widehat{\Cov}\big(\nabla_{\beta_i} r_i(\theta,Z^i;\tilde\beta_i),\hat s_i(\theta)\big)\widehat{\Cov}\big(\hat s_i(\theta)\big)^{-1},
$$
where both covariance terms are empirically estimated. By incorporating $\hat{M}_i$ with the resulting estimator $\hat{s}_i(\theta)$ in our estimation procedure, we effectively ensure that the estimator achieves improved efficiency compared to the empirical-risk-based estimator.

Moreover, since the last integral in (\ref{equ:loss1}) cannot be calculated directly most of the time due to its complexity, we apply the Monte-Carlo method to approximate it by sample average separately. We later show that the separate Monte-Carlo method for the integral will not influence the asymptotic variance. 
Denote the Monte-Carlo samples for the last integral as $\{\tilde\theta_k: \tilde\theta_k \sim \D_\theta, k\in[\tilde N_i]\}$, the final objective risk function for estimating the distributional parameter $\beta_i$ becomes
\begin{equation}
\label{equ:objective function for beta}
    \argmin_{\beta_i \in \mathcal{B}_i} \text{ } \mathcal{L}_i(\beta_i) = \frac{1}{N_i}\sum_{k\in[N_i]}\bigg\{r_i(\theta_k,Z^i_k;\beta_i)-\frac{\tilde N_i}{N_i+\tilde N_i}\beta_i^\top\hat M_i\hat s_i(\theta_k)\bigg\}+\frac{1}{N_i+\tilde N_i}\sum_{k\in[\tilde N_i]}\beta_i^\top\hat M_i\hat s_i(\tilde\theta_k).
\end{equation}
Note that here we set $N_i = N$ for each $i$.

In practice, we apply the three-fold cross-fitting procedure to decouple the dependence between these nested estimation steps, based on the work \cite{ji2025predictions}. The estimation procedure is summarized in Algorithm \ref{alg:beta}.

\begin{algorithm}[H]
    \caption{Recalibrated Estimation for Distributional Parameter}
    \label{alg:beta}
    \begin{algorithmic}
        \State{\bf Input:} Data $\{(\theta_k,Z^i_k):i \in [m], k\in[N]\}$ and Monte-Carlo samples $\{\tilde\theta_k:k\in[\tilde N_i]\}$.
        \State{\bf Output:} Cross-fitted estimator $\hat\beta_i$ for player $i$.
        \State{\bf Step 1:} Randomly split the data $\{(\theta_k,Z^i_k):i \in [m], k\in[N]\}$ into three parts $\mathcal{M}_1$, $\mathcal{M}_2$ and $\mathcal{M}_3$.
        \State{\bf Step 2:} On $\mathcal{M}_3$, compute the initial estimator 
        \[\tilde\beta_i^{(1)}=\argmin_{\beta_i \in \mathcal{B}_i}\frac{1}{|\mathcal{M}_3|}\sum_{(\theta,Z^i)\in\mathcal{M}_3}r_i(\theta,Z^i;\beta_i).\]
        \State{\bf Step 3:} On $\mathcal{M}_2$, use any machine learning algorithm to estimate $\E[\nabla_{\beta_i} r_i(\theta,Z^i;\tilde\beta_i^{(1)})|\theta]$ as $\hat s_i^{(1)}(\theta)$.
        \State{\bf Step 4:} On $\mathcal{M}_1$, compute
        \[\hat M_i^{(1)}=\widehat{\Cov}\big(\nabla_{\beta_i} r_i(\theta,Z^i;\tilde\beta_i^{(1)}),\hat s_i^{(1)}(\theta)\big)\widehat{\Cov}\big(\hat s_i^{(1)}(\theta)\big)^{-1}.\]
        where $\widehat{\Cov}$ denotes the sample covariance matrix.
        \State{\bf Step 5:} On $\mathcal{M}_1$ and the Monte-Carlo data, solve
        \begin{equation*}
            \begin{split}
                \hat\beta_i^{(1)}=& \argmin_{\beta_i \in \mathcal{B}_i}\frac{1}{|\mathcal{M}_1|}\sum_{(\theta,Z^i)\in\mathcal{M}_1}\bigg\{r_i(\theta,Z^i;\beta_i)-\frac{\tilde N_i}{N_i+\tilde N_i}\beta_i^\top\hat M_i^{(1)}\hat s_i^{(1)}(\theta)\bigg\} \\
                &+\frac{1}{N_i+\tilde N_i}\sum_{k\in[\tilde N_i]}\beta_i^\top\hat M_i^{(1)}\hat s_i^{(1)}(\tilde\theta_k).
            \end{split}
        \end{equation*}
        \State{\bf Step 6:} Repeat Steps 2-5 with fold rotations: $(\mathcal{M}_2,\mathcal{M}_3,\mathcal{M}_1)$ and $(\mathcal{M}_3,\mathcal{M}_1,\mathcal{M}_2)$ to get $\hat\beta_i^{(2)}$ and $\hat\beta_i^{(3)}$.
        \State{\bf Step 7:} Compute the final estimator as $\hat\beta_i=\sum_{j\in[3]}\frac{|\mathcal{M}_j|}{N}\hat\beta_i^{(j)}$.
    \end{algorithmic}
\end{algorithm}

The benefits of the Recalibrated Estimation method are twofold, which are shown in Remark \ref{remark:benefits}. First, we don't need to make stringent model assumptions on the conditional expectation $s(\theta)$, and can estimate it by any machine learning algorithm. No matter how $\hat s_i$ performs, the final estimator $\hat\beta_i$ is always at least as good as that in \cite{lin2023plug}, generated by the classical empirical risk minimization. Second, if $\hat s_i(\theta)$ is indeed a consistent estimator of the conditional expectation, then $\hat{\beta}_i$ can be shown to be efficient. This property is essential for our problem, as we make no assumptions about the true distribution map $\D_i$, making it extremely difficult to estimate the conditional expectation accurately. Therefore, the use of the Recalibrated Estimation method is appropriate in this setting.

\subsubsection{Estimation for $\theta_{PO}^{\beta^*}$: Importance Sampling}
Given the fitted distributional parameter $\hat \beta_i$, we now turn to estimate the plug-in Nash equilibria. 
Based on the form of the plug-in performative optimization, which substitutes the true distribution map $\D(\theta)$ with the plug-in map $\D_\beta$, we construct the form of the Nash equilibria with the same method. 

\begin{definition}[Plug-in Nash Equilibrium]
    A vector $\theta_{PO}^{\beta} \in \R^d$ is called a plug-in Nash equilibrium for a performative prediction set with plug-in distribution map $\D_{\beta}$, if for every $i \in [m]$, the following holds:
    \begin{equation*}
    \theta_{PO}^{\beta_i} = \argmin_{\theta^i \in \Theta_i} \E_{Z^i \sim \D_{\beta_i}(\theta^i, \theta_{PO}^{\beta_{-i}})} \ell_i(\theta^i, \theta_{PO}^{\beta_{-i}}, Z^i), \quad \forall i \in [m].
\end{equation*}
\end{definition}

Though with the fitted distributional parameter $\hat \beta$, the exact form of the distribution map $D_{\hat{\beta}_i}(\theta)$ is now known, its probability density function still depends on the unknown model parameter $\theta$, which makes collecting samples an intractable question for estimation. To address this problem, we expect to find a method for accurately estimating the expectation of interest, even when the available samples are drawn from a different yet simpler and fixed distribution. It is intuitive to use a more advanced Monte Carlo method to solve this problem. However, the Markov Chain Monte Carlo method and the Sequential Monte Carlo method tend to be overly complex for our setting. On the other hand, classical Monte Carlo approaches such as rejection sampling and the inversion method are impractical, as they require evaluating the density function during their procedures, which is infeasible in the performative setting. Therefore, we adopt importance sampling. By appropriately reweighting the samples, this method allows us to construct an unbiased estimator of the target expectation.

Assume that the support for $\D_i(\theta) \ell(\theta,Z)$ is contained in the support for the proposal distribution $q_i(z)$. Since we know the probability density function of $\D_{\hat \beta}(\theta)$ is known to us, we rewrite the risk function for each $i$ by importance sampling:
\begin{equation*}
    \begin{split}
        \mathbf{PR}^{\hat{\beta}_i}(\theta) = \mathbb{E}_{D_{\hat{\beta}_i}(\theta)} \ell_i(Z^i;\theta)&= \int_{\mathbb{Z}_i} D_{\hat{\beta}_i}(z^i;\theta) \cdot \ell_i(z^i;\theta)  dz^i \\
        &= \int_{\mathbb{Z}_i} q_i(z) \cdot \frac{D_{\hat{\beta}_i}(z^i;\theta)}{q_i(z^i)}\ell_i(z^i;\theta)  dz^i \\
        &= \mathbb{E}_{Z^i \sim q_i(z)} \left[\frac{D_{\hat{\beta}_i}(Z^i;\theta)}{q_i(Z^i)}\ell_i(Z^i;\theta)\right].
    \end{split}
\end{equation*}
where the underlying distribution no longer depends on $\theta$ but a fixed and known distribution $q_i(\cdot)$. Then we are able to simplify our estimation as follows:
\begin{equation*}
    \hat \theta_{PO}^{\hat \beta_i} = \argmin_{\theta^i \in \Theta_i} \frac{1}{n_i} \sum_{k=1}^n \left(\frac{\D_{\hat \beta_i}(\theta^i, \hat \theta_{PO}^{\hat \beta_{-i}},Z^i_k)}{q_i(Z_k^i)}\ell_i(\theta^i, \hat \theta_{PO}^{\hat \beta_{-i}}, Z_k^i)\right), \quad \forall i \in [m],
\end{equation*}
where $Z^i_k \sim q_i(z) $. For simplicity, we set all $n_i = n$.
Note that since the proposal distribution $q_i(\cdot)$ is known, we can always have the number of Monte Carlo samples $n_i = O(N^\alpha)$ with $\alpha>1$, where $N$ is the sample size for fitting distribution map. This relation of sample sizes is important in the asymptotic normality in the later theorem.

It is worth noticing that importance sampling is only applicable in the plug-in setting, but not in the original performative setting, since the probability density function of the true distribution map $\mathcal{D}_i$ is unknown. In contrast, under the plug-in framework, the distribution $\mathcal{D}_{\hat{\beta}_i}$ is known and fully specified. Since $\mathcal{D}_{\hat{\beta}_i}$ is typically a function of the decision parameter $\theta$, it allows us to express the dependence of the data distribution on the parameter explicitly. This structure motivates us to perform importance sampling, after which the parameter-dependent distribution appears as part of the loss function, and the parameter $\theta$ is shifted from the data-generating process to the objective loss function, making the problem more tractable.

Similarly, we rewrite our problem into its first-order condition form.
Denote the loss function after importance sampling as follows:
$$
g(\theta,Z,\beta) = (g_1(\theta,Z^1,\beta_1),..., g_m(\theta,Z^m,\beta_m)) = \left(\frac{\D_{\beta_1}(\theta,Z^1)}{q_1(Z^1)}\ell_1(\theta,Z^1),..., \frac{\D_{\beta_m}(\theta,Z^m)}{q_m(Z^m)}\ell_m(\theta,Z^m)\right),
$$
and the vector of gradient functions as $G(\theta,Z,\beta) = (\nabla_1g_1, ..., \nabla_mg_m)$.
Suppose that the corresponding component of plug-in Nash equilibria $\theta_{PO}^{\beta^*}$ lies in the interior of the $\Theta$, then the normal cone here similarly reduces to zero. Therefore, we have the solution map of the plug-in optimality based on the true distributional parameter $\beta^*$ and the fitted distributional parameter $\hat \beta$ as 
\begin{equation}
\begin{split}
    \theta_{PO}^{\beta^*} &= \mathrm{sol}(\beta^*)  = \Pi_\Theta\{\theta \mid \E_{Z \sim q(Z)}G(\theta, Z, \beta^*) = 0\} = \left[\Pi_{\Theta_i}\{\theta \mid \E_{Z^i \sim q_i(Z^i)}G_i(\theta,\theta_{PO}^{\beta_{-i}^*}, Z^i, \beta_i^*) = 0\}\right]_{i=1}^m,\\
    \theta_{PO}^{\hat \beta} &= \mathrm{sol}(\hat \beta)  = \Pi_\Theta\{\theta \mid \E_{Z \sim q(Z)}G(\theta, Z, \hat \beta) = 0\} = \left[\Pi_{\Theta_i}\{\theta \mid \E_{Z^i \sim q_i(Z^i)}G_i(\theta,\theta_{PO}^{\hat \beta_{-i}}, Z^i, \hat \beta_i) = 0\}\right]_{i=1}^m,
\end{split}
\end{equation}
and the solution map of our estimated plug-in optimality as
\begin{equation}
\label{equ:plug-in optimal estimator}
\begin{split}
    \hat \theta_{PO}^{\hat \beta} = \widehat{\mathrm{sol}}(\hat \beta) &= \Pi_\Theta \left\{\theta \mid \frac{1}{n}\sum_{k=1}^nG(\theta, Z_k, \hat \beta) = 0, Z_k \sim q(z)\right \}\\
    & = \left[\Pi_\Theta \left\{\theta \mid \frac{1}{n}\sum_{k=1}^nG_i(\theta,\hat \theta_{PO}^{\hat \beta_{-i}}, Z_k^i, \hat \beta_i) = 0, Z_k^i \sim q_i(z) \right\}\right]_{i=1}^m.
\end{split}
\end{equation}

\subsection{Consistency and Asymptotic Normality}
In this section, we focus on establishing the central limit theorem for the estimation of the distributional parameter and the plug-in performative optimum. Note that our results are not directly towards the true Nash equilibria $\theta_{PO}$ but the best plug-in Nash equilibria $\theta_{PO}^{\beta^*}$, as here the underlying distribution for every minimization is the misspecified distribution map $\D_\beta(\theta)$. However, in the section \ref{sec:error gap} we will show that under certain conditions, the inference study for the plug-in optimal point $\theta_{PO}^{\beta^*}$ is efficient for the true optimality.

We start by establishing the asymptotic normality for $\beta_i$ for each player. The following assumptions on the objective function (\ref{equ:objective function for beta}) are required for the asymptotic normality, which are similar to the assumptions given in \cite{athey2019surrogate}.

\begin{assumption}
\label{assumption for beta}
    Assume that for each player $i$, the loss function $r_i(\theta,Z^i;\beta_i)$, its gradient $\nabla_{\beta_i} r_i(\theta,Z^i;\beta_i)$ and the imputed loss function $h_i(\theta,Z^i,\beta_i) = \beta_i^\top\hat M_i\hat s_i(\theta)$ for fitting the distribution map satisfy:
\begin{enumerate}
    \item (Locally Lipschitz) $r_i(\theta,Z^i;\beta_i)$, $\nabla r_i(\theta,Z^i;\beta_i)$ and $h_i(\theta,Z^i,\beta_i)$ are locally lipschitz around $\beta_i^*$, that is, for $\beta_i \in \mathcal{B}_i$, there exists a neighborhood $U_i$ of $\beta_i^*$ and constants $L_{U_1}^i > 0$, $L_{U_2}^i > 0$ and $L_{U_3}^i > 0$ such that for all $\beta_1, \beta_2 \in U_i$:
    $$
    \|r_i(\theta,Z^i;\beta_1) - r_i(\theta,Z^i;\beta_2)\| \leq L_{U_1}^i(\theta, Z^i) \|\beta_1 - \beta_2\|,
    $$
    $$
    \|\nabla_{\beta_i} r_i(\theta,Z^i;\beta_1) - \nabla_{\beta_i} r_i(\theta,Z^i;\beta_2)\| \leq L_{U_2}^i(\theta, Z^i) \|\beta_1 - \beta_2\|,
    $$
    $$
    \|h_i(\theta,Z^i;\beta_1) - h_i(\theta,Z^i;\beta_2)\| \leq \|\hat M_i \hat s_i(\theta)\| \|\beta_1 - \beta_2\|,
    $$
    with $\E (L_{U_1}^i(\theta, Z^i)+L_{U_2}^i(\theta, Z^i)+\|\hat M_i \hat s_i(\theta)\|) < \infty$.
    \item (Differentiable) The functions $r_i(\theta,Z^i;\beta_i)$, $\nabla r_i(\theta,Z^i;\beta_i)$ and $h_i(\theta,Z^i;\beta_i)$ is differentiable in $\beta_i$ at $\beta_i^*$.
    \item (Invertibility and Positive Definite) The hessian matrix $H_i(\beta_i^*) = \mathbb{E}[\nabla^2_{\beta_i} r_i(\theta, Z^i; \beta_i^*)]$ is nonsingular, and two covariance matrices $\operatorname{Cov}\nabla_{\beta_i} (r_i(\theta,Z^i;\beta_i^*))$ and $\operatorname{Cov}(\E_{\D_i(\theta)}\nabla_{\beta_i} r_i(\theta,Z^i;\beta_i^*))$ are positive definite.
    \item (Convexity) The loss function $r_i(\theta,Z^i;\beta_i)$ and is strongly convex over $\beta_i$ with parameter $\gamma_i$, and the function $h_i(\theta,Z^i;\beta_i)$ is convex in $\beta_i$.
\end{enumerate}
\end{assumption}

With necessary assumptions on the objective function (\ref{equ:objective function for beta}), we can construct a central limit theorem for our estimator, as in Theorem \ref{thm:CLT of beta}.

\begin{theorem} 
\label{thm:CLT of beta}
Assume that Assumption \ref{assumption for beta} holds. If sample sizes satisfy that $\frac{N}{\tilde N_i}\rightarrow 0$, and $\E\|\hat s_i(\theta) - s_i(\theta)\|^2 \xrightarrow{p} 0$ for some $s(\theta)$, we have the central limit theorem that:
    $$
    \sqrt{N}(\hat{\beta}_i - \beta_i^*) \xrightarrow{P} N(0,\Sigma_{\beta_i}).
    $$ 
Moreover, if $s_i(\theta) = s_i^*(\theta)= \E\big[\nabla_{\beta_i} r_i(\theta,Z^i;\beta_i^*)|\theta\big]$, then we have the asymptotic covariance as 
    $$
    \Sigma_{\beta_i} = H_i(\beta_i^*)^{-1} V_i(\beta_i^*) H_i(\beta_i^*)^{-1},
    $$
    $$
    V_i(\beta_i^*) = \operatorname{Cov}\left(\nabla_{\beta_i} r_i(\theta,Z^i;\beta_i^*)\right) - \operatorname{Cov}\left(\E\big[\nabla_{\beta_i} r_i(\theta,Z^i;\beta_i^*)|\theta\big]\right).
    $$
\end{theorem}

\begin{remark}
    Note that a more general result for the asymptotic covariance is when $\frac{N}{\tilde N_i} \rightarrow r_i$ holds:
    $$
    \Sigma_{\beta_i}' = H_i(\beta_i^*)^{-1} \left(\operatorname{Cov}\left(\nabla_{\beta_i} r_i(\theta,Z^i;\beta_i^*)\right) - \frac{1}{1+r_i}\operatorname{Cov}\left(\E\big[\nabla_{\beta_i} r_i(\theta,Z^i;\beta_i^*)|\theta\big]\right)\right) H_i(\beta_i^*)^{-1}.
    $$
    Since the samples $\tilde \theta$ for Monte Carlo are drawn from a known distribution $\D_\theta$, the number of Monte Carlo samples $\tilde N_i$ can be specified by us independently of the number of samples $N$. Therefore, we can always have $\tilde N_i = O(N^{\alpha})$ with $\alpha>1$, so $\frac{N}{\tilde N_i}\rightarrow 0$ always holds, and the covariance matrix $\Sigma_{\beta_i}'$ reduce to $\Sigma_{\beta_i}$.
\end{remark}

\begin{remark}
\label{remark:benefits}
Theorem \ref{thm:CLT of beta} highlights two important advantages of using the Recalibrated Estimation approach in this setting. If $\hat s_i(\theta)$ is consistent with the true conditional expectation $s_i(\theta)$, then $\hat\beta_i$ is the optimal estimation. If $\hat s_i$ is asymptotically biased, that is, $\hat s_i$ is consistent around some other $\tilde s_i$, the covariance of the conditional expectation is still positive, ensuring the inequality $V_i(\beta_i^*) \leq \operatorname{Cov}\left(\nabla_{\beta_i} r_i(\theta,Z^i;\beta_i^*)\right)$. 

\end{remark}

\begin{remark}
    The work \cite{lin2023plug} uses the empirical risk minimization as a main choice for fitting $\beta_i$
    $$
    \hat \beta_i \triangleq \arg\min_{\beta_i \in \mathcal{B}_i} \frac{1}{N}r_i(\theta_k,Z_k^i;\beta_i),
    $$
    where $(\theta_k,Z_k^i) \sim p_i(\theta,Z^i)$. By a similar proof process, we know that the consistency of $\hat \beta_i$ holds, and the asymptotic covariance is 
    $$
    \Sigma_{\beta_i} = H_i(\beta_i^*)^{-1} \operatorname{Cov}\left(\nabla_{\beta_i} r_i(\theta,Z^i;\beta_i^*)\right) H_i(\beta_i^*)^{-1}.
    $$
    Since $V_i(\beta_i^*) \leq \operatorname{Cov}\left(\nabla_{\beta_i} r_i(\theta,Z^i;\beta_i^*)\right)$ always holds from Remark \ref{remark:benefits}, we demonstrate that our estimator is at least as good as the estimator for the distributional parameter generated by method in \cite{lin2023plug}.
\end{remark}

Now we build the central limit theorem for plug-in Nash equilibria $\theta_{PO}^\beta$. In addition to the assumptions stated in Assumption \ref{assumption for beta}, we need several extra assumptions on the gradient function $G(\theta,Z;\beta)$, outlined in Assumption \ref{assumption for theta, optimal}. 

\begin{assumption}
\label{assumption for theta, optimal}
Suppose the modified gradient function $G(\theta,Z;\beta)$ satisfies:
\begin{enumerate}
    \item (Differentiable) The map $\mathrm{sol}(\beta)$ is differentiable in $\beta$ at $\beta^*$.
    \item (Locally Lipschitz) The function $G(\theta,Z,\beta)$ is locally Lipschitz over $\theta_{PO}^{\hat \beta}$.
    \item (Bounded Jacobian) The Jacobian matrix has bounded second moment $\theta_{PO}^{\beta^*}$:
    $$
    \E_{Z \sim q(z)}\|G(\theta_{PO}^{\beta^*},Z,\beta)\|^2 < \infty.
    $$
    \item (Positive definite) The expectation of the Hessian matrix exists and has full rank at $\theta_{PO}^{\beta^*}$:
    $$
    V_{\beta}(\theta_{PO}^{\beta^*}) = \E_{Z \sim q(z)}\left[\frac{\partial G(\theta_{PO}^{\beta^*},Z,\beta)}{\partial \theta^T}\right] \quad \text{is nonsingular}.
    $$
    \item\label{it optimal theta:pdf} (strongly smooth distribution) The estimators $\hat{\theta}_{PO}^{\hat \beta}$ and $\hat \beta$ admit a Lebesgue-measurable probability density function and a characteristic function that is absolutely integrable.
\end{enumerate}
\end{assumption}

{The first four assumptions follow the classical framework for asymptotic normality of M-estimators: differentiability and local Lipschitzness ensure a valid first-order linearization, the nonsingularity of the limiting Jacobian guarantees identification, and the finite second-moment condition allows application of the central limit theorem, according to \cite[Theorem 5.21]{van2000asymptotic}. 
Similarly, the estimator of the performative optimality is constructed via a plug-in approach based on the fitted distributional parameter. As a result, the asymptotic analysis for the plug-in estimator naturally decomposes into two stages: first we establish the conditional asymptotic normality of the plug-in estimator given the distributional parameter, and then combine this result with the asymptotic normality of the distributional parameter itself to derive the marginal asymptotic distribution of the performative optimality estimator. The last condition \ref{it optimal theta:pdf} ensures the validity of this second step.}


As shown in the estimation equation (\ref{equ:plug-in optimal estimator}), the plug-in estimator is highly related to the estimator of the parameter of the distribution map. This connection suggests that the asymptotic result of the estimator $\hat{\theta}^{\hat{\beta}_i}_{PO}$ should be influenced by the asymptotic result of $\hat{\beta}_i$. The following Theorem \ref{thm:CLT of theta, optimal} confirms this intuition, showing that the covariance matrices $\Sigma_{\beta}$ for $\hat \beta$ form a key component of the asymptotic covariance structure of the plug-in estimator.

\begin{theorem}
    \label{thm:CLT of theta, optimal}
Suppose Assumption \ref{assumption for beta} and Assumption \ref{assumption for theta, optimal} hold. Denote $s_i^*(\theta)= \E\big[\nabla_{\beta_i} r_i(\theta,Z^i;\beta_i^*)|\theta\big]$ and the Jacobian matrix $J_{sol}(\beta)$ of the map $\mathrm{sol}(\beta)$, if the sample sizes satisfy $\frac{N}{n} \rightarrow0$ and $\frac{N}{\tilde N_i} \rightarrow0$, and $\E\|\hat s_i^{(j)} - s_i^*\|^2 \xrightarrow{P} 0$ for $j=1,2,3$, then optimums satisfy $\hat{\theta}^{\hat{\beta}}_{PO} \xrightarrow{p} \theta^{\beta^*}_{PO}$, and we have
\begin{equation*}
    \sqrt{N}(\hat{\theta}^{\hat{\beta}}_{PO} - \theta^{\beta^*}_{PO})  \xrightarrow{d}  N(0,\Sigma),
\end{equation*}
where
$$
\Sigma = (J_{sol}(\beta^*)) \Sigma_{\beta} (J_{sol}(\beta^*)) ^T,
$$
$$
\Sigma_\beta = \operatorname{diag}\{H_i(\beta_i^*)^{-1}\left(\operatorname{Cov}\left( \nabla_{\beta_i} r_i(\theta_k,Z_k^i;\beta_i^*) \right) - \operatorname{Cov}\left(s_i^*(\theta_k)\right) \right)H_i(\beta_i^*)^{-1}\}.
$$
\end{theorem}

\begin{remark}
    Here the sample size $n$ for estimating the plug-in optimum $\theta_{PO}^{\beta^*}$ is chosen by us independently of the sample size $N$ for estimating the best distributional parameter $\beta^*$, since the proposal distribution in importance sampling is fully known to us, so similarly we can always have $n = O(N^\alpha)$ with $\alpha > 1$. Therefore, the fraction of sample sizes can always satisfy $\frac{N}{n} \rightarrow0$.
\end{remark}

Note that although the direct sample size for estimating the $\hat{\theta}^{\hat{\beta}}_{PO}$ in importance sampling is $n$, the true scale in our theorem is related to the sample size of joint data $N$. Recall that we have $\theta_{PO}^{\beta} = \mathrm{sol}(\beta)$, so it is a deterministic function of the parameter $\beta$. As $\beta$ itself is a functional of the joint distribution of $(\theta,Z)$ by its definition, $\theta_{PO}^{\beta}$ is therefore also a functional of this joint distribution. Accordingly, the uncertainty of our estimates should be quantified at the scale of $o(N)$ instead of $o(n)$, where $N$ represents the sample size of joint data pairs $(\theta_k,Z_k^i)$, while $n$ represents the sample size for estimating the plug-in Nash equilibria.

\subsubsection{Numerical Estimation of Covariance}
\label{subsub:numerical optimal}

As for the distribution atlas parameter $\beta$, we have the asymptotic covariance that
$$
\Sigma_{\beta} = \operatorname{diag}\{H_i(\beta_i^*)^{-1} V_i(\beta_i^*) H_i(\beta_i^*)^{-1}\},
$$
where $H_i(\beta_i^*) = \mathbb{E}[\nabla^2_{\beta_i} r_i(\theta, Z^i; \beta_i^*)]$ and $V_i(\beta_i^*) = \operatorname{Cov}\left(\nabla_{\beta_i} r_i(\theta,Z^i;\beta_i^*)\right) - \operatorname{Cov}\left(\E\big[\nabla_{\beta_i} r_i(\theta,Z^i;\beta_i^*)|\theta\big]\right)$. 
As their closed forms are complicated to calculate directly, we similarly use classical sample estimations as substitutes, and explain their validity by their properties of consistency.

\begin{theorem}
\label{thm:consistency of estimated cov, beta}
    Suppose that conditions $\mathbb{E}\|\nabla^2_{\beta_i} r_i(\theta, Z^i; \beta_i^*)\|^2 \leq \infty$, $\mathbb{E}\|\nabla_{\beta_i} r_i(\theta, Z^i; \beta_i^*)\|^2 \leq \infty$ and $\sup_\theta \E\big[\|\nabla_{\beta_i} r_i(\theta,Z^i;\beta_i^*)\|^2|\theta] \leq \infty$ hold.
    Denote the classical sample estimators as follows:
    \begin{align*}
        \hat H_i(\beta_i^*) &= \frac{1}{N} \sum_{k = 1} ^N \left[ \nabla_{\beta_i}^2 r_i(\theta_k,Z_{k}^i; \beta_i^*)\right], \\
        \hat V_a(\beta_i^*) &= \frac{1}{N} \sum_{k =1}^N \left(\nabla r_i(\theta_k,Z_{k}^i; \beta_i^*) - L_i^* \right)\left(\nabla r_i(\theta_k,Z_{k}^i; \beta_i^*) - L_i^* \right)^T, \\
        \hat V_b(\beta_i^*) &= \frac{1}{N} \sum_{k =1}^N \left(\frac{1}{M}\sum_{j=1}^M\nabla r_i(\theta_k,Z_{k,j}^i; \beta_i^*) - W_i^* \right)\left(\frac{1}{M}\sum_{j=1}^M\nabla r_i(\theta_k,Z_{k,j}^i; \beta_i^*) - W_i^* \right)^T,
    \end{align*}
    where the samples $(\theta_k,Z_{k}^i)$ and $(\theta_k,Z_{k,j}^i)$ are i.i.d. from $D_\theta \times D_i(\theta_k)$, and 
    $$L_i^* = \frac{1}{N} \sum_{k =1}^N \nabla r_i(\theta_k,Z_{k}^i; \beta_i^*),$$
    $$W_i^* = \frac{1}{N}\sum_{k=1}^N\frac{1}{M}\sum_{j=1}^M\nabla r_i(\theta_k,Z_{k,j}^i; \beta_i^*).$$
    Let our estimated covariance for the distributional parameter $\hat \Sigma_{\beta} = \operatorname{diag}\{\hat H_i(\beta_i^*)^{-1} (\hat V_a(\beta_i^*) - \hat V_b(\beta_i^*)) \hat H_i(\beta_i^*)^{-1}\}$, then we obtain its consistency:
    $$\hat \Sigma_{\beta} \xrightarrow{P} \Sigma_{\beta}.$$
\end{theorem}

As for the plug-in optimum $\theta_{PO}^{\beta^*}$, we have the asymptotic covariance that
$$
\Sigma = (J_{sol}(\beta^*)) \Sigma_{\beta} (J_{sol}(\beta^*))^T.
$$
We similarly use the theorem of implicit function. We first denote the derivative of a bivariate function that satisfies
$$
F(\beta,\mathrm{sol}(\beta)) = \mathbb{E}_{Z \sim q(z)}G(Z,\theta;\beta)|_{\theta=\mathrm{sol}(\beta)}= 0.
$$
Taking derivative over $\beta_i$, we obtain:
\begin{equation*}
    \begin{split}
    J_{sol}(\beta) = \frac{\partial \mathrm{sol}(\beta^*)}{\partial \beta^\top} &= -\left[\frac{\partial F(\beta^*,\mathrm{sol}(\beta^*))}{\partial \theta^\top}\right]^{-1}\left[\frac{\partial F(\beta^*,\mathrm{sol}(\beta^*))}{\partial \beta^\top}\right]\\
    &= -\left[\frac{\partial \mathbb{E}_{Z \sim q(z)} G(Z,\theta;\beta^*)}{\partial \theta^\top}|_{\theta=\mathrm{sol}(\beta^*)} \right]^{-1}\left[\frac{\partial \mathbb{E}_{Z \sim q(z)} G(Z,\theta;\beta^*)}{\partial \beta^\top}|_{\theta=\mathrm{sol}(\beta^*)} \right] \\
    &=-\left[\mathbb{E}_{Z \sim q(z)} \frac{\partial G(Z,\theta_{PO}^{\beta^*};\beta^*)}{\partial \theta^\top}\right]^{-1}\left[ \mathbb{E}_{Z \sim q(z)} \frac{\partial G(Z,\theta_{PO}^{\beta^*};\beta^*)}{\partial \beta^\top}\right].
    \end{split}
\end{equation*}
Since we know the form of the distribution density function of distribution maps $\D_{\beta_i^*}(\cdot)$ according to the definition of distribution atlas, the derivative of the loss function is calculable, so similarly we can construct the sample estimation for each term, while the law of large numbers supports its consistency 

\begin{theorem}
\label{thm:consistency of estimated cov, optimal theta}
    Suppose that $\mathbb{E}_{Z \sim q(z)} \left\lVert\frac{\partial G(Z,\theta_{PO}^{\beta^*};\beta^*)}{\partial \theta^\top}\right\rVert^2 \leq \infty$ and $\mathbb{E}_{Z \sim q(z)} \left\lVert \frac{\partial G(Z,\theta_{PO}^{\beta^*};\beta^*)}{\partial \beta^\top}\right\rVert^2 \leq \infty$ hold.
    Denote the classical sample estimators as follows:
    \begin{align*}
        \hat J_{1}(\beta) &= \frac{1}{N} \sum_{k = 1} ^N \left[ \frac{\partial G(Z_k,\theta_{PO}^{\beta^*};\beta^*)}{\partial \theta^\top}\right], \\
        \hat J_{2}(\beta) &= \frac{1}{N} \sum_{k =1}^N \left[\frac{\partial G(Z_k,\theta_{PO}^{\beta^*};\beta^*)}{\partial \beta^\top}\right],
    \end{align*}
    where the samples $Z_k \overset{\text{i.i.d.}}{\sim} q(z)$. 
    Let $\hat J_{sol}(\beta) = -\hat J_{1}(\beta)^{-1} \hat J_{2}(\beta)$, and our estimate covariance for the plug-in optimum as $\hat \Sigma = \hat J_{sol}(\beta)^{-1}\hat \Sigma_\beta \hat J_{sol}(\beta)^{-1}$, then we can obtain the consistency result:
    $$
    \hat \Sigma \xrightarrow{P} \Sigma.$$
\end{theorem}

\subsection{Efficiency}
\label{subsec:EIF}
The imputed loss in the risk function \eqref{equ:loss1} is not chosen randomly. Rather, it is constructed from the efficient influence functions (EIFs) of the target parameters. This calibration ensures that the resulting estimators achieve the lower bound of the asymptotic covariance. In this section, we will study the semiparametric efficiency of estimating $\theta_{PO}^{\beta^*}$ by deriving the efficient influence functions for both $\beta^*=(\beta_1^{*\top},\ldots,\beta_m^{*\top})^\top$ and $\theta_{PO}^{\beta^*}$. 

Recall that $\theta_{PO}^{\beta^*}=\mathrm{sol}(\beta^*)$ is a function of $\beta^*$ and $\beta_i^*$ is a functional of the joint distribution $\D_i(\theta)\times\D_\theta$, thus $\theta_{PO}^{\beta^*}$ is also a functional of the joint distribution $\D_\theta\times\prod_{i\in[m]}\D_i(\theta)$. Note that the map $\mathrm{sol}(\beta)$ is fully determined once $\beta$ is specified, since the objective functions $\E_{Z\sim\D_{\beta}(\theta)}G(\theta,Z)$ of $\theta$ are then known and do not require further estimation. Consequently, when $\theta_{PO}^{\beta}=\mathrm{sol}(\beta)$ is differentiable with respect to $\beta$ at $\beta^*$, Theorem 25.47 in \cite{van2000asymptotic} implies that it suffices to study the efficiency of estimating $\beta^*$. The efficiency of $\theta_{PO}^{\beta^*}$ then follows directly via the Delta method.

Let $P_{\theta,Z}$ denote the joint distribution of $(\theta,Z^1,\ldots,Z^m)$. We assume that the marginal distribution of $\theta$, denoted by $P_\theta = \mathcal{D}_\theta$, is known to us. However, we are agnostic to the structure of the conditional distribution $P_{Z|\theta} = \mathcal{D}(\theta)$; its form is unknown and potentially highly flexible. To formalize this, we consider a class of distributions defined as 
\[\mathscr{P}_{\theta,Z}=\{Q_{\theta,Z}:Q_\theta=\D_\theta,Q_{Z|\theta}=\tilde\D(\theta)=\prod_{i\in[m]}\tilde\D_i(\theta),\text{ $\tilde\D$ satisfies Assumptions \ref{assumption for beta} and \ref{assumption for theta, optimal}}\}.\]
The distribution class $\mathscr{P}_{\theta,Z}$ consists of all distributions with a fixed marginal distribution on $\theta$, but otherwise an unspecified conditional distribution on $Z$ given $\theta$ as long as Assumptions \ref{assumption for beta} and \ref{assumption for theta, optimal} are satisfied.

Similar to the stable point, for simplicity, we make the following assumptions to guarantee the existence of local parametric sub-models.
\begin{assumption}\label{asm:bounded_optimal}
    We assume $r_i(\theta,Z^i;\beta_i^*)$, $\nabla_{\beta_i}r_i(\theta,Z^i;\beta_i^*)$ and $\nabla_{\beta_i}^2 r_i(\theta,Z^i;\beta_i^*)$ are bounded on $\Theta\times\Zcal_i$ for $i\in[m]$.
\end{assumption}

Denote $G_r(\theta,Z;\beta)=(\nabla_{\beta_1}^\top r_1(\theta,Z^1;\beta_1),\ldots,\nabla_{\beta_m}^\top r_m(\theta,Z^m;\beta_m))^\top$. The following Lemma~\ref{lem:EIF} characterizes the efficient influence functions for both the target distributional map $\beta^*$ and the plug-in optimum $\theta_{PO}^{\beta^*}$ in the distribution class $\mathcal{P}_{\theta,Z}$.

\begin{lemma}
\label{lem:EIF}
    Under Assumption \ref{asm:bounded_optimal}, the efficient influence functions of $\beta^*$'s and $\theta_{PO}^{\beta^*}$ in the distribution space $\mathscr{P}_{\theta,Z}$ are
    \[\Psi_{\beta^*}(\theta,Z)=-\big\{\E_{P_{\theta,Z}}\nabla^\top_\beta G_r(\theta,Z;\beta^*)\big\}^{-1}\big\{G_r(\theta,Z;\beta^*)-\E_{\D(\theta)}G_r(\theta,Z;\beta^*)\big\},\]
    \[\Psi_{\theta_{PO}^{\beta^*}}(\theta,Z)=\nabla_\beta^\top \mathrm{sol}(\beta^*)\Psi_{\beta^*}(\theta,Z).\]
\end{lemma}

Before establishing the efficiency lower bounds for the distributional estimator and the plug-in estimator, we first introduce the concept of regularity for the estimator.

\begin{definition}[regularity]
\label{regularity, optimal}
Denote $P_{\theta,Z}^u$ to be any sub-model in the distributional class $\mathscr{P}_{\theta,Z}$ with score function $s(\theta,Z)$ such that $P_{\theta,Z}^0=P_{\theta,Z}$.
We say the estimators $\hat\beta$ and $\hat\theta_{PO}$, based on $N$ sample pairs $\{(\theta_i,Z_i):i\in[N]\}$, are regular estimates at $P_{\theta,Z}$ for $\beta^{*}$ and $\theta_{PO}^{\beta^{*}}$ if there exist
\[\sqrt{N}\big(\hat\beta-\beta^{*(1/\sqrt{N})}\big)\overset{P_{\theta,Z}^{1/\sqrt{N}}}{\rightsquigarrow}L_\beta,\quad \sqrt{N}\big(\hat\theta_{PO}-\theta_{PO}^{\beta^{*(1/\sqrt{N})}}\big)\overset{P_{\theta,Z}^{1/\sqrt{N}}}{\rightsquigarrow}L_\theta.\]
where $\beta^{*(1/\sqrt{N})}$ and $\theta_{PO}^{\beta^{*(1/\sqrt{N})}}$ are the solutions under the local sub-model indexed by $u=1/\sqrt{N_t}$, and $\overset{P_{\theta,Z}^{1/\sqrt{N}}}{\rightsquigarrow}$ denotes weak convergence along the sequence of probability measures $P_{\theta,Z}^{1/\sqrt{N}}$. The limiting laws $L_{\beta}$ and $L_\theta$ are probability measures which do not depend on $u$.
\end{definition}

\begin{theorem}[Convolution Theorem]
\label{thm:efficiency of optimal point}
    Suppose that Assumptions \ref{assumption for beta}, \ref{assumption for theta, optimal} and \ref{asm:bounded_optimal} hold, then for any regular estimators $\hat\beta$ and $\hat\theta_{PO}$ as defined in Definition \ref{regularity, optimal}, we have
    \[\sqrt{N}\big(\hat\beta-\beta^{*}\big)\overset{P_{\theta,Z}}{\rightsquigarrow} W_\beta+R_\beta,\quad \sqrt{N}\big(\hat\theta_{PO}-\theta_{PO}^{\beta^{*}}\big)\overset{P_{\theta,Z}}{\rightsquigarrow} W_\theta+R_\theta,\]
    where $R_\beta\indep W_\beta$, $R_\theta\indep W_\theta$ and $W_\beta\sim N\big(0,\Cov_{P_{\theta,Z}}(\Psi_{\beta^*})\big)$, $W_\theta\sim N\big(0,\Cov_{P_{\theta,Z}}(\Psi_{\theta_{PO}^{\beta^*}})\big)$.
\end{theorem} 
By the Theorem \ref{thm:efficiency of optimal point}, we can see that asymptotic covariance for $\sqrt{N}\big(\hat\beta-\beta^{*})$ is lower bounded by the covariance of $W_\beta$, and the asymptotic covariance for $\sqrt{N}\big(\hat\theta_{PO}-\theta_{PO}^{\beta^{*}}\big)$ is lower bounded by the covariance of $W_\theta$. Therefore, by combining the efficient influence functions of $\beta^*$ and $\theta_{PO}^{\beta^*}$, we obtain the asymptotic covariance lower bound for the distributional estimator
$$
\Sigma_1 = \big\{\E_{P_{\theta,Z}}\nabla^2_\beta r(\theta,Z;\beta^*)\big\}^{-1} \Cov \big\{\nabla_\beta r(\theta,Z;\beta^*)-\E_{\D(\theta)}\nabla_\beta r(\theta,Z;\beta^*)\big\}\big\{\E_{P_{\theta,Z}}\nabla^2_\beta r(\theta,Z;\beta^*)\big\}^{-1},
$$
and the asymptotic covariance lower bound for the plug-in estimator 
\begin{equation}
\label{equ:lower bound theta}
    \Sigma_2 = \nabla f(\beta^*)\Sigma_1\nabla^\top f(\beta^*).
\end{equation}
Note that $\Cov\left(\nabla r(\theta,Z;\beta^*),s^*(\theta)\right) = \Cov(s^*(\theta))$ holds as $s^*(\theta) = \E[\nabla_\beta r(\theta,Z;\beta)\mid\theta]$, so the asymptotic covariance lower bound for the distributional estimator can be rewritten as
\begin{equation}
\label{equ:lower bound beta}
    \Sigma_1 = \big\{\E_{P_{\theta,Z}}\nabla^2_\beta r(\theta,Z;\beta^*)\big\}^{-1} \big\{\Cov(\nabla_\beta r(\theta,Z;\beta^*))-\Cov(\E_{\D(\theta)}\nabla_\beta r(\theta,Z;\beta^*))\big\} \big\{\E_{P_{\theta,Z}}\nabla^2_\beta r(\theta,Z;\beta^*)\big\}^{-1}.
\end{equation}

As we have demonstrated in Theorem \ref{thm:CLT of beta}, the asymptotic covariance matrix $\Sigma_\beta$ of the estimator $\hat{\beta}$, obtained via the recalibrated inference method, exactly attains the lower bound $\Sigma_1$ in (\ref{equ:lower bound beta}). This result reveals the statistical efficiency and optimality of our procedure for estimating the distributional parameter $\beta^*$. Furthermore, Theorem \ref{thm:CLT of theta, optimal} shows that the asymptotic covariance matrix $\Sigma_\theta$ of the plug-in estimator, constructed using importance sampling, also achieves the lower bound $\Sigma_2$ in (\ref{equ:lower bound theta}). This highlights the optimality of our method for estimating the plug-in optimum $\theta_{PO}^{\beta^*}$. Taken together, these results demonstrate that our two-stage estimation procedure, first estimating the best distributional parameter and then the corresponding plug-in decision, achieves semiparametric efficiency at each stage. This establishes the theoretical foundation for our approach and underscores its strength in achieving the lowest possible asymptotic variance within the given model class.

\subsection{Error Gap between True Nash Equilibria and Plug-in Nash Equilibria}
\label{sec:error gap}
All the previous inference studies are not directly for the Nash equilibria. Since the underlying distribution map for the original prediction is not required to be in the distribution atlas, the plug-in Nash equilibrium is not ensured to be the true one, so the inference study we present may not be valid for the true Nash equilibria. However, in this section, we analyze and quantify the error between the plug-in Nash equilibria $\theta^{\beta^*}_{PO}$ and the true Nash equilibria $\theta_{PO}$, noting its dependence on the misspecification under certain conditions of the performative risk function and the plug-in risk function.

\begin{theorem}[Error Gap]
\label{gaps, optimal}
Suppose for each player $i$, the distribution atlas $\D_{\mathcal{B}_i}$ is $\eta_i$-misspecified and $\gamma_i$-smooth in total-variation distance, and the loss function is uniformly bounded. Moreover, suppose that at least one of the risk functions 
$$
\mathbf{PR}^{\beta_i^*}(\theta^i) = \mathbb{E}_{Z^i \sim \D_{\beta_i^*}(\theta)}\ell_i(\theta^i, \theta^{\beta_{-i}^*}_{PO}, Z^i) \quad \text{and} \quad \mathbf{PR}^i(\theta^i) = \mathbb{E}_{Z^i \sim \D_i(\theta)} \ell_i(\theta^i, \theta^{-i}_{PO}, Z^i)
$$  
is strongly convex over $\theta^i$ with convex parameter $\lambda_i$. Then the gap between the true performative optimum and the plug-in performative optimum is bounded as follows:
\begin{equation*}
    \|\theta_{PO} - \theta^{\beta^*}_{PO}\|_2^2  \leq \sum_{i=1}^m \frac{8M_i \cdot \eta_i}{\lambda_i}.
\end{equation*}
\end{theorem}

\begin{remark}
    Since the true distribution map $\D(\theta)$ is typically unknown, it is more reasonable to assume the strong convexity of the objective function $\mathbf{PR}^{\beta_i^*}(\theta)$ based on the distribution atlas.
\end{remark}


Therefore, Theorem \ref{gaps, optimal} imposes an additional requirement on the choice of the distribution atlas, specifically a convexity condition on the risk function. When this condition is satisfied, the gap between the true Nash equilibria and the plug-in Nash equilibria can be explicitly quantified in terms of the distance between the two distribution maps. As the result indicates, when the specified distribution atlas exactly contains the true distribution map, the misspecification parameter $\eta_i$ vanishes to zero for each $i$. In this ideal case, the plug-in optimum $\theta^{\beta^*}_{PO}$ coincides with the true performative optimum $\theta_{PO}$. This observation reveals that our inference framework for the plug-in optimum not only yields statistically efficient estimators in the plug-in setting, but is also meaningful in approximating the true performative optimum when the model is well specified.

\section{Special Case: Single-player Performative Prediction}
\label{sec:single-player}
While our estimation procedure and analysis have been developed under the general multi-player performative prediction framework, it is important to note that the single-player setting arises as a natural special case, as we have mentioned in the section \ref{subsec:PP setting}. When the number of agents reduces to $m = 1$, the performatively stable and Nash equilibria respectively coincide with the classical notions of performative stability and performative optimality introduced in \cite{perdomo2020performative}. Under this simplification, our inference framework remains fully valid.
In particular, the estimation method based on the empirical repeated retraining reduces to the standard repeated empirical risk minimization (RERM) procedure, and the corresponding asymptotic normality and asymptotic optimality results continue to hold. Similarly, the plug-in inference procedure combining Plug-in Minimization, Recalibrated Prediction Powered Inference, and Importance Sampling directly applies to the single-agent case, providing asymptotically efficient estimation of both the fitted distribution parameter and the plug-in optimum.

Consequently, the theoretical guarantees derived under the multiplayer framework naturally extend to the single-player setting. This demonstrates that the proposed inference framework is not only general enough to capture multi-agent interactions but also consistent with the foundational single-agent performative prediction. In this section, we specified the inference framework in the single-player performative setting, confirming the unified nature and validity of our approach.

\subsection{Performative Stability}
When $m=1$, the equations for finding the stable equilibria (\ref{equ:stable equilibria}) in the multi-player setting reduce to the following form:
\begin{equation*}
    \theta_{PS} = \arg\min _{\theta \in \Theta}\mathbb{E}_{Z \sim \mathcal{D}(\theta_{PS})} \ell(\theta,Z),
\end{equation*}
which exactly matches the performative stability in the single-player setting in the work \cite{perdomo2020performative}. They also proposed a model update algorithm for finding the performative stable point called \textit{repeated risk minimization (RRM)}. Similarly, the procedure begins with a randomly-chosen model parameter $\theta_0$, and iteratively updates the model $f_{\theta_{t+1}}$ by minimizing the risk function evaluated on the distribution induced by the previous model $f_{\theta_{t}}$, according to the update rule:
\begin{equation}
\theta_{t+1} = f(\theta_t) \triangleq \arg\min_{\theta \in \Theta} \mathbb{E}_{Z \sim D(\theta_t)} \ell(\theta, Z),
\end{equation}
for $t \in \mathbb{T}$. As mentioned in Remark \ref{rmk:single existence and convergence}, Assumption \ref{asm:existence and convergence} in the general setting simplifies to the corresponding conditions in the single-player case. We summarize it into the following assumption.

\begin{assumption}[Single-player version of Assumption \ref{asm:existence and convergence}]
\label{asm:single convergence stable}
Assume the following assumptions hold:
\begin{enumerate}
    \item ($\epsilon$-sensitivity) The distribution map \(D(\cdot)\) is \textit{$\epsilon$-sensitive}, that is, for all \(\theta, \theta' \in \Theta\):
    $$
    W_1\bigl(D(\theta), D(\theta')\bigr) \leq \epsilon \| \theta - \theta' \|_2,
    $$
    where \(W_1\) denotes the Wasserstein-1 distance.
    \item ($\beta$-jointly smoothness) The loss function $\ell(\theta, Z)$ is $\beta$-jointly smooth, that is, its gradient $\nabla_\theta \ell(\theta, Z)$ is $\beta$-Lipschitz continuous in both $\theta$ and $Z$, i.e.,
$$
\left\| \nabla_\theta \ell(\theta, Z) - \nabla_\theta \ell(\theta', Z) \right\| \leq \beta \left\| \theta - \theta' \right\|,
$$
$$
\left\| \nabla_\theta \ell(\theta, Z) - \nabla_\theta \ell(\theta,Z') \right\| \leq \beta \left\| Z - Z' \right\|,
$$
for all $\theta, \theta' \in \Theta$ and $Z, Z' \in \mathcal{Z}$.
    \item ($\alpha$-strongly convexity) The loss function $\ell(\theta, Z)$ is $\alpha$-strongly convex, that is,
\[
\ell(\theta, Z) \geq \ell(\theta', Z) + \nabla_\theta \ell(\theta', Z)^\top (\theta - \theta') + \frac{\alpha}{2} \left\| \theta - \theta' \right\|_2^2,
\]
for all $\theta, \theta' \in \Theta$ and $Z \in \mathcal{Z}$.
    \item (compatibility) The coefficients satisfy the inequality: $\epsilon < \frac{\alpha}{\beta}$.
\end{enumerate}
\end{assumption}

It is worth noting that the $\alpha$-strong convexity condition ensures that the update procedure will converge to a unique stable point, and the gradient $G(\theta,Z) =\nabla_\theta \ell(\theta,Z)$ of the loss function corresponds to $\alpha$-strong monotonicity. Additionally, the compatibility condition guarantees that the change of the distribution map with respect to $\theta$ is smooth, thereby ensuring that the dynamics induced by the unknown distribution map remain controllable. By \cite[Theorem 3.5]{perdomo2020performative}, the update iterates $\theta_t$ will converge to a unique stable point $\theta_{PS}$ at a linear rate only if all the conditions in Assumption \ref{asm:single convergence stable} hold. 

\subsubsection{Asymptotic Normality}

Based on the RRM algorithm, the work \cite{li2025statisticalinferenceperformativity} developed the repeated empirical risk minimization (RERM) algorithm for estimating the iteration $\theta_t$ in the single-player setting. At time $t=1$, we choose an initial model parameter $\theta_0$ and draw samples $\{Z_{0,i}\}_{i=1}^{N} \triangleq \{(X_{0,i}, Y_{0,i})\}_{i=1}^{N}$ from the initial distribution $D(\theta_0)$. Therefore, the estimator is constructed for $\theta_1$:
\begin{equation*}
     \hat{\theta}_{1}  = \arg\min _{\theta \in \Theta}\frac{1}{N}\sum_{i = 1}^{N} \ell(\theta,Z_{0,i}), \quad  Z_{0,i} \sim  \mathcal{D}(\theta_0).
\end{equation*}
Then for all $t > 1$, the estimator is constructed by the similar update procedure:
\begin{equation*}
     \hat{\theta}_{t}  = \arg\min _{\theta \in \Theta}\frac{1}{N}\sum_{i = 1}^{N} \ell(\theta,Z_{t-1,i}), \quad  Z_{t-1,i} \sim  \mathcal{D}(\hat{\theta}_{t-1}).
\end{equation*}
The central limit theorem of the RERM-based estimators is as follows, where the covariance at time $t$ is a weighted accumulation of all of the previous ones:
\begin{corollary}[Theorem 3.4 in \cite{li2025statisticalinferenceperformativity}]
\label{thm:single clt stable}
Suppose Assumption \ref{asm:single convergence stable} and Assumption \ref{asm:CLT stable} with $m=1$ hold, for each $t \in \mathbb{T}$. Denote Jacobian matrix $H_{\theta_{t-1}}(\theta) = \mathbb{E}_{Z\sim D(\theta_{t-1})}[\nabla_\theta G(\theta, Z)]$ and the covariance matrix $V_{\theta_{t-1}}(\theta) = \E_{Z\sim D(\theta_{t-1})}[G(\theta,Z)G(\theta,Z)^\top]$,
we have
$$
\sqrt{N}(\hat \theta_t - \theta_t) \xrightarrow{d} N(0,\Sigma_t),
$$
where 
\begin{equation*}
\begin{split}
    \Sigma_t & = H_{\theta_{t-1}}(\theta_t)^{-1} V_{\theta_{t-1}}(\theta_t) H_{\theta_{t-1}}(\theta_t)^{-1} + (\nabla G(\theta_{t-1})) \Sigma_{t-1} (\nabla G(\theta_{t-1}))^T \\
    & = \sum_{i=1}^t\left[\prod_{k=i}^{t-1}\nabla f(\theta_k)\right] H_{\theta_{i-1}}(\theta_i)^{-1} V_{\theta_{i-1}}(\theta_i) H_{\theta_{i-1}}(\theta_i)^{-1}\left[\prod_{k=i}^{t-1}\nabla f(\theta_k)\right]^T.
\end{split}
\end{equation*}
\end{corollary}

\subsubsection{Efficiency}
Besides the asymptotic normality, the local asymptotic optimality of the RERM-based estimators can be established at each time $t \in \mathbb{T}$, following arguments similar to those used for the ERR-based estimators. Since fixing any initial point $\theta_0$, the RRM-based $\theta_t$ is also merely a functional of the distribution map $\D$. To ensure the validity of the estimation procedure under a perturbed distribution, the sub-model $\D^u$ in our distribution spaces $\mathscr{D}$ should similarly hold the property of admissibility, that is, Assumption \ref{asm:single convergence stable} and the Assumption \ref{asm:CLT stable} with $m=1$ should hold for $\D^u$.

Denote $\bm S_j=\{Z_{j,i}:i\in[N_j]\}$, $\bm S_{[t]}=\cup_{j\in[t]}\bm S_j$, we also need the following constraints of regularity on the considered algorithms.

\begin{definition}[Regularity in Single-player Setting]
\label{def:estimator}
    Denote the estimators $\hat\theta_j$ generated by a sequence of algorithms $\Acal_j$ under $\D^u$ as
    \[\hat\theta_j=\Acal_j(\bm S_{[j]}),\quad \bm S_j\overset{\rm i.i.d.}{\sim}\D^u(\hat\theta_{j-1}),\quad j\in[t],\quad\hat\theta_{0}=\theta_0.\]
    Denote $P_t^u=\prod_{j\in[t]}\D^u(\hat\theta_{j-1})^{\otimes N_j}$ as the joint distribution of all the samples.
    We assume $\frac{N_t}{N_j}\rightarrow \mu_{t,j}$, $\hat\theta_j\overset{P_t^0}{\rightsquigarrow} \theta_j$ for $j\in[t-1]$ and the estimator $\hat\theta_t$ is regular, i.e., 
    \[\sqrt{N_t}\big(\hat\theta_t-\theta^{(1/\sqrt{N_t})}_t\big)\overset{P_t^{1/\sqrt{N_t}}}{\rightsquigarrow} L,\]
    where $\theta_t^{(1/\sqrt{N_t})}$ is the solution under the sub-model indexed by $u=\frac{1}{\sqrt{N_t}}$, $\overset{P_t^{1/\sqrt{N_t}}}{\rightsquigarrow}$ denotes weak convergence along the sequence of probability measures $P_t^{1/\sqrt{N_t}}$, and the limiting law $L$ doesn't depend on the parametric sub-model.
\end{definition}

\begin{corollary}[Efficiency in Single-player Setting]\label{thm:single convolution}
    Suppose that Assumption \ref{asm:single convergence stable} and the single-player version of Assumption \ref{asm:CLT stable} hold. Suppose $\theta_{t-1}\ne \theta_{PS}$, then for any regular estimator $\hat\theta_t$ as defined in Definition \ref{def:estimator}, we have
    \[\sqrt{N_t}\big(\hat\theta_t-\theta_t\big)\overset{P_t^0}{\rightsquigarrow} W+R,\]
    where $R\indep W$, $W\sim N(0,\Sigma_t)$, and
    \[\Sigma_t=\sum_{j=0}^{t}\mu_{t,j}\bigg(\prod_{k=j}^{t-1}\nabla_{\theta}^\top G(\theta_{k})\bigg)\tilde\Sigma_j\bigg(\prod_{k=j}^{t-1}\nabla_{\theta}^\top G(\theta_{k})\bigg)^\top,\]
    \[\tilde\Sigma_j=\big\{\E_{\D(\theta_{j-1})}\nabla^2\ell(\theta_j,Z)\big\}^{-1}\Cov_{\D(\theta_{j-1})}(\nabla \ell(\theta_j,Z))\{\E_{\D(\theta_{j-1})}\nabla^2\ell(\theta_j,Z)\big\}^{-1}.\]
\end{corollary}
Since \cite{li2025statisticalinferenceperformativity} has set $N_t = N$ for all $t$, $\frac{N_t}{N_j}\rightarrow \mu_{t,j}=1$ for all $t$ and $j$, in terms of the Louwner's ordering \cite[Definition 7.13]{li2019graduate}, the asymptotic covariance of $\sqrt{N}\big(\hat\theta_t-\theta_t\big)$ is lower bounded by the covariance of the limiting Gaussian variable $W$. From Theorem \ref{thm:single clt stable}, we see that the asymptotic covariance of the iterated estimation $\hat{\theta_t}$, corresponding to the RRM-based iterates $\theta_t$, exactly attains this lower bound. Therefore, the RERM estimation procedure is asymptotically efficient for estimating the sequence of repeated risk minimizers $\{\theta_t\}_{t=1}$.

\subsection{Performative Optimality}

Performative optimality is defined more directly as it represents the point that minimizes the performative risk function. In the single-player setting (m = 1), the equation for the Nash equilibria in (\ref{equ:Nash equilibria}) naturally reduces to the corresponding equation in the performative prediction framework:
\begin{equation*}
    \theta_{PO} =  \arg\min_{\theta \in \Theta}\mathbf{PR}(\theta) =\arg\min _{\theta \in \Theta}\mathbb{E}_{Z \sim \mathcal{D}(\theta)} \ell(\theta,Z).
\end{equation*}
According to the Plug-in minimization method mentioned in section \ref{subsec:Nash equilibria bg}, we construct a distribution atlas $\D_{\mathcal{B}} = \{\D_\beta\}_{\beta \in \mathcal{B}}$, and draw a sample set $(\theta_i, Z_i)_{i=1}^N$ where $\theta_i \sim \D_\theta$, and $Z_i \sim \D(\theta_i)$ with specified $\D_\theta$. Therefore, we obtain the fitted distributional parameter $\hat \beta$ by certain mapping functions, and then the plug-in optimum $\theta_{PO}^{\hat \beta}$ with plug-in distribution map $\D_{\hat \beta}$.

In \cite{lin2023plug}, the mapping function for the distributional parameter $\beta$ is chosen as the empirical risk function, which is canonical yet suboptimal. The limitation arises because this approach utilizes $D_\theta$ only through $N$ sampled observations, thereby neglecting the full information available from the known distribution $D_\theta$. The inference framework we propose under the multiplayer formulation naturally resolves this issue and remains applicable to the single-player performative prediction setting.

\subsubsection{Asymptotic Normality}
Set the number of players as $m=1$, so the objective risk function in (\ref{equ:loss1}) specialize to the following risk minimization problem:
\begin{equation}
\label{equ:single RePPI}
    \argmin_\beta \text{ } \mathcal{L}(\beta) = \frac{1}{N}\sum_{i\in[N]}\bigg\{r(\theta_i,Z_i;\beta)-\frac{\tilde N}{N+\tilde N}\beta^\top\hat M\hat s(\theta_i)\bigg\}+\frac{1}{N+\tilde N}\sum_{i\in[\tilde N]}\beta^\top\hat M\hat s(\tilde\theta_i).
\end{equation}
where the sample set for the first term is$\{(\theta_i,Z_i): (\theta_i,Z_i) \sim \D_\theta \times \D(\theta_i), i \in [N]\}$, and for the second monte carlo term is $\{\tilde\theta_i: \tilde\theta_i \sim \D_\theta, i\in[\tilde N]\}$. Similarly, $\hat s(\theta)$ is the machine-learning estimation for the the conditional expectation $s(\theta) = \E\big[\nabla_\beta r(\theta,Z;\tilde\beta)|\theta\big]$, and the de-correlated matrix is
$$
\hat M=\widehat{\Cov}\big(\nabla_\beta r(\theta,Z;\tilde\beta),\hat s(\theta)\big)\widehat{\Cov}\big(\hat s(\theta)\big)^{-1}.
$$
To obtain the estimator $\hat \beta$ for the distributional parameter, we follow the algorithm \ref{alg:single beta}, which is the single-player version of the algorithm \ref{alg:beta}.

\begin{algorithm}
    \caption{Recalibrated Estimation for Distributional Parameter, Single-player}
    \label{alg:single beta}
    \begin{algorithmic}
        \State{\bf Input:} Data $\{(\theta_i,Z_i):i\in[N]\}$ and Monte-Carlo samples $\{\tilde\theta_i:i\in[\tilde N]\}$.
        \State{\bf Output:} Cross-fitted estimator $\hat\beta$.
        \State{\bf Step 1:} Randomly split the data $\{(\theta_i,Z_i):i\in[N]\}$ into three parts $\mathcal{M}_1$, $\mathcal{M}_2$ and $\mathcal{M}_3$.
        \State{\bf Step 2:} On $\mathcal{M}_3$, compute the inital estimator 
        \[\tilde\beta^{(1)}=\argmin_\beta\frac{1}{|\mathcal{M}_3|}\sum_{(\theta,Z)\in\mathcal{M}_3}r(\theta,Z;\beta).\]
        \State{\bf Step 3:} On $\mathcal{M}_2$, use any machine learning algorithm to estimate $\E[\nabla_\beta r(\theta,Z;\tilde\beta^{(1)})|\theta]$ as $\hat s^{(1)}(\theta)$.
        \State{\bf Step 4:} On $\mathcal{M}_1$, compute
        \[\hat M^{(1)}=\widehat{\Cov}\big(\nabla_\beta r(\theta,Z;\tilde\beta^{(1)}),\hat s^{(1)}(\theta)\big)\widehat{\Cov}\big(\hat s^{(1)}(\theta)\big)^{-1}.\]
        where $\widehat{\Cov}$ denotes the sample covariance matrix.
        \State{\bf Step 5:} On $\mathcal{M}_1$ and the Monte-Carlo data, solve
        \[\hat\beta^{(1)}=\argmin_\beta\frac{1}{|\mathcal{M}_1|}\sum_{(\theta,Z)\in\mathcal{M}_1}\bigg\{r(\theta,Z;\beta)-\frac{\tilde N}{N+\tilde N}\beta^\top\hat M^{(1)}\hat s^{(1)}(\theta)\bigg\}+\frac{1}{N+\tilde N}\sum_{i\in[\tilde N]}\beta^\top\hat M^{(1)}\hat s^{(1)}(\tilde\theta_i).\]
        \State{\bf Step 6:} Repeat Steps 2-5 with fold rotations: $(\mathcal{M}_2,\mathcal{M}_3,\mathcal{M}_1)$ and $(\mathcal{M}_3,\mathcal{M}_1,\mathcal{M}_2)$ to get $\hat\beta^{(2)}$ and $\hat\beta^{(3)}$.
        \State{\bf Step 7:} Compute the final estimator as $\hat\beta=\sum_{j\in[3]}\frac{|\mathcal{M}_j|}{N}\hat\beta^{(j)}$.
    \end{algorithmic}
\end{algorithm}

Then the plug-in performatively optimal point is as follows:
\begin{equation*}
\theta_{PO}^{\hat{\beta}} = \arg\min_{\theta \in \Theta}\mathbf{PR}^{\hat \beta}(\theta)  = \arg\min_\theta \mathbb{E}_{Z \sim D_{\hat{\beta}}(\theta)} \ell(\theta, Z).
\end{equation*}
Similarly, we can rewrite the risk minimization form by importance sampling and generate the estimated plug-in optimum:
$$
\hat{\theta}^{\hat{\beta}}_{PO}  = \arg\min_\theta\frac{1}{n} \sum_{i = 1}^n \left[\frac{D_{\hat{\beta}}(Z_i;\theta)}{q(Z_i)}\ell(Z_i;\theta)\right] \triangleq \arg\min_\theta \frac{1}{n} \sum_{i = 1}^n g(\theta,Z_i;\hat{\beta}),
$$
where $Z_i \sim q(Z)$, and $q(Z)$ is a known and fixed distribution.

Denote $f(\beta) = \arg\min_{\theta \in \Theta}\mathbf{PR}^{\beta}(\theta)$, and the hessian matrix for $\beta$ as $H(\beta) = \mathbb{E}[\nabla^2_\beta r(\theta, Z; \beta)]$, we can construct the central limit theorem for both distributional parameter estimator and plug-in estimator.

\begin{corollary}[Asymptotic Normality in Single-player Setting]
\label{thm:single CLT of theta, optimal}
Suppose the Assumption \ref{assumption for beta} and Assumption \ref{assumption for theta, optimal} hold when $m=1$. Denote $s^*(\theta)= \E\big[\nabla_\beta r(\theta,Z;\beta^*)|\theta\big]$. 
If the sample sizes satisfy $\frac{N}{n} \rightarrow0$ and $\frac{N}{\tilde N} \rightarrow0$, and $\E\|\hat s - s\|^2 \xrightarrow{P} 0$ for some $s$, then we have the asymptotic normality for both the distributional parameter and the plug-in optimum estimation:
\begin{align*}
    \sqrt{N}(\hat{\beta} - \beta^*) &\xrightarrow{d} N(0,\Sigma_\beta), \\
    \sqrt{N}(\hat{\theta}^{\hat{\beta}}_{PO} - \theta^{\beta^*}_{PO})  &\xrightarrow{d}  N(0,\Sigma_{\theta}).
\end{align*}
Moreover, if $s(\theta) = s^*(\theta)$, then we have the asymptotic covariance as 
$$
\Sigma_\beta = H(\beta^*)^{-1} (\operatorname{Cov}\left(\nabla_\beta r(\theta,Z;\beta^*)\right) - \operatorname{Cov}\left(\E\big[\nabla_\beta r(\theta,Z;\beta^*)|\theta\big]\right) ) H(\beta^*)^{-1},
$$
$$
\Sigma_\theta = (\nabla f(\beta^*)) \Sigma_{\beta} (\nabla f(\beta^*)) ^T.
$$
\end{corollary}
Note that since the proposal distribution $q(\cdot)$ and the distribution $\D_\theta$ are known, we similarly can always have the number of Monte Carlo samples $n = O(N^{\alpha_1})$ and $\tilde N = O(N^{\alpha_2})$ with $\alpha_1>1$ and $\alpha_2>1$, where $N$ is the sample size for fitting the distribution map. Therefore, the sample sizes can always satisfy $\frac{N}{n} \rightarrow0$ and $\frac{N}{\tilde N} \rightarrow 0$.
Same as the multiplayer case, this quantification relationship of the sample size is important for obtaining efficiency.

\subsubsection{Efficiency}
The risk function in the minimization problem (\ref{equ:single RePPI}) remains intrinsically connected to the efficient influence function, even when the framework degenerates to the single-player setting. This structural linkage ensures that our estimation procedure retains its efficiency in the single-player performative prediction.

Set $m=1$, we reduce the joint distribution $P_{\theta, Z}$ to denote $(\theta,Z)$, and the gradient $G_r(\theta,Z;\beta) = \nabla_\beta r(\theta,Z,\beta)$. Similarly, we need Assumption \ref{asm:bounded_optimal} holds with $m=1$ to guarantee the existence of local parametric sub-models in the single-player setting:
\begin{assumption}\label{asm:single bounded_optimal}
    We assume $r(\theta,Z;\beta^*)$, $\nabla_{\beta}r(\theta,Z;\beta^*)$ and $\nabla_{\beta}^2 r(\theta,Z;\beta^*)$ are bounded on $\Theta\times\Zcal$.
\end{assumption}
Under Assumption \ref{asm:single bounded_optimal}, the efficient influence functions of $\beta^*$ and $\theta_{PO}^{\beta^*}$ in the distribution space $\mathscr{P}_{\theta,Z}$ remain in the same form as in the general case, and are given as follows:
$$\Psi_{\beta^*}(\theta,Z)=-\big\{\E_{P_{\theta,Z}}\nabla^\top_\beta G_r(\theta,Z;\beta^*)\big\}^{-1}\big\{G_r(\theta,Z;\beta^*)-\E_{\D(\theta)}G_r(\theta,Z;\beta^*)\big\},$$
$$\Psi_{\theta_{PO}^{\beta^*}}(\theta,Z)=\nabla_\beta^\top \mathrm{sol}(\beta^*)\Psi_{\beta^*}(\theta,Z).$$
By the Theorem \ref{thm:efficiency of optimal point}, we similarly obtain the convolution theorem for the single-player plug-in estimation.
\begin{corollary}[Efficiency in Single-player Setting]
\label{thm:efficiency of optimal point, single} 
    Suppose the \ref{assumption for beta}, \ref{assumption for theta, optimal}, and \ref{asm:bounded_optimal} hold when $m=1$.
    For any regular estimators $\hat\beta$ and $\hat\theta_{PO}$, we have
    \[\sqrt{N}\big(\hat\beta-\beta^{*}\big)\overset{P_{\theta,Z}}{\rightsquigarrow} W_\beta+R_\beta,\quad \sqrt{N}\big(\hat\theta_{PO}-\theta_{PO}^{\beta^*}\big)\overset{P_{\theta,Z}}{\rightsquigarrow}W_\theta+R_\theta,\]
    where $R_\beta\indep W_\beta$, $R_\theta\indep W_\theta$ and $W_\beta\sim N\big(0, \Sigma_\beta\big)$, $W_\theta\sim N\big(0,\Sigma_\theta\big)$.
    $$
    \Sigma_\beta = \big\{\E_{P_{\theta,Z}}\nabla^2_\beta r(\theta,Z;\beta^*)\big\}^{-1} \big\{\Cov(\nabla_\beta r(\theta,Z;\beta^*))-\Cov(\E_{\D(\theta)}\nabla_\beta r(\theta,Z;\beta^*))\big\} \big\{\E_{P_{\theta,Z}}\nabla^2_\beta r(\theta,Z;\beta^*)\big\}^{-1},
    $$
    $$
    \Sigma_\theta = \nabla f(\beta^*)\Sigma_1 \nabla f(\beta^*)^\top.
    $$
\end{corollary}
Therefore, as shown in Corollary \ref{thm:single CLT of theta, optimal}, the asymptotic covariance of the plug-in estimator under the single-player setting attains the semiparametric efficiency bound. This result confirms that our inference framework achieves asymptotic efficiency for performative prediction, ensuring that no regular estimator can asymptotically outperform it in terms of variance within a local neighborhood of the true parameter.

\subsubsection{Error Gap}
Similarly, we can quantify the error between the plug-in optimum $\theta^{\beta^*}_{PO}$ and the true performative optimum $\theta_{PO}$, notifying its dependence on the misspecification under certain conditions of the performative risk function and the plug-in risk function.

\begin{corollary}[Error Gap in Single-player setting]
\label{gaps, optimal}
Suppose the distribution atlas $\D_{\mathcal{B}}$ is $\eta_{TV}$-misspecified and $\gamma$-smooth in total-variation distance, and the loss function is uniformly bounded. Moreover, suppose that at least one of the risk functions 
$$
\mathbf{PR}^{\beta^*}(\theta) = \mathbb{E}_{Z \sim \D_{\beta^*}(\theta)}\ell(Z;\theta) \quad \text{and} \quad \mathbf{PR}(\theta) = \mathbb{E}_{Z \sim \D(\theta)} \ell(Z;\theta)
$$  
is strongly convex over $\theta$ with convex parameter $\lambda$. Then the gap between the true performative optimum and the plug-in performative optimum is bounded as follows:
\begin{equation*}
    \|\theta_{PO} - \theta^{\beta^*}_{PO}\|_2^2  \leq \frac{8M \cdot \eta_{TV}}{\lambda}.
\end{equation*}
\end{corollary}

\begin{remark}
    Since the true distribution map $\D(\theta)$ is typically unknown, it is more reasonable to assume the strong convexity of the objective function $\mathbf{PR}^{\beta^*}(\theta)$ based on the distribution atlas.
\end{remark}

\begin{example}
    We show an example in location-family that the strong convexity for $\mathbf{PR}^{\beta^*}(\theta) = \mathbb{E}_{Z \sim \D_{\beta^*}(\theta)}\ell(Z;\theta)$ can hold.
    Assume that $\D_\theta = U(-1,1)$, the true distribution map is $Z \sim \mathcal{N}(b+\beta_1\theta+\epsilon\beta_2\theta^2,\sigma^2)$ and the distribution atlas is $Z \sim \mathcal{N}(b+\beta\theta,\sigma^2)$. The loss functions are $r(\theta,Z;\beta) = (Z-\beta\theta)^2$ and $\ell(\theta,Z,\beta) = (Z-\theta)^2$. By direct calculation, we obtain $\beta^* = \beta_1 \neq 1$. Therefore, we have
    $$
    \mathbf{PR}^{\beta^*}(\theta) = \mathbb{E}_{Z \sim \D_{\beta^*}(\theta)}\ell(Z;\theta) = \sigma^2 + (b-\beta_1\theta)^2 - 2\theta(b+\beta_1\theta) + \theta^2,
    $$
    which is strongly convex in $\theta$.
\end{example}

When the true distribution map is contained in the distribution atlas, the misspecification parameter $\eta_{TV}$ vanishes to zero, and the plug-in optimum $\theta^{\beta^*}_{PO}$ becomes the true performative optimum $\theta_{PO}$. Therefore, we can expect our plug-in estimation method to demonstrate statistical properties of the true performative optimum pretty well under certain conditions under single-player performativity.

\section{Numerical Simulations}
In this section, we complement the theoretical analysis by presenting numerical experiments within the single-player performative prediction framework. Specifically, we conduct simulation studies under Gaussian family models to empirically validate and illustrate our theoretical results.

\subsection{Performative Stability}

Given $\theta \in \R^d$, define the distribution map as
$$
\D(\theta) = N(\epsilon\theta, \Sigma), \quad\Sigma = diag(\sigma_1^2, ... ,\sigma_d^2),
$$ where $\epsilon, \sigma_1^2, ... ,\sigma_d^2 \in \R$. Thus, the distribution map $\D(\theta)$ is $\epsilon$-sensitive.
For each step, we want to process the update procedure based on the squared error loss function $\ell(\theta,Z) = \frac{1}{2}\|Z - \theta\|^2$, which is $1$-smooth and $1$-strongly convex. According to the requirement for convergence that $\epsilon < \frac{\gamma}{\beta}$, the sensitive parameters should satisfy $\epsilon < 1$. In the following simulations, We set $d = 2$, $\sigma_1^2 = \sigma_2^2 = 0.25$, and the initial point $\theta_0 = (1.0, 2.0)^T$. We set the number of samples $N = 10000$.

From Figure \ref{fig:stable qqplot}, we validate the asymptotic normality of our estimators under different values of $\epsilon = 0.01,\ 0.05,$ and $0.2$ by presenting multivariate Q-Q plots based on the Mahalanobis distance. 

\begin{figure}[htbp]
    \centering
    \subfigure[eps=0.01]{
        \includegraphics[width=0.25\textwidth]{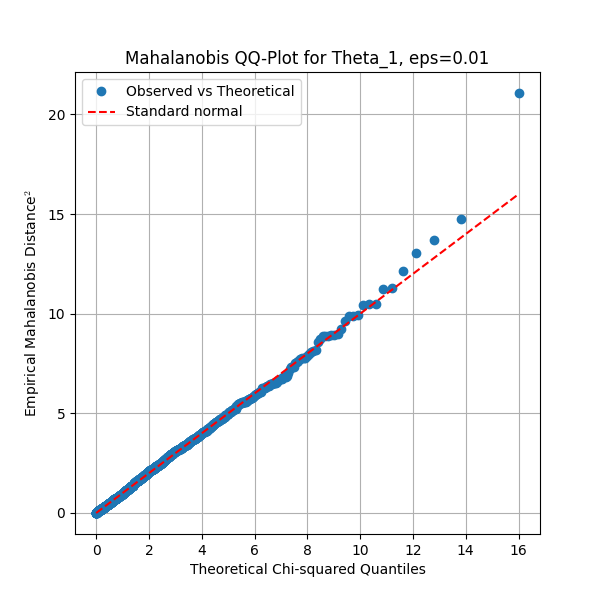}
        \label{fig:qqplot-0.01}
    }
    \subfigure[eps=0.05]{
        \includegraphics[width=0.25\textwidth]{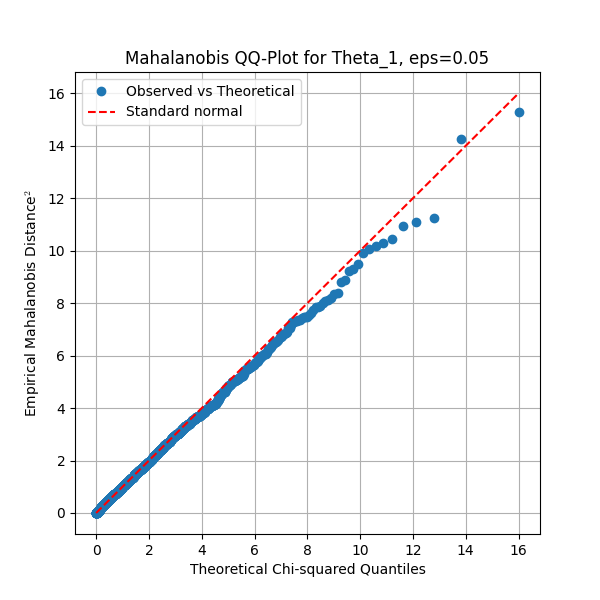}
        \label{fig:qqplot-0.05}
    }
    \subfigure[eps=0.2]{
        \includegraphics[width=0.25\textwidth]{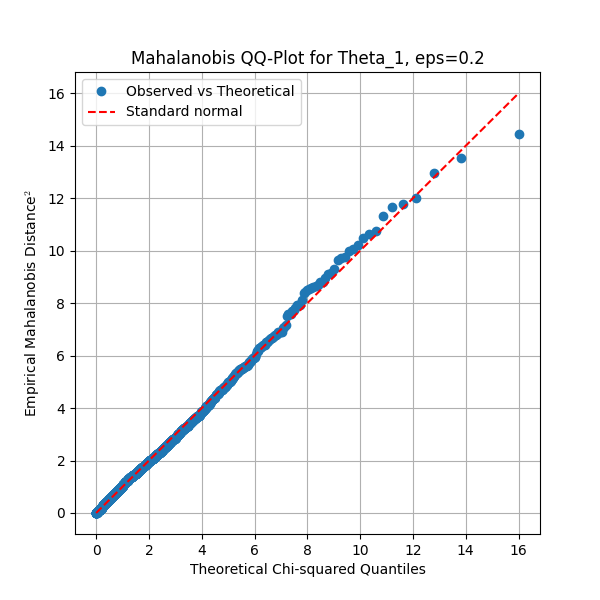}
        \label{fig:qqplot-0.2}
    }
    \caption{Mahalanobis qq-plot under different sensitivities}
    \label{fig:stable qqplot}
\end{figure}

In addition, in Figure \ref{fig:stable coverage rate}, we observe that for both entries and each value of $\epsilon = 0.01,\ 0.05,$ and $0.2$, the coverage rates for both the RRM-based $\theta_t$ and the stable point $\theta_{PS}$, which is computed using our theoretical covariance, are consistently close to the nominal level of $\alpha = 0.95$. This indicates that our theoretical construction achieves accurate coverage across a range of misspecification levels. Moreover, when $\epsilon$ is small, such as $0.01$ or $0.05$, the coverage rate curves for the RRM-based $\theta_t$ and the stable point $\theta_{PS}$ essentially overlap much earlier, suggesting that the two estimators behave nearly identically in low-sensitivity regimes. This overlap highlights the diminishing effect of sensitivity when the level of distributional shift is small. We provide a more detailed explanation of this phenomenon in Appendix \ref{appdix:subsub:add}.

\begin{figure}[htbp]
    \centering
    \includegraphics[width=0.7\textwidth]{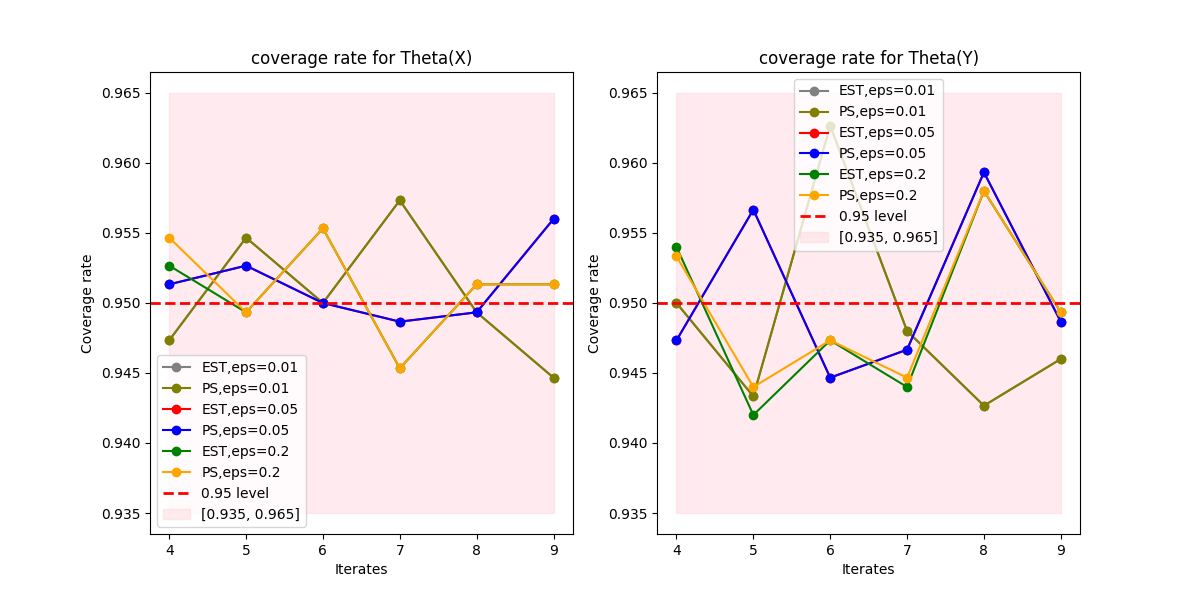}
    \caption{Coverage Rate for $\theta_t$ vs. sensitivity}
    \label{fig:stable coverage rate}
\end{figure}

\subsection{Performative Optimality}

We follow the location-family setting in \cite{lin2023plug} to construct problems for performative optimality here.
Assume that the true distribution map is a linear model with a quadratic term 
$$
Z = b + \beta_1 *\theta + \epsilon \beta_2 *\theta^2 + Z_0, \quad Z_0 \sim N(0,\sigma^2I_d),
$$
where $\epsilon \geq 0$ quantifies how severely the model is misspecified, $b, \theta \in \R^d$. To be more specific, we generate the true model parameters from $b \sim N(0,\sigma_b^2)$, $\beta_1 = \frac{\beta'_1}{\|\beta'_1\|_{OP}}$ and $\beta_2 = \frac{\beta'_2}{\|\beta'_2\|_{OP}}$, where $\beta_1, \beta_2 \sim N(0,\sigma_\beta^2)$ are i.i.d. entries-wise.
We construct the distribution atlas with simpler linear models
$$
\mathcal{D}_\beta(\theta) = b + \beta *\theta + Z_0, \quad Z_0 \sim N(0,\sigma^2I_d).
$$
Thus, $\epsilon$ is the misspecification level, and if $\epsilon = 0$, the true distribution map is contained in the distribution atlas. For fitting the distributional parameter $\beta$, we define the loss function $r(\theta,Z;\beta) = \|Z - \beta\theta\|_2^2$, where $\theta_i = \{\theta: \|\theta\| \leq 1\}$. In this setting, we can calculate that the target distributional parameter satisfies $\beta^* = \beta_1$, and the plug-in optimum has a closed form $\theta_{PO}^{\beta^*} = \frac{-b}{\beta^*-1}$. The estimation procedure is based on the squared error loss function $\ell(\theta,Z) = \|Z - \theta\|^2$. In the following simulations, we set $d = 1$, $\theta_i \sim U(-1,1)$, $\sigma_b = 1$, $ \sigma_\beta = 1$ and $\sigma = 0.5$.

\begin{figure}[htbp]
    \centering
    \subfigure[Coverage rate]{
        \includegraphics[width=0.55\textwidth]{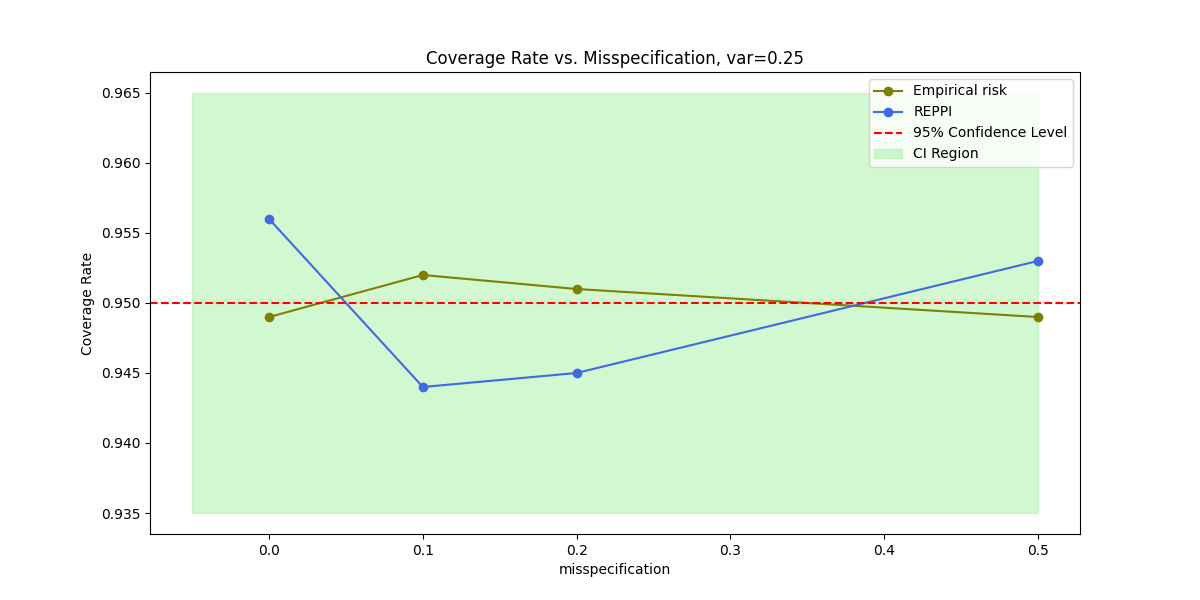}
        \label{fig:optimal coverage rate}
    }
    \subfigure[Interval width]{
        \includegraphics[width=0.55\textwidth]{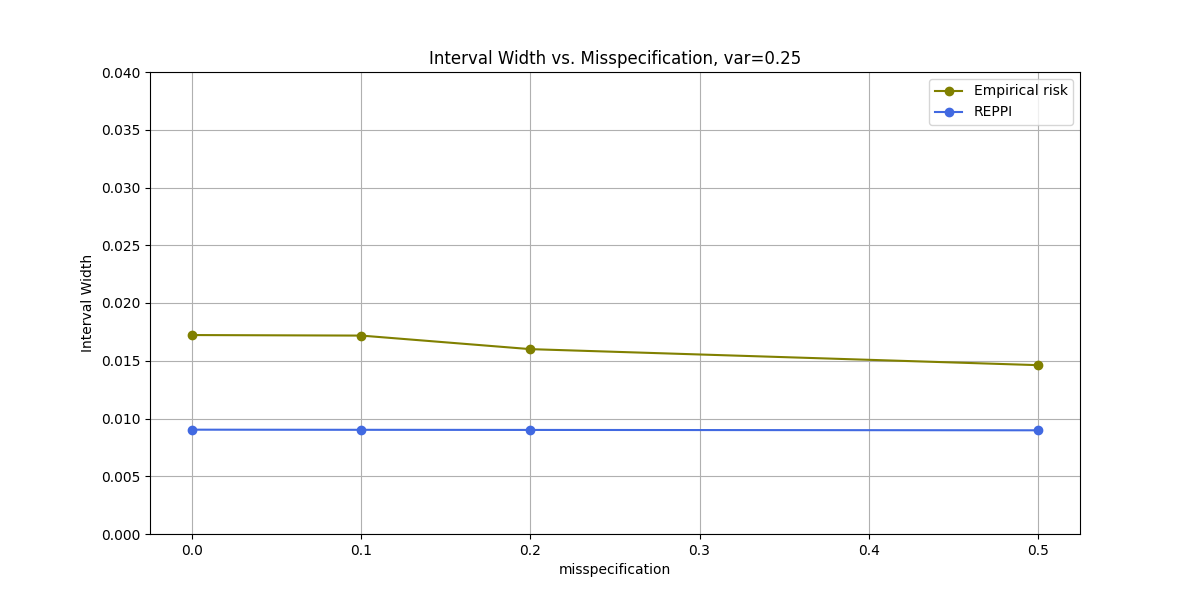}
        \label{fig:optimal interval width}
    }
    \caption{Inferential results for $\theta_{PO}^{\beta^*}$ under different misspecifications}
    \label{fig:optimal figures}
\end{figure}

Figure \ref{fig:optimal figures} presents the simulation results for the coverage rate and interval width of estimators derived from both the empirical risk function and the Recalibrated Inference method, across varying levels of model misspecification. As shown, both methods achieve the nominal coverage level of $\alpha = 0.95$ regardless of the degree of misspecification. However, the estimators obtained via Recalibrated Inference consistently exhibit narrower confidence intervals compared to those based on empirical risk, across all tested levels of misspecification, demonstrating the efficiency of the estimation procedure in our work.

\clearpage
\bibliographystyle{plain}  
\bibliography{paper}  

\clearpage  
\appendix

\section{Theoretical proofs} 

\subsection{Stable equilibria}

\subsubsection{Proof of Proposition \ref{prop:existence and convergence}}

\begin{proof}[Proof of Proposition \ref{prop:existence and convergence}]
We first prove that the solution map $\mathrm{sol}(\theta)$ is $C$-Lipschitz in $\theta$. Denote the function $f_i(Z) = \langle G_i(y,Z^i),v^i\rangle$ for the player $i$ where the vector $v^i$ satisfies $\|v^i\| \leq 1$, we can prove that it is $\beta_i$-Lipschitz:
\begin{equation*}
\begin{split}
    \|f_i(Z) - f_i(Z') \| &= \|\langle G_i(y,Z^i),v^i\rangle - \langle G_i(y,Z^{'i}),v^i\rangle \| \\
    & = \|\langle G_i(y,Z^i) - G_i(y,Z^{'i}),v^i\rangle \| \\
    & \leq \|G_i(y,Z^i) - G_i(y,Z^{'i})\| \|v^i\| \\
    & \leq \beta_i \|Z^i - Z^{'i}\|,
\end{split}
\end{equation*}
where the first inequality is by the Cauchy-Schwarz Inequality. For any $\theta, \theta', y \in \Theta$, by the dual of the norm that $\|u\| = \sup_{\|v\| \leq 1}\langle u,v \rangle$ and the duality theorem of the optimal transport in Lemma \ref{lem:duality theorem}, we can prove the $\beta_i\epsilon_i$-Lipschitzness of $G_{i,\theta}(y)$ in $\theta$:
\begin{equation*}
\begin{split}
    \|G_{i,\theta}(y) - G_{i,\theta'}(y)\| & = \|\E_{Z^i \sim \D_i(\theta)}G_i(y,Z^i) - \E_{Z^i \sim \D_i(\theta')}G_i(y,Z^i)\| \\
    & = \beta_i \cdot \sup_{\|v^i\| \leq 1}\left\{\E_{Z^i \sim \D_i(\theta)}\frac{1}{\beta_i}\langle G_i(y,Z^i),v^i\rangle - \E_{Z^i \sim \D_i(\theta')}\frac{1}{\beta_i}\langle G_i(y,Z^i),v^i\rangle \right\} \\
    & = \beta_i \cdot W_1(\D_i(\theta),\D_i(\theta')) \\ 
    & \leq \beta_i \epsilon_i \cdot \|\theta - \theta'\|.
\end{split}
\end{equation*}
Therefore, we deduce the result that
$$
\|G_{\theta}(y) - G_{\theta'}(y)\|^2 = \sum_{k=1}^m \|G_{i,\theta}(y) - G_{i,\theta'}(y)\|^2 \leq \sum_{k=1}^m (\beta_i\epsilon_i)^2 \|\theta-\theta'\|^2.
$$
By the definition of $\mathrm{sol(\theta)}$ and $\mathrm{sol}(\theta')$ and the strong monotonicity of map $G_\theta(\cdot)$, we have the inequality:
\begin{equation*}
\begin{split}
    \alpha \|\mathrm{sol}(\theta) - \mathrm{sol}(\theta')\|^2 & \leq \langle G_\theta(\mathrm{sol}(\theta)) - G_{\theta}(\mathrm{sol}(\theta')), \mathrm{sol}(\theta) - \mathrm{sol}(\theta') \rangle \\
    & = \langle G_{\theta'}(\mathrm{sol}(\theta')) - G_{\theta}(\mathrm{sol}(\theta')), \mathrm{sol}(\theta) - \mathrm{sol}(\theta') \rangle \\
    & \leq \|G_{\theta'}(\mathrm{sol}(\theta')) - G_{\theta}(\mathrm{sol}(\theta'))\| \|\mathrm{sol}(\theta) - \mathrm{sol}(\theta')\| \\
    & \leq \sqrt{\sum_{k=1}^m (\beta_i\epsilon_i)^2} \|\theta - \theta'\| \|\mathrm{sol}(\theta) - \mathrm{sol}(\theta')\|.
\end{split}
\end{equation*}
Therefore, we have the result that $\mathrm{sol}(\theta)$ is $C$-Lipschitz in $\theta$:
\begin{equation}
\label{equ:solution map lipschitz}
    \|\mathrm{sol}(\theta) - \mathrm{sol}(\theta')\| \leq \sqrt{\sum_{k=1}^m (\frac{\beta_i\epsilon_i}{\alpha})^2}\|\theta - \theta'\|.
\end{equation}
By the Banach fixed-point theorem, there exists a unique fixed point satisfying $\mathrm{sol}(\theta) = \theta$, which corresponds to the notion of stability as we have defined.

Then we prove the convergence of our iterates $\theta_t$ by our update algorithm. Note that $\theta_t = \mathrm{sol}(\theta_{t-1})$ and $\theta_{PS} = \mathrm{sol}(\theta_{PS})$ by the definition of our algorithm and stability, so with the result of (\ref{equ:solution map lipschitz}), we have
$$
\|\theta_t - \theta_{PS}\| = \|\mathrm{sol}(\theta_{t-1}) - \mathrm{sol}(\theta_{PS})\| \leq C\|\theta_{t-1} - \theta_{PS}\| \leq C^{t} \|\theta_0 - \theta_{PS}\|.
$$
Therefore, if $t \geq (1-C)^{-1} \log(\frac{\|\theta_0 - \theta_{PS}\|}{\delta})$, we have $\|\theta_t - \theta_{PS}\| \leq \delta$. The iterates $\theta_t$ converge to a unique equilibrium point at a linear rate.
\end{proof}

\begin{lemma}[Kantorovich-Rubinstein Duality Theorem]
\label{lem:duality theorem}
Let $P$ and $Q$ be probability measures on $\mathbb{R}^d$ with finite first moments. The 1-Wasserstein distance between them is given by the duality:
\[
W_1(P, Q) = \sup_{f \in \mathrm{Lip}_1} \left\{ \mathbb{E}_{X \sim P}[f(X)] - \mathbb{E}_{Y \sim Q}[f(Y)] \right\},
\]
where the supremum is taken over all 1-Lipschitz functions $f: \mathbb{R}^d \to \mathbb{R}$.
\end{lemma}

\subsubsection{Proof of Theorem \ref{thm:CLT of theta, stable}}

\begin{lemma}
\label{lem:consistency of z-estimation} 
Suppose Assumption \ref{asm:existence and convergence} holds, then we have the consistency as follows:
$$
\hat{\theta}_{t+1} =\mathrm{\widehat{sol}}(\hat \theta_t)  \xrightarrow{p}  \tilde{\theta}_{t+1} \triangleq \mathrm{sol}(\hat \theta_t).
$$
\end{lemma}

\begin{proof}[Proof of Lemma \ref{lem:consistency of z-estimation}]
Denote the maps
$\widehat{M}_{t}(\theta) = \frac{1}{N} \sum_{k = 1}^{N} G(\theta, Z_{k}),M_t(\theta) = \mathbb{E}_{ Z \sim  \mathcal{D}(\hat{\theta}_t)} G(\theta, Z)
$
where $Z_{k} \sim  \mathcal{D}(\hat{\theta}_t)$.
Since the intermediate points $\tilde \theta_{t+1}$ are interiors, by Kolmogorov's strong law of large numbers and the local Lipschitzness, for every $\epsilon >0$, we have a compact set $\Theta' = \{\theta: \|\theta - \tilde \theta_{t+1}\| \leq \epsilon\} \subseteq \Theta$. Therefore, we have the uniform convergence as follows:
\begin{equation*}
    \sup_{\theta \in \Theta'}\|\widehat{M}_{t}(\theta) - M_t(\theta)\| \xrightarrow{P} 0.
\end{equation*}
Since the map $G_\theta(y)$ is strongly monotone, the minimizer $\tilde \theta_{t+1}$ for $G_{\hat \theta_t}(y)$ is unique. Thus, let $\eta = \alpha\epsilon >0$, for every $\epsilon>0$ such that for every $\theta$ satisfies $\{\theta:\|\theta - \tilde \theta_{t+1}\| \geq \epsilon\}$, we have:
$$
\|M_t(\theta) - M_t(\tilde \theta_{t+1})\| \geq \alpha\|\theta-\tilde \theta_{t+1}\| \geq \eta.
$$
Denote the edge of the $\Theta'$ as $\partial\Theta' = \{\theta: \|\theta - \tilde \theta_{t+1}\| = \epsilon\}$.Therefore, the following inequality holds for $\theta \in \partial\Theta'$:
\begin{equation*}
\begin{split}
    & \qquad \inf_{\theta \in \partial\Theta'} \|\widehat{M}_{t}(\theta) - \widehat{M}_{t}(\tilde \theta_{t+1})\| \\
    & = \inf_{\theta \in \partial\Theta'} \|(\widehat{M}_{t}(\theta) - M_{t}(\theta)) + (M_{t}(\theta) - M_{t}(\tilde \theta_{t+1})) + (M_{t}(\tilde \theta_{t+1}) - \widehat{M}_{t}(\tilde \theta_{t+1}))\| \\
    & \geq \eta - 2\sup_{\theta \in \partial\Theta'} \|\widehat{M}_{t}(\theta) - M_{t}(\theta)\| \\
    & = \eta - o_p(1).
\end{split}
\end{equation*}
As for $\theta \in (\Theta')^c$, fix a point $\theta_1 = \tilde \theta_{t+1} + \frac{\theta - \tilde \theta_{t+1}}{\|\theta- \tilde \theta_{t+1}\|}\epsilon$ which is in the edge $\partial \Theta'$, so we have $\theta = \tilde \theta_{t+1} + \lambda(\theta_1 - \tilde \theta_{t+1})$ where $\lambda = \frac{\|\theta - \tilde \theta_{t+1}\|}{\epsilon} > 1$. By the strong monotonicity of $G_\theta(\cdot)$, we know that
\begin{equation*}
\begin{split}
    \langle \widehat{M}_{t}(\theta) - \widehat{M}_{t}(\tilde \theta_{t+1}), \theta -\tilde \theta_{t+1} \rangle &\geq \alpha \|\theta - \tilde \theta_{t+1}\|^2, \\
    \langle \widehat{M}_{t}(\theta_1) - \widehat{M}_{t}(\theta), \theta_1 - \theta \rangle &\geq \alpha \|\theta_1 - \theta\|^2.
\end{split}
\end{equation*}
We can simplify two inequalities as follows:
\begin{equation*}
\begin{split}
    \langle \widehat{M}_{t}(\theta) - \widehat{M}_{t}(\tilde \theta_{t+1}), \theta_1 -\tilde \theta_{t+1} \rangle &\geq \lambda \alpha \|\theta_1 - \tilde \theta_{t+1}\|^2, \\
    \langle \widehat{M}_{t}(\theta_1) - \widehat{M}_{t}(\theta), \theta_1 - \tilde \theta_{t+1} \rangle &\geq (1-\lambda) \alpha \|\theta_1 - \tilde \theta_{t+1}\|^2.
\end{split}
\end{equation*}
Add two inequalities together, and we get
\begin{equation*}
    \langle \widehat{M}_{t}(\theta_1) - \widehat{M}_{t}(\tilde \theta_{t+1}), \theta_1 -\tilde \theta_{t+1} \rangle \geq \alpha \|\theta_1 - \tilde \theta_{t+1}\|^2.
\end{equation*}
By the Cauchy-Schwarz inequality, we have inequalities based on the norm
\begin{equation}
\label{equ:consistency1}
\begin{split}
    \| \widehat{M}_{t}(\theta_1) - \widehat{M}_{t}(\tilde \theta_{t+1})\| &\geq \alpha \|\theta_1 - \tilde \theta_{t+1}\|,\\
    \| \widehat{M}_{t}(\theta) - \widehat{M}_{t}(\tilde \theta_{t+1})\| &\geq \lambda \alpha \|\theta_1 - \tilde \theta_{t+1}\|.
\end{split}
\end{equation}
Denote $\beta = \min_{1 \leq i \leq m} \beta_i$. By the Lipschitzness of the function $G_i$, we have
\begin{equation}
\label{equ:consistency2}
\begin{split}
    \|\widehat{M}_{t}(\theta_1) - \widehat{M}_{t}(\tilde \theta_{t+1})\| &= \|\frac{1}{N} \sum_{k = 1}^{N} G(\theta_1, Z_{k}) - \frac{1}{N} \sum_{k = 1}^{N} G(\tilde \theta_{t+1}, Z_{k})\| \\
    & \leq \frac{1}{N} \sum_{k = 1}^{N} \|G(\theta_1, Z_{k}) - G(\tilde \theta_{t+1}, Z_{k})\| \\
    & = \frac{1}{N} \sum_{k = 1}^{N} \sum_{i=1}^m \|G_i(\theta_1^i, \hat \theta_{t+1}^{-i}, Z_{k}^i) - G_i(\tilde \theta_{t+1}^i,\hat \theta_{t+1}^{-i}, Z_{k}^i)\|  \\
    & \leq \sum_{i=1}^m \beta_i \|\theta_1^i - \tilde \theta_{t+1}^i\|  \\
    & \leq \beta \sum_{i=1}^m \|\theta_1^i - \tilde \theta_{t+1}^i\|  \\
    & \leq \beta \|\theta_1 - \tilde \theta_{t+1}\|.
\end{split}
\end{equation}
Thus, we can derive the following inequality by (\ref{equ:consistency1}) and (\ref{equ:consistency2}):
$$
\|\widehat{M}_{t}(\theta) - \widehat{M}_{t}(\tilde \theta_{t+1})\| \geq \frac{\lambda \alpha}{\beta}\|\widehat{M}_{t}(\theta_1) - \widehat{M}_{t}(\tilde \theta_{t+1})\| \geq \frac{\lambda \alpha}{\beta}\eta - o_p(1).
$$
Therefore, there is no minimizer to $\widehat{M}_{t}(\theta)$ in the set $\{\theta:\|\theta - \tilde \theta_{t+1}\| \geq \epsilon\}$, so the consistency of the estimators holds for the iteration at time $t$ that
$$
\mathbb{P}(\|\hat \theta_{t+1} - \tilde \theta_{t+1}\| \geq \epsilon)= 0.
$$
\end{proof}

\begin{proof}[Proof of Theorem \ref{thm:CLT of theta, stable}]
We first prove the consistency of the estimator sequence $\{\hat \theta_t\}_{t=1}$ by induction.
For $t=0$, the initial point is fixed as $\hat{\theta}_0 = \theta_0$, so by Lemma \ref{lem:consistency of z-estimation}, we have
\begin{equation*}
        \hat{\theta}_{1} = \mathrm{\widehat{sol}}(\hat \theta_0) \xrightarrow{P}  \mathrm{sol}(\hat \theta_0) = \mathrm{sol}(\theta_0) = \theta_1.
\end{equation*}
Thus, the consistency for $\hat \theta_1$ holds.
Suppose that at iteration $t$, we have already proved that $\hat{\theta}_t \xrightarrow{P} \theta_t$, then for iteration $t+1$, by Lemma \ref{lem:consistency of z-estimation}, we have 
    \begin{equation*}
        \hat{\theta}_{t+1} = \mathrm{\widehat{sol}}(\hat \theta_t) \xrightarrow{P} \tilde \theta_{t+1} = \mathrm{sol}(\hat \theta_t).
    \end{equation*}
As we have proved in Proposition \ref{prop:existence and convergence}, we know that $\mathrm{sol}(\theta)$ is $C$-Lipschitz in $\theta$, which ensures its continuity with respect to $\theta$. By the Continuous Mapping Theorem, we have 
$$
\mathrm{sol}(\hat \theta_t) \xrightarrow{P} \mathrm{sol}(\theta_t).
$$
By the triangle inequality, we have the following inequality:
\begin{equation*}
\begin{split}
   \| \hat{\theta}_{t+1}  -  \theta_{t+1}\| &= \| \hat{\theta}_{t+1}  -  \mathrm{sol}(\theta_t)\| \\
   &= \| \hat{\theta}_{t+1}  -  \mathrm{sol}(\hat \theta_t) + \mathrm{sol}(\hat \theta_t) -\mathrm{sol}(\theta_t) \| \\
    & \leq \| \hat{\theta}_{t+1}  -  \mathrm{sol}(\hat \theta_t) \|+ \| \mathrm{sol}(\hat \theta_t) - \mathrm{sol}(\theta_t)\|.
\end{split}
\end{equation*}
Then take the probability on both sides:
\begin{equation*}
\begin{split}
     \mathbb{P}(\| \hat{\theta}_{t+1}  -  \theta_{t+1}\| \geq \epsilon) &= \mathbb{P}(\| \hat{\theta}_{t+1}  -  \mathrm{sol}(\theta_t)\| \geq \epsilon) \\
     &= \mathbb{P}(\|\hat{\theta}_{t+1}  -  \mathrm{sol}(\hat \theta_t) + \mathrm{sol}(\hat \theta_t) -\mathrm{sol}(\theta_t)\| \geq \epsilon) \\
    & \leq \mathbb{P}(\| \hat{\theta}_{t+1}  -  \mathrm{sol}(\hat \theta_t)\| \geq \frac{\epsilon}{2}) + \mathbb{P}(\| \mathrm{sol}(\hat \theta_t) -\mathrm{sol}(\theta_t)\| \geq \frac{\epsilon}{2}).
\end{split}
\end{equation*}
Taking the limit on both sides, we can obtain the consistency of the estimator sequence $\{\hat \theta_t\}_{t = 1}$:
\begin{equation*}
    \begin{split}
        & \quad\lim_{N \rightarrow \infty}\mathbb{P}(\| \hat{\theta}_{t+1}  -  \theta_{t+1}\| \geq \epsilon)  \\
         & \leq \lim_{N \rightarrow \infty}\mathbb{P}(\| \hat{\theta}_{t+1}  -  \mathrm{sol}(\hat \theta_t)\| \geq \frac{\epsilon}{2}) + \lim_{N \rightarrow \infty} \mathbb{P}(\| \mathrm{sol}(\hat \theta_t) -\mathrm{sol}(\theta_t)\| \geq \frac{\epsilon}{2})= 0.
    \end{split}
\end{equation*}

Now we turn to prove the asymptotic normality between $\hat{\theta}_t$ and $\theta_t$ by induction. As $\sqrt{N}(\hat{\theta}_{t} - \theta_{t}) = \sqrt{N}(\hat{\theta}_{t} - \tilde{\theta}_{t}) + \sqrt{N}(\tilde{\theta}_{t} - \theta_{t})$, we separate our proofs into two parts. 

For $t=1$, we choose an initial parameter $\theta_0 \triangleq \hat{\theta}_0$, so $\tilde \theta_1 = \theta_1$ holds, and $\sqrt{N}(\hat{\theta}_{1} - \tilde{\theta}_{1}) = \sqrt{N}(\hat{\theta}_{1} - \theta_{1})$.
By the Taylor expansion, we obtain the following equations at the first iteration :
$$
0 = \sum_{k=1}^N G(\hat \theta_1,Z_{k}) = \sum_{k=1}^N G(\theta_1, Z_{k}) + \sum_{k=1}^N \frac{\partial G(\theta_1, Z_k)}{\partial \theta^\top} (\hat \theta_1 - \theta_1) ,
$$
where $Z_{k} \sim \D(\theta_0)$. By the Law of Large Numbers, we have the following results that
$$
\frac{1}{N}\sum_{k=1}^N G(\theta_1, Z_{k}) \xrightarrow{P} 0,
$$
$$
\frac{1}{N} \sum_{k=1}^N \frac{\partial G(\theta_1, Z_k)}{\partial \theta^\top} \xrightarrow{P} V_{\theta_0}(\tilde \theta_1) = V_{\theta_0}(\theta_1)=\E_{Z \sim \D(\theta_0)}\left[\frac{\partial G(\theta_1, Z)}{\partial \theta^\top}\right].
$$
Therefore, by the central limit theorem, we have
\begin{equation*}
\begin{split}
    \sqrt{N}(\hat \theta_1 - \theta_1) &= -\left(\frac{1}{N} \sum_{k=1}^N \frac{\partial G(\theta_1, Z_k)}{\partial \theta^\top} \right)^{-1}\left(\sqrt{N} \cdot \frac{1}{N}\sum_{k=1}^N G(\theta_1, Z_{k})\right) \\
    &= -V_{\theta_0}(\theta_1)^{-1}\left(\sqrt{N} \cdot \frac{1}{N}\sum_{k=1}^N G(\theta_1, Z_{k})\right) + O_P\left(\frac{1}{\sqrt{N}}\right) \\
    &\xrightarrow{d} N(0, \Sigma_1),
\end{split}
\end{equation*}
where the covariance matrix is
$$
\Sigma_1 = V_{\theta_0}(\theta_1)^{-1} \E_{Z \sim \D(\theta_0)}\left(G(\theta_1, Z)G(\theta_1, Z)^\top \right) V_{\theta_0}(\theta_1)^{-1}.
$$

Suppose that for iteration $t-1$, we have already proved that $\sqrt{N}(\hat{\theta}_{t-1} - \theta_{t-1}) \xrightarrow{d} N(0, \Sigma_{t-1})$, then at the iteration $t$, denote 
$$
\mathbb{G}_{t}G(\theta,Z) = \sqrt{N}(\widehat{M}_{t}(\theta) - M_t(\theta)) = \sqrt{N}\left(\frac{1}{N} \sum_{k = 1}^{N} G(\theta, Z_{k}) - \mathbb{E}_{Z \sim  \mathcal{D}(\hat{\theta}_{t-1})} G(\theta, Z)\right),
$$
where $Z_{k} \sim  \mathcal{D}(\hat{\theta}_{t-1})$. By the \cite[Theorem 5.12]{van2000asymptotic}, we have the following convergence based on the consistency of $\hat \theta_t$ and the local Lipschitzness.
\begin{equation}
\label{equ:G consistency}
\mathbb{G}_tG(\hat \theta_t,Z) -\mathbb{G}_t G(\tilde \theta_t,Z) \xrightarrow{P} 0.
\end{equation}
Since $\hat \theta_{t}$ is the zero of $\widehat{M}_{t}(\theta)$ and $\tilde \theta_{t}$ is the zero of $M_t(\theta)$, we can rewrite the $\mathbb{G}_{t}G(\hat \theta_t,Z)$ as follows
$$
\mathbb{G}_{t}G(\hat \theta_t,Z) = \sqrt{N}(\widehat{M}_{t}(\hat \theta_t) - M_t(\hat \theta_t)) = \sqrt{N}(M_t(\tilde \theta_t) - M_t(\hat \theta_t)).
$$
By the first-term Taylor expansion, we have
$$
M_t(\hat \theta_t) = M_t(\tilde \theta_t) + V_t(\tilde \theta_t)(\hat \theta_t - \tilde \theta_t) + o(\|\hat \theta_t - \tilde \theta_t\|),
$$
where $V_{\hat{\theta}_{t-1}}(\tilde \theta_t) = \mathbb{E}_{Z \sim  \mathcal{D}(\hat{\theta}_{t-1})} \left[\frac{\partial G(\tilde \theta_t,Z)}{\partial \theta^\top}\right]$. Thus, by the equation (\ref{equ:G consistency}), we find that
$$
\mathbb{G}_{t}G(\tilde \theta_t,Z) + o_P(1) = \mathbb{G}_{t}G(\hat \theta_t,Z) = -\sqrt{N}V_{\hat{\theta}_{t-1}}(\tilde \theta_t)(\hat \theta_t - \tilde \theta_t) + \sqrt{N}o_P(\|\hat \theta_t - \tilde \theta_t\|).
$$
From the equality above, we have the inequality of its norm expression
\begin{equation*}
\begin{split}
    \sqrt{N}\|V_{\hat{\theta}_{t-1}}(\tilde \theta_t)(\hat \theta_t - \tilde \theta_t)\| &\leq \|\mathbb{G}_{t}G(\tilde \theta_t,Z)\| + o_P(1) + o_P(\sqrt{N}\|\hat \theta_t - \tilde \theta_t\|) \\
    & = O_P(1)+ o_P(\sqrt{N}\|\hat \theta_t - \tilde \theta_t\|).
\end{split}
\end{equation*}
Since $V_{\hat{\theta}_{t-1}}(\tilde \theta_t)$ is positive definite, it is invertible, so we have the following inequality
$$
\sqrt{N}\|\hat \theta_t - \tilde \theta_t\| \leq \|V_{\hat{\theta}_{t-1}}(\tilde \theta_t)^{-1}\| \sqrt{N}\|V_{\hat{\theta}_{t-1}}(\tilde \theta_t)(\hat \theta_t - \tilde \theta_t)\| = O_P(1) + o_P(\sqrt{N}\|\hat \theta_t - \tilde \theta_t\|).
$$
We can obtain the result
\begin{equation*}
\begin{split}
    \sqrt{N}(\hat \theta_t - \tilde \theta_t) &= - V_{\hat{\theta}_{t-1}}(\tilde \theta_t)^{-1}\mathbb{G}_{t}G(\tilde \theta_t,Z) + o_P(1) \\
    & = - V_{\hat{\theta}_{t-1}}(\tilde \theta_t)^{-1}\left(\frac{1}{\sqrt{N}} \sum_{k = 1}^{N} G(\tilde \theta_t, Z_{k})\right)+ o_P(1). 
\end{split}
\end{equation*}
By the central limit theorem, we obtain the conditional asymptotic normality based on the previous estimation:
$$
\sqrt{N}(\hat \theta_t - \tilde \theta_t) \mid \hat \theta_{t-1} \xrightarrow{d} N(0, \hat \Sigma),
$$
where the covariance matrix is
$$
\hat \Sigma = V_{\hat{\theta}_{t-1}}(\tilde \theta_t)^{-1} \E_{Z \sim \D(\hat \theta_{t-1})}\left(G(\tilde \theta_t, Z)G(\tilde \theta_t, Z)^\top\right) V_{\hat{\theta}_{t-1}}(\tilde \theta_t)^{-1}.
$$
Therefore, the conditional distribution for the estimator $\hat \theta_t - \theta_t$ at each $t$ is
$$
\sqrt{N}(\hat \theta_t - \theta_t)\mid \hat \theta_{t-1} \xrightarrow{d} N(\sqrt{N}(\tilde \theta_t -  \theta_t), \hat \Sigma) \triangleq N(\hat \mu, \hat \Sigma).
$$

Now we prove the distribution of $\sqrt{N}(\hat{\theta}_t - \theta_t)$ by characteristic function.
We denote $X_t = \sqrt{N}(\hat{\theta}_t - \theta_t)$, and the variance as
\begin{equation*}
    \Sigma = V_{\theta_{t-1}}(\theta_t)^{-1} \E_{Z \sim \D(\theta_{t-1})}\left(G(\theta_t, Z)G(\theta_t, Z)^\top\right) V_{\theta_{t-1}}(\theta_t)^{-1}.
\end{equation*}
The characteristic function of the condition distribution is:
$\phi_{X_{t} \mid X_{t-1}}(z) \xrightarrow{P} \exp\left\{iz^T\hat \mu - \frac{1}{2}z^T\hat \Sigma z\right\}$, and the distribution at $t-1$ can be described by characteristic function as:
\begin{equation*}
    \begin{split}
        \mathbb{P}(X_{t-1}) &= \frac{1}{(2\pi)^d} \int \phi_{X_{t-1}}(s)\cdot e^{-is^TX_{t-1}} ds \\
        &= \frac{1}{(2\pi)^d} \int \exp\left\{-\frac{1}{2}s^T\Sigma_{t-1}s-is^TX_{t-1}\right\} ds.
    \end{split}
\end{equation*}
Then we have the characteristic function
\begin{equation*}
    \begin{split}
        \phi_{X_t}(z) &= \mathbb{E}(e^{iz^TX_t}) = \mathbb{E}_{X_{t-1}}\left(\mathbb{E}(e^{iz^TX_t} \mid X_{t-1})\right) \\
        &= \frac{1}{(2\pi)^d}  \int \int \exp\left\{iz^T\hat \mu - \frac{1}{2}z^T\hat \Sigma z\right\} \exp\left\{-\frac{1}{2}s^T\Sigma_{t-1}s-is^TX_{t-1}\right\} ds dX_{t-1}. \\
    \end{split}
\end{equation*}
To simplify the formulation, we let $-\frac{1}{2}s^T\Sigma_{t-1}s-is^TX_{t-1} =  -\frac{1}{2}(s-A_1)^TM_1(s-A_1) + B_1$, and by comparing the terms, we have:
\begin{equation*}
    \begin{split}
        &M_1 = \Sigma_{t-1}, \\
        &A_1 = M_1^{-1}i X_{t-1} =  \Sigma_{t-1}^{-1}i X_{t-1},\\
        &B_1 = \frac{1}{2}A_1^TM_1A_1 = -\frac{1}{2}X_{t-1}^T\Sigma_{t-1}^{-1}X_{t-1}.
    \end{split}
\end{equation*}
Thus, the characteristic function can be rewritten as
\begin{equation*}
    \begin{split}
        \phi_{X_t}(z) &= \frac{1}{(2\pi)^d}  \int \int \exp\left\{iz^T\hat \mu - \frac{1}{2}z^T\hat \Sigma z\right\} \exp\left\{-\frac{1}{2}(s-A_1)^TM_1(s-A_1) + B_1\right\} ds dX_{t-1} \\
        &= \frac{\det|\Sigma_{t-1}|}{(2\pi)^{d/2}}  \int \exp\left\{iz^T\hat \mu - \frac{1}{2}z^T\hat \Sigma z  -\frac{1}{2}X_{t-1}^T\Sigma_{t-1}^{-1}X_{t-1}\right\}dX_{t-1}.
    \end{split}
\end{equation*}
Since we have
\begin{equation*}
        \left|\exp\left\{iz^T\hat \mu - \frac{1}{2}z^T\hat \Sigma z  -\frac{1}{2}X_{t-1}^T\Sigma_{t-1}^{-1}X_{t-1}\right\}\right| = \exp\left\{ - \frac{1}{2}z^T\hat \Sigma z  -\frac{1}{2}X_{t-1}^T\Sigma_{t-1}^{-1}X_{t-1}\right\}
\end{equation*}
and $z^T\hat \Sigma z  >0$ for all $z \in \mathcal{Z}$, and therefore $-\frac{1}{2}z^T\hat \Sigma z  <0$ for all $z \in \mathcal{Z}$, the exponential term is bounded:
\begin{equation*}
    \left|\exp\left\{iz^T\hat \mu - \frac{1}{2}z^T\hat \Sigma z  -\frac{1}{2}X_{t-1}^T\Sigma_{t-1}^{-1}X_{t-1}\right\}\right| \leq \left|\exp\left\{-\frac{1}{2}X_{t-1}^T\Sigma_{t-1}^{-1}X_{t-1}\right\}\right|.
\end{equation*}
Besides, by the (\ref{equ:solution map lipschitz}) in Proof of Proposition \ref{prop:existence and convergence}, $\mathrm{sol}(\theta)$ is $C$-Lipschitz in $\theta$. Denote $J_{sol}(\theta)$ as the Jacobian matrix of the map $\mathrm{sol}(\theta)$, so we have the following convergence by the Taylor expansion:
\begin{equation*}
        \sqrt{N}(\tilde{\theta}_{t} - \theta_{t}) = \sqrt{N}\left(\mathrm{sol}(\hat \theta_{t-1}) - \mathrm{sol}(\theta_{t-1})\right) \rightarrow J_{sol}(\theta_{t-1})X_{t-1},
\end{equation*}
as $N \rightarrow \infty$. Thus, by the control convergence theorem, we have:
\begin{equation*}
    \begin{split}
        \lim_{N\rightarrow \infty}\phi_{X_t}(z) &= \frac{\det|\Sigma_{t-1}|}{(2\pi)^{d/2}}  \int \lim_{N \rightarrow \infty}\exp\left\{iz^T\hat \mu - \frac{1}{2}z^T\hat \Sigma z  -\frac{1}{2}X_{t-1}^T\Sigma_{t-1}^{-1}X_{t-1}\right\}dX_{t-1}  \\
        &= \frac{\det|\Sigma_{t-1}|}{(2\pi)^{d/2}}  \int \exp\left\{iz^T J_{sol}(\theta_{t-1}) X_{t-1}- \frac{1}{2}z^T\Sigma z  -\frac{1}{2}X_{t-1}^T\Sigma_{t-1}^{-1}X_{t-1}\right\}dX_{t-1}.
    \end{split}
\end{equation*}
Similarly, we let $iz^T J_{sol}(\theta_{t-1}) X_{t-1} -\frac{1}{2}X_{t-1}^T\Sigma_{t-1}^{-1}X_{t-1} =  -\frac{1}{2}(X_{t-1}-A_2)^TM_2(X_{t-1}-A_2) + B_2$, and by comparing the terms, we have:
\begin{equation*}
    \begin{split}
        &M_2 = \Sigma_{t-1}^{-1}, \\
        &A_2 = i M_2^{-1}J_{sol}(\theta_{t-1})^T z =  i\Sigma_{t-1}J_{sol}(\theta_{t-1})^T z,\\
        &B_2 = \frac{1}{2}A_2^TM_2A_2 = -\frac{1}{2}z^TJ_{sol}(\theta_{t-1})\Sigma_{t-1}J_{sol}(\theta_{t-1})^Tz.
    \end{split}
\end{equation*}
Then the limit of the characteristic function is
\begin{equation*}
    \begin{split}
        \lim_{N \rightarrow \infty}\phi_{X_t}(z)  &= \frac{\det|\Sigma_{t-1}|}{(2\pi)^{d/2}}  \int \exp\left\{iz^T J_{sol}(\theta_{t-1})X_{t-1}- \frac{1}{2}z^T\Sigma z  -\frac{1}{2}X_{t-1}^T\Sigma_{t-1}^{-1}X_{t-1}\right\}dX_{t-1}  \\
        &= \frac{\det|\Sigma_{t-1}|}{(2\pi)^{d/2}}  \int \exp\left\{- \frac{1}{2}z^T\Sigma z -\frac{1}{2}(X_{t-1}-A_2)^TM_2(X_{t-1}-A_2) + B_2\right\}dX_{t-1}  \\
        &= \exp\left\{- \frac{1}{2}z^T(\Sigma +J_{sol}(\theta_{t-1})\Sigma_{t-1}J_{sol}(\theta_{t-1})^T)z \right\}, 
    \end{split}
\end{equation*}
which means the distribution function of $X_t = \sqrt{N}(\hat{\theta}_t - \theta_t)$ is a normal distribution with zero mean and variance $\Sigma_t = \Sigma +J_{sol}(\theta_{t-1})\Sigma_{t-1}J_{sol}(\theta_{t-1})^T$.
\end{proof}

\subsubsection{Proof of Theorem \ref{thm:consistenct of estimated cov, stable theta}}

\begin{proof}[Proof of Theorem \ref{thm:consistenct of estimated cov, stable theta}]
    The proof follows from the same process as in the Proof of Theorem \ref{thm:CLT of theta, stable}. Since $Z_k$ are i.i.d. from $\D(\hat \theta_{t-1})$ and $Z_j$ are i.i.d. from $\D_{\hat \beta}(\hat \theta_{t-1})$, the law of large numbers, together with the second step in establishing the marginal asymptotic distribution of $\hat \theta_t$, directly yields
    $$ 
    \widehat \Sigma_t^{\hat \beta} \xrightarrow{P} \Sigma_t^{\hat \beta}.
    $$
    Furthermore, if the distribution atlas $\mathcal{D}_{\mathcal{B}} = \{\mathcal{D}_{\beta} \}_{\beta \in \mathcal{B}}$ contains the true distribution map, which can be parameterized by some $\beta^*$, and the fitted $\hat \beta$ converge to $\beta^*$, then applying the same law of large numbers argument together with the second step of the marginal asymptotic analysis yields the consistency $\widehat \Sigma_t^{\hat \beta} \xrightarrow{P} \Sigma_t$.
\end{proof}

\subsubsection{Proof of Theorem \ref{thm:lower_stable}}

To derive the semiparametric lower bound for $\theta_t$, it is crucial to find the hardest parametric sub-model in $\mathscr{D}$. However, the observed data consists of samples from different sources, and different players may have distinct sample sizes in each round. To take the distinct sample sizes into account, motivated by the missing data literature, we introduce another independent index variable $R\in[t]\times[m]$ with $\Prob(R=(j,i))=q_{j,i}\approx\frac{N_j^i}{\sum_{j\in[t],i\in[m]}N_j^i}$ and $P_{W|R=(j,i)}=\D_i(\theta_{j-1})$. When the observed sample $Z_{j,k}^i$ is from $\D_i(\theta_j)$, we formulate the observed data $\{(Z_{j,k}^i,(j,i)):j\in[t],i\in[m],k\in[N_j^i]\}$ as $\sum_{j\in[t],i\in[m]}N_j^i$ i.i.d. samples from $P_{W,R}$. Then the parameter $\theta_t$ can be viewed as a functional $P_{W,R}$. Note that the observed data is not generated in an i.i.d. manner, since samples under $\theta_j$ are always generated after samples under $\theta_{j-1}$. However, the formulation $P_{W,R}$ helps identify the hardest parametric sub-model. It is worth mentioning that the introduction of $P_{W,R}$ is merely for motivating the hardest sub-model, and our proof of Theorem \ref{thm:lower_stable} doesn't rely on this data-generating process.

We consider the distribution space $\mathscr{P}$ as the set of distributions $\tilde P_{W,R}$ such that $\tilde P_{W|R=(j,i)}=\tilde\D_i(\tilde \theta_j)$ for some $\tilde\D_{[m]}$ that satisfies Assumptions \ref{asm:existence and convergence} and \ref{asm:CLT stable}, and $\tilde\theta_j$'s are defined based on $\tilde\D_{[m]}$,
\begin{equation}\label{eq:dist_space_stable}
\begin{aligned}
    \mathscr{P}=\big\{\tilde P_{W,R}:&P_R((j,i))=q_{j,i},\tilde P_{W|R=(j,i)}=\tilde\D_i(\tilde\theta_j),\text{ $\tilde\D_{[m]}$ satisfies Assumptions \ref{asm:existence and convergence} and \ref{asm:CLT stable}} \\  & \text{for some $\tilde\theta_i$ and $\tilde\alpha$, and $\tilde\theta_j$'s are defined recursively based on $\tilde\D_{[m]}$}\big\}.
\end{aligned}\end{equation}
The following lemma characterizes the efficient influence function (EIF) of $\theta_t$ for $P_{W,R}$ in the distribution space $\mathscr{P}$.
\begin{lemma}(\textbf{EIF})\label{lem:eif_stable}
    Under Assumption \ref{asm:compact_stable}, the EIF of $\theta_t$ at $P_{W,R}$ in the distribution space $\mathscr{P}$ is 
    \[\Psi_t(W,R)=-\sum_{j\in[t]}\bigg(\prod_{l=j}^{t-1}\nabla_\theta^\top\mathrm{sol}(\theta_l)\bigg)\bigg\{\E_{Z\sim\D(\theta_{j-1})}\nabla_\theta^\top G\big(\theta_j,Z\big)\bigg\}^{-1}\tilde G_j\big(\theta_j,W,R\big),\]
    where $\D(\theta_{j-1})=\prod_{i\in[m]}\D_i(\theta_{j-1})=\prod_{i\in[m]}P_{W|R=(j,i)}$ is the product distribution and 
    \[\tilde G_j(\theta,W,R)=\bigg(\frac{\bm{1}\{R=(j,1)\}}{q_{j,1}}G_1^\top(\theta,W),\ldots,\frac{\bm{1}\{R=(j,m)\}}{q_{j,m}}G_m^\top(\theta,W)\bigg)^\top
    \]
    is the concatenation of weighted gradients.
\end{lemma}

\begin{proof}[Proof of Theorem \ref{thm:lower_stable}]
    Motivated by Lemma \ref{lem:eif_stable}, we choose the score functions $s_{j,i}(Z^i)$ as
    \[s_{j,i}(Z^i)=-\bigg(\prod_{l=j}^{t-1}\nabla_\theta^\top\mathrm{sol}(\theta_l)\bigg)\bigg\{\E_{Z\sim\D(\theta_{j-1})}\nabla_\theta^\top G(\theta_j,Z)\bigg\}^{-1}\tilde G_j\big(\theta_j,Z^i,(j,i)\big), \quad j\in[t],i\in[m],\]
    where $\tilde G$ is defined in Lemma \ref{lem:eif_stable} with $q_{j,i}=\frac{1/\mu_{t,j}^i}{\sum_{\tilde j\in[t],\tilde i\in[m]}1/\mu_{t,\tilde j}^{\tilde i}}$. It follows from the proof of Lemma \ref{lem:eif_stable} that there exists functions $s_i(\theta,Z^i)$ that satisfy $s_i(\theta_{j-1},Z^i)=s_{j,i}(Z^i)$.

    Define the sub-model $\{\D_i^u:u\in\R^d,\|u\|_2\le 1\}$ as
    \[\frac{d\D_i^u(\theta)}{d\D_i(\theta)}(Z^i)=\frac{K(\frac{1}{\sqrt{N_t}}u^\top s_i(\theta,Z^i))}{C_i^u(\theta)},\quad K(x)=\frac{2}{1+e^{-2x}},\quad C_i^u(\theta)=\E_{\D_i(\theta)}K\bigg(\frac{1}{\sqrt{N_t}}u^\top s_i(\theta,Z)\bigg).\]
    The proof of Lemma \ref{lem:eif_stable} ensures $\D^u_{[m]}$ are in the admissible space $\mathscr{D}$ for $N_t$ large enough. For any regular estimators $\{\hat\theta_{j-1}:j\in[t]\}$ defined in Definition \ref{def:regular_stable}, recall that $P_t^u$ is the joint distribution of all the data $\bm S_{[t]}$. Similar to the proof of \cite[Lemma 22]{cutler2024stochasticapproximationdecisiondependentdistributions}, we can show
    \[\log\frac{dP_t^u}{dP_t}(\bm S_{[t]})=\frac{1}{\sqrt{N_t}}\sum_{j\in[t],i\in[m],k\in[N_j^i]}u^\top s_i(\hat\theta_{j-1},Z_{j,k}^i)-\frac{1}{2}\sum_{j\in[t],i\in[m]}\frac{1}{\mu_{t,j}^i}u^\top\Cov_{\D_i(\theta_{j-1})}\big(s_i(\theta_{j-1},Z^i)\big)u+o_{P_t}(1),\]
    \[\frac{1}{\sqrt{N_t}}\sum_{j\in[t],i\in[m],k\in[N_j^i]}s_i(\hat\theta_{j-1},Z_{j,k}^i)\overset{P_t}{\rightsquigarrow}N(0,\Sigma_t),\]
    with the covariance matrix
    \begin{align*}
        \Sigma_t=&\sum_{j\in[t],i\in[m]}\frac{1}{\mu_{t,j}^i}\Cov_{\D_i(\theta_{j-1})}\big(s_i(\theta_{j-1},Z^i)\big)\\
        =&\sum_{j\in[t]}\bigg(\prod_{k=j}^{t-1}\nabla_\theta^\top\mathrm{sol}(\theta_l)\bigg)\bigg\{\E_{\D(\theta_{j-1})}\nabla_\theta^\top G\big(\theta_j,Z\big)\bigg\}^{-1}\mathrm{diag}\bigg\{\mu_{t,j}^i\Cov_{\D_i(\theta_{j-1})}\big(G_i(\theta_j,Z^i)\big):i\in[m]\bigg\}\\
        &\cdot\bigg\{\E_{\D(\theta_{j-1})}\nabla_\theta^\top G\big(\theta_j,Z\big)\bigg\}^{-\top}\bigg(\prod_{k=j}^{t-1}\nabla_\theta^\top\mathrm{sol}(\theta_l)\bigg)^\top.
    \end{align*}
    Under the regularity condition and the local asymptotically normality, we can apply the convolution theorem \citep[e.g.][Theorem 10.3]{li2019graduate} to conclude the results.
\end{proof}

\begin{lemma}\label{lem:interior_stable}
    Under assumption \ref{asm:compact_stable}, we assume $\{\theta_j^i:j\in[t]\}$ is in the interior of $\Theta_i$ for $i\in[m]$. For some parametric sub-model $P_{W,R}^u$ in $\mathscr{P}$ with $P_{W,R}^0=P_{W,R}$, we assume $P_{W|R=(j,i)}^u=\D_i^u(\theta_{j-1}^{(u)})$ and $P_{W,R}^u$ is differentiable in quadratic mean (DQM) for all $\|u\|_2\le\delta$ small enough. Then $\{\theta_j^{(u),i}:j\in[t]\}$ is also in the interior of $\Theta_i$ for $\|u\|_2$ small enough.
\end{lemma}
\begin{proof}[Proof of Lemma \ref{lem:interior_stable}]
    We prove by induction. 
    
    In the first round of the game, the objective function for the $i$th player satisfies that for $\|u_1\|_2\vee\|u_2\|_2\le\delta$,
    \begin{align*}
        &|\E_{\D_i^{u_1}(\theta_0)}\ell_i(\theta,Z^i)-\E_{\D_i^{u_2}(\theta_0)}\ell_i(\theta,Z^i)|\\
        =&\bigg|\int \bigg(p_i^{u_1}(\theta_0,Z^1)-p_i^{u_2}(\theta_0,Z^1)\bigg)\ell_i(\theta,Z^1)\bigg|\\
        \le&\bigg|\int\bigg(\sqrt{p_i^{u_1}(\theta_0,Z^1)}-\sqrt{p_i^{u_2}(\theta_0,Z^1)}\bigg)^2d\mu(Z^1)\int\bigg(\sqrt{p_i^{u_1}(\theta_0,Z^1)}+\sqrt{p_i^{u_2}(\theta_0,Z^1)}\bigg)^2\big(\ell_i(\theta,Z^1)\big)^2d\mu(Z^1)\bigg|^{\frac{1}{2}}\\
        \lesssim&\|u_1-u_2\|_2(1+o(1)),
    \end{align*}
    where the last inequality is due to the DQM of the sub-model and the boundedness of the continuous loss $\ell_i$ on the compact set $\Theta\times\Zcal_i$. Moreover, the boundedness and continuity of $\ell_i(\theta,Z^i)$ in $\theta$ together with the dominated convergence theorem imply that the objective function is also continuous in $\theta$. Therefore, the objective functions in the first round are continuous in $u$ and $\theta$. Then Corollary 3.6 in \cite{feinstein2022continuity} together with the uniqueness in Proposition \ref{prop:existence and convergence} imply that $\theta_1^{(u)}$ is continuous in $u$ for $\|u\|_2<\delta$. 

    For the $t$-th round, we assume $\theta_{t-1}^{(u)}$ is continuous in $u$ for $\|u\|_2<\delta$, then
    \begin{align*}
        &|\E_{\D_i^{u_1}(\theta_{t-1}^{(u_1)})}\ell_i(\theta,Z^i)-\E_{\D_i^{u_2}(\theta_{t-1}^{(u_2)})}\ell_i(\theta,Z^i)|\\
        \le&|\E_{\D_i^{u_1}(\theta_{t-1}^{(u_1)})}\ell_i(\theta,Z^i)-\E_{\D_i^{u_1}(\theta_{t-1}^{(u_2)})}\ell_i(\theta,Z^i)|+|\E_{\D_i^{u_1}(\theta_{t-1}^{(u_2)})}\ell_i(\theta,Z^i)-\E_{\D_i^{u_2}(\theta_{t-1}^{(u_2)})}\ell_i(\theta,Z^i)|\\
        \lesssim&\|\theta_{t-1}^{(u_1)}-\theta_{t-1}^{(u_2)}\|_2+\|u_1-u_2\|_2(1+o(1)),
    \end{align*}
    where the last inequality is due to 1) the sensitivity of $\D_i^u$ in Assumption \ref{asm:existence and convergence} due to the definition of $\mathscr{P}$ in \eqref{eq:dist_space_stable}, and 2) the Lipschitz property of the continuous loss $\ell_i$ in $Z$ on $\Theta\times\Zcal_i$. Therefore, the objective functions in the $t$th round are continuous. Similarly, we have $\theta_t^{(u)}$ is continuous in $u$.

    Since $\theta_t$ is in the interior of $\Theta$, the continuity of $\theta_t^{(u)}$ implies that $\theta_t^{(u)}$ is in the interior of $\Theta$ for $\|u\|$ small enough.
\end{proof}
\begin{proof}[Proof of Lemma \ref{lem:eif_stable}]
    The proof consists of two parts. Firstly, we show $\Psi_t$ is an influence function. Then, we prove $\Psi_t$ is in the tangent space of $\mathscr{P}$.

    We start by proving $\Psi_t$ is an influence function. For any parametric sub-models $P_{W,R}^u$ as described in Lemma \ref{lem:interior_stable}, since $P_R^u((j,i))=q_{j,i}$ is fixed, we know the score function at $u=0$ must have the form
    \begin{equation}\label{eq:score_stable}
        \frac{\partial}{\partial u}\log\frac{dP_{W,R}^u}{dP_{W,R}}(W,R)\bigg|_{u=0}=s(W,R)=\sum_{j\in[t],i\in[m]}\bm{1}(R=(j,i))s_{j,i}(W),
    \end{equation}
    with $s_{j,i}$ to be the score function of $P_{W|R=(j,i)}^u=\D_i^u(\theta_{j-1}^{(u)})$. Then it follows from the definition of $\theta_t^{(u)}$ and Lemma \ref{lem:interior_stable} that
    \[\E_{Z\sim\D^u(\theta_{t-1}^{(u)})}G(\theta_t^{(u)},Z)=0,\]
    where $\D^u(\theta_{t-1}^{(u)})=\prod_{i\in[m]}\D_i^u(\theta_{t-1}^{(u)})$ is the product measure of $Z=(Z^{1\top},\ldots,Z^{m\top})^\top$ for $Z^i\sim\D_i^u(\theta_{t-1}^{(u)})$. Taking the derivative on both sides implies that
    \begin{align*}
        \frac{\partial\theta_t^{(u)}}{\partial u^\top}\bigg|_{u=0}=&-\bigg\{\E_{\D(\theta_{t-1})}\nabla_\theta^\top G(\theta_t,Z)\bigg\}^{-1}\E_{\D(\theta_{t-1})}G(\theta_t,Z)\bigg\{\sum_{i\in[m]}s_{t,i}(Z^i)+\nabla_\theta^\top p(\theta_{t-1},Z)\frac{\partial\theta_{t-1}^{(u)}}{\partial u^\top}\bigg|_{u=0}\bigg\}\\
        =&-\sum_{j\in[m]}\bigg(\prod_{l=j}^{t-1}\nabla_\theta^\top\mathrm{sol}(\theta_l)\bigg)\bigg\{\E_{\D(\theta_{j-1})}\nabla_\theta^\top G(\theta_j,Z)\bigg\}^{-1}\E_{\D(\theta_{j-1})}G(\theta_j,Z)\sum_{i\in[m]}s_{j,i}^\top(Z^i)\\
        =&-\E_{P_{W,R}}\sum_{j\in[m]}\bigg(\prod_{l=j}^{t-1}\nabla_\theta^\top\mathrm{sol}(\theta_l)\bigg)\bigg\{\E_{\D(\theta_{j-1})}\nabla_\theta^\top G(\theta_j,Z)\bigg\}^{-1}\tilde G_j(\theta_j,W,R)s^\top(W,R).
    \end{align*}
    Therefore, $\Psi_t$ is an influence function. 
    
    Then we show elements of $\Psi_t$ are in the tangent space of $\mathscr{P}$ at $P_{W,R}$. Clearly $\Psi_t$ has the form of \eqref{eq:score_stable}. Define $\D_i^u$ as
    \[\frac{d\D_i^u(\theta)}{d\D_i(\theta)}(Z^i)=\frac{K(u^\top s_i(\theta,Z^i))}{C^u(\theta)},\quad K(x)=\frac{2}{1+e^{-2x}},\quad C_i^u(\theta)=\E_{\D_i(\theta)}K(u^\top s_i(\theta,Z^i)),\]
    for some $s_i(\theta,Z^i)$ satisfying
    \[s_i(\theta_{j-1},Z^i)=\Psi_t(Z^i,(j,i)).\]
    It suffices to show $\D_{[m]}^u$ satisfies Assumption \ref{asm:existence and convergence} and \ref{asm:CLT stable} for $\|u\|_2$ small enough. 

    Firstly, we show $\theta_0,\ldots,\theta_{t-1}$ are all different if $\theta_{t-1}\ne\theta_{PS}$. To see this, if $\theta_l=\theta_k$ for some $l<k<t$, then the value of $\theta_l$ appears infinitely often in the sequence $\{\theta_j:j\ge 0\}$. Since $\theta_j\rightarrow \theta_{PS}$ by Proposition \ref{prop:existence and convergence}, we know $\theta_l=\theta_{PS}$ and thus $\theta_k=\theta_{PS}$ for all $k\ge l$. This implies $\theta_{t-1}=\theta_{PS}$, contradicting the assumption $\theta_{t-1}\ne\theta_{PS}$. Therefore, $\theta_0,\ldots,\theta_{t-1}$ are all different.

    Then we construct the scores $s_i(\theta,Z^i)$. For $\theta=\sum_{j\in[t]}\lambda_j\theta_{j-1}$ in the linear span of $\{\theta_0,\ldots,\theta_{t-1}\}$, we set $s_i(\theta,Z^i)=\sum_{j\in[t]}\lambda_js_i(\theta_{j-1},Z^i)$, which is a linear function of $\theta$ inside the linear space. Then we have $s_i(\theta,Z^i)$ is differentiable in $\theta$ along any directions inside the linear space, and the gradients are spanned by $s_i(\theta_j,Z^i)$'s. Since $\ell_i$'s are twice continuous differentiable in $\theta$ on the compact set $\Theta\times\Zcal_i$, we know $s_i(\theta_{j-1},Z^i)$'s are bounded. Therefore, $s_i(\theta,Z^i)$ is Lipschitz. Finally, we can extend $s_i(\theta,Z^i)$ to $\Theta$ by defining $s_i(\theta,Z^i)=s_i(\tilde\theta,Z^i)$ with $\tilde\theta$ to be the linear projection of $\theta$ onto the linear space spanned by $\{\theta_0,\ldots,\theta_{t-1}\}$. Then $s_i(\theta,Z)$ is linear and Lipschitz differentiable in $\theta\in\Theta$.

    Then we verify Assumption \ref{asm:existence and convergence} and \ref{asm:CLT stable} respectively.

    \noindent \textbf{Assumption \ref{asm:existence and convergence}.1: $\tilde\epsilon_i$-sensitivity}
    
    Note that $\sup_{x\in\R}|\nabla K(x)|=1$, then
    \[|C_i^u(\theta)-1|=\big|\E_{\D_i(\theta)}K(u^\top s_i(\theta,Z^i))-1\big|\le\E_{\D_i(\theta)}|u^\top s_i(\theta,Z)|=O(\|u\|_2).\]
    Since $\nabla_i\ell(\theta,Z^i)$ is Lipschitz in $Z^i$ according to Assumption \ref{asm:existence and convergence}, we know $s_i(\theta,Z^i)$ is also Lipschitz in $Z^i$. Therefore $K(u^\top s_i(\theta,Z^i))$ is $O(\|u\|_2)$-Lipschitz in $Z$. Since $\E_{\D_i(\theta)}f(Z)$ is $\epsilon_i$-Lipschitz in $\theta$ for any 1-Lipschitz function $f$ by Assumption \ref{asm:existence and convergence}, we know
    \begin{equation}\label{eq:sensitive_gradient}
        \|\nabla_\theta\E_{\D_i(\theta)}f(Z^i)\|_2=\|\E_{\D_i(\theta)}f(Z^i)\nabla_\theta\log p_i(\theta,Z^i)\|_2\le\epsilon_i,
    \end{equation}
    where $p_i(\theta,Z^i)$ is the density of $\D_i(\theta)$. Then we have
    \begin{align*}
        \|\nabla_\theta C^u_i(\theta)\|_2\le&\|\E_{\D_i(\theta)}\nabla_\theta K(u^\top s_i(\theta,Z^i))\|_2+\|\E_{\D_i(\theta)}K(u^\top s_i(\theta,Z^i))\nabla_\theta\log p_i(\theta,Z)\|_2\\
        \le&\|u\|_2\E_{\D_i(\theta)}\|\nabla_\theta s_i(\theta,Z^i)\|_2+O(\|u\|_2)\\
        =&O(\|u\|_2).
    \end{align*}
    For any 1-Lipschitz function $f$, we know $f$ is bounded on $\Zcal_1$, and 
    \begin{align*}
        &\bigg\|\nabla_{Z^i}f(Z^i)\frac{K(u^\top s_i(\theta,Z^i))}{C_i^u(\theta)}\bigg\|_2\\
        \le&\bigg\|\frac{K(u^\top s_i(\theta,Z^i))}{C_i^u(\theta)}\nabla_{Z^i}f(Z^i)\bigg\|_2+\bigg\|f(Z^i)\frac{\nabla_{Z^i}K(u^\top s_i(\theta,Z^i))}{C_i^u(\theta)}\bigg\|_2\\
        \le&1+O(\|u\|_2).
    \end{align*}
    then for $\|u\|_2$ small enough, we have
    \begin{align*}
        &\|\nabla_\theta\E_{\D_i^u(\theta)}f(Z^i)\|_2\\
        =&\bigg\|\nabla_\theta\E_{\D_i(\theta)}f(Z^i)\frac{K(u^\top s_i(\theta,Z^i))}{C_i^u(\theta)}\bigg\|_2\\
        \le&\bigg\|\E_{\D_i(\theta)}f(Z^i)\frac{\nabla_\theta K(u^\top s_i(\theta,Z^i))}{C_i^u(\theta)}\bigg\|_2+\bigg\|\E_{\D_i(\theta)}f(Z^i)\frac{K(u^\top s_i(\theta,Z^i))\nabla_\theta C_i^u(\theta)}{(C_i^u(\theta))^2}\bigg\|_2\\
        &+\bigg\|\E_{\D_i(\theta)}f(Z^i)\frac{K(u^\top s_i(\theta,Z^i))}{C_i^u(\theta)}\nabla_\theta\log p_t(\theta,Z)\bigg\|_2\\
        \le&\sup_{f\in\mathrm{Lip}_1}\|\nabla_\theta\E_{\D_i(\theta)}f(Z^i)\|_2+O(\|u\|_2)\\
        \le&\epsilon_i+O(\|u\|_2).
    \end{align*}

    \noindent  \textbf{Assumption \ref{asm:existence and convergence}.2: $\tilde\alpha$-strong monotonicity}

    For every $\theta,\theta',\theta''\in\Theta$, we have
    \begin{align*}
        &\langle \E_{\D^u(\theta)} G(\theta',Z)-G(\theta'',Z),\theta'-\theta''\rangle\\
        =&\E_{\D(\theta)}\prod_{i\in[m]}\frac{K(u^\top s_i(\theta,Z^i))}{C_i^u(\theta)}\langle G(\theta',Z)-G(\theta'',Z),\theta'-\theta''\rangle\\
        \ge&\E_{\D(\theta)}\langle G(\theta',Z)-G(\theta'',Z),\theta'-\theta''\rangle-\E_{\D(\theta)}\bigg|\prod_{i\in[m]}\frac{K(u^\top s_i(\theta,Z^i))}{C_i^u(\theta)}-1\bigg|\|G(\theta',Z)-G(\theta'',Z)\|_2\|\theta'-\theta''\|_2.
    \end{align*}
    Since $\Theta\times\Zcal_1\times\ldots\times\Zcal_m$ is compact and $\ell_i$'s are twice continuously differentiable, we know $\|\nabla_\theta\ell_i(\theta,Z^i)\|_2$ and $\|\nabla_\theta^2\ell_i(\theta,Z^i)\|_{\rm sp}$ are bounded on $\Theta\times\Zcal_i$. Then $|K(u^\top s_i(\theta,Z^i))-1|\le|u^\top s_i(\theta,Z^i)|\lesssim\|u\|_2$, 
    \begin{equation}\label{eq:G_lipschitz}
        \|G(\theta',Z)-G(\theta'',Z)\|_2^2=\sum_{i\in[m]}\|\nabla_i\ell_i(\theta',Z^i)-\nabla_i\ell_i(\theta'',Z^i)\|_2^2\lesssim\|\theta'-\theta''\|_2^2,
    \end{equation}
    and therefore,
    \[\E_{\D(\theta)}\bigg|\prod_{i\in[m]}\frac{K(u^\top s_i(\theta,Z^i))}{C_i^u(\theta)}-1\bigg|\|G(\theta',Z)-G(\theta'',Z)\|_2\|\theta'-\theta''\|_2=O(\|u\|_2)\|\theta'-\theta''\|_2^2.\]
    Consequently, we get the $\alpha$-strong monotonicity for $\|u\|_2$ small enough
    \[\langle \E_{\D^u(\theta)} G(\theta',Z)-G(\theta'',Z),\theta'-\theta''\rangle>(\alpha-O(\|u\|_2))\|\theta'-\theta''\|_2^2.\]

    \noindent  \textbf{Assumption \ref{asm:existence and convergence}.4: Compatibility}
    
    Since $\sum_{i\in[m]}\big(\frac{\epsilon_i\beta_i}{\alpha}\big)^2<1$, we know $\sum_{i\in[m]}\big(\frac{(\epsilon_i+O(\|u\|_2))\beta_i}{\alpha-O(\|u\|_2)}\big)^2<1$ if $\|u\|_2$ is small enough.
    
    \noindent \textbf{Assumption \ref{asm:CLT stable}.1: Local Lipschitzness}

    This follows from \eqref{eq:G_lipschitz}.

    \noindent  \textbf{Assumption \ref{asm:CLT stable}.2: Bounded Jacobian}

    Since $\Theta\times\Zcal_1\times\ldots\times\Zcal_m$ is compact and $\ell_i$'s are twice continuously differentiable, we know $\|\nabla_\theta\ell_i(\theta,Z^i)\|_2$ is bounded on $\Theta\times\Zcal_i$. Consequently, we have $\|G(\theta,Z)\|_2$ is also bounded on $\Theta\times\Zcal_1\times\ldots\times\Zcal_m$.

    \noindent \textbf{Assumption \ref{asm:CLT stable}.3: Differentiable}

    Note that
    \[\E_{\D_i^u(\theta)} G_i(\theta',Z^i)=\E_{\D_i(\theta)}G_i(\theta',Z^i)\frac{K(u^\top s_i(\theta,Z^i))}{C_i^u(\theta)}.\]
    Since $G_i(\theta',Z^i)$ is differentiable in $\theta'$ and $\frac{K(u^\top s_i(\theta,Z^i))}{C_i^u(\theta)}\nabla_\theta G_i(\theta',Z^i)$ is bounded on $\Theta\times\Theta\times\Zcal_i$, it follows from dominated convergence theorem that $\E_{\D_i^u(\theta)} G_i(\theta',Z_i)$ is differentiable in $\theta'$.
\end{proof}

\subsection{Nash equilibria}

\subsubsection{Proof of Theorem \ref{thm:CLT of beta}}

\begin{proof}[Proof of Theorem \ref{thm:CLT of beta}]
    We first prove the consistency of the estimator $\hat{\beta}_i$. By Lemma \ref{lem:inter CLT for beta}, each  cross-fitted estimator $\hat \beta^{(j)}_i$ is consistent, so the final estimator is also consistent
    $$
    \hat\beta_i=\sum_{j\in[3]}\frac{|\mathcal{M}_j|}{N}\hat\beta_i^{(j)} \xrightarrow{P} \beta_i^*.
    $$
    Now we turn to the proof of asymptotic normality. According to Lemma \ref{lem:inter CLT for beta}, for each dataset $\mathcal{M}_j$ we have
    \begin{equation*}
        \begin{split}
            \sqrt{|\mathcal{M}_j|}(\hat\beta_i^{(j)} - \beta_i^*)  &= \sqrt{|\mathcal{M}_j|}H_i(\beta_i^*)^{-1}(\mathcal{G}_i(\hat\beta_i^{(j)}) - \mathcal{G}_i\left(\beta_i^*)\right)+ o_p(\sqrt{|\mathcal{M}_j|}\|\hat\beta_i^{(j)} - \beta_i^*\|) \\
            & = -H_i(\beta_i^*)^{-1}\left[\mathbb{G}_{N_j}(\nabla_{\beta_i} r_i(\theta,Z^i;\beta_i^*)) - \frac{\tilde N_i}{N+\tilde N_i} M_i \left(\mathbb{G}_{N_j}(s_i^*(\theta)) - \sqrt{\frac{|\mathcal{M}_j|}{\tilde N_i}}\mathbb{G}_{\tilde N}(s_i^*(\tilde\theta))\right)\right] \\
            & \qquad + o_p(1),
        \end{split}
    \end{equation*}
    where $\mathbb{G}_{N_j}$ is defined on the separated dataset $\mathcal{M}_j$. Therefore, we have the asymptotic normality for the final estimator that
    \begin{equation*}
        \begin{split}
            \sqrt{N}(\hat\beta_i - \beta_i^*) &= \sum_{j=[3]}\sqrt{\frac{|\mathcal{M}_j|}{N}}\sqrt{|\mathcal{M}_j|}(\hat\beta_i^{(j)} - \beta_i^*)  \\
            &= \sum_{j=[3]}\sqrt{\frac{|\mathcal{M}_j|}{N}}\sqrt{|\mathcal{M}_j|}H_i(\beta_i^*)^{-1}(\mathcal{G}_i(\hat\beta_i^{(j)}) - \mathcal{G}_i\left(\beta_i^*)\right)+ o_p(\sqrt{|\mathcal{M}_j|}\|\hat\beta_i^{(j)} - \beta_i^*\|) \\
            & = -H_i(\beta_i^*)^{-1}\left[\mathbb{G}_N(\nabla_{\beta_i} r_i(\theta,Z^i;\beta_i^*)) - \frac{\tilde N_i}{N+\tilde N_i} M_i \left(\mathbb{G}_N(s_i^*(\theta)) - \sqrt{\frac{N}{\tilde N_i}}\mathbb{G}_{\tilde N}(s_i^*(\tilde\theta))\right)\right] + o_p(1) \\
            & \xrightarrow{d} \mathcal{N}\left(0,H_i(\beta_i^*)^{-1} \left(\operatorname{Cov}\left(\nabla_{\beta_i} r_i(\theta,Z^i;\beta_i^*)\right) - \operatorname{Cov}\left(\E\big[\nabla_{\beta_i} r_i(\theta,Z^i;\beta_i^*)|\theta\big]\right)\right) H_i(\beta_i^*)^{-1}\right).
        \end{split}
    \end{equation*}
    
\end{proof}

\begin{lemma}
    \label{lem:inter CLT for beta}
    Suppose Assumption \ref{assumption for beta}, and $\E\|\hat s_i(\theta) - s_i^*(\theta)\|^2 \rightarrow 0$ hold. Denote $\hat{\beta}_{\hat{M}_i}$ as the result of the objective function (\ref{equ:objective function for beta}), then we have $\hat{\beta}_{\hat{M}_i} \xrightarrow{P} \beta_i^*$ and $\sqrt{N}(\hat{\beta}_{\hat{M}_i} - \beta_i^*) \xrightarrow{d} \mathcal{N}(0,\Sigma_{\beta_i})$, where $\Sigma_{\beta_i}$ is in Theorem \ref{thm:CLT of beta}.
\end{lemma}

\begin{proof}[Proof of Lemma \ref{lem:inter CLT for beta}]
\label{proof:inter CLT for beta}
    By the law of large number, the estimated matrix $\hat{M}_i$ for de-correlation is consistent to the population optimal matrix $M$:
    $$
    \hat M_i=\widehat{\Cov}\big(\nabla_{\beta_i} r_i(\theta,Z^i;\tilde\beta_i),\hat s_i(\theta)\big)\widehat{\Cov}\big(\hat s_i(\theta)\big)^{-1} \xrightarrow{P} \Cov\big(\nabla_{\beta_i} r_i(\theta,Z^i;\beta_i^*), s_i^*(\theta)\big)\Cov\big(s_i^*(\theta)\big)^{-1} = M_i.
    $$
    
    First, we prove the consistency of $\hat{\beta}_{\hat{M}_i}$ for each $i$. Since the objective function $\mathcal{L}_i(\beta_i)$ is unbiased for $R_i(\beta_i)$, by Kolmogorov's strong law of large numbers and local Lipschitzness for $\beta_i^*$, we know there exists $\epsilon_i>0$ such that $\{\beta_i:\|\beta_i - \beta_i^*\|\leq \epsilon_i\} \subseteq \mathcal{B}_i$ and
    $$
    \sup_{\beta_i:\|\beta_i - \beta_i^*\|\leq \epsilon_i}|\mathcal{L}_i(\beta_i) - R_i(\beta_i)| \xrightarrow{P}0.
    $$
    Since the loss function $r_i(\theta,Z^i;\beta_i)$ is $\gamma_i$-strongly convex in $\beta_i$, the minimizer $\beta_i^*$ is unique. Therefore, for every $\epsilon_i >0$, there exists a $\eta_i = \frac{\gamma_i}{2}\epsilon_i^2>0$ such that for every $\beta_i \in \mathcal{B}_i$ with $\|\beta_i - \beta_i^*\| \geq \epsilon_i$, we have $R_i(\beta_i) \geq  R_i(\beta_i^*) + \frac{\gamma_i}{2}\|\beta_i - \beta_i^*\|^2 \geq R_i(\beta_i^*) + \eta_i$. For the $\epsilon_i$-shell $\{\beta_i:\|\beta_i - \beta_i^*\| = \epsilon_i\}$, we have:
    \begin{equation*}
        \begin{split}
            & \qquad \inf_{\|\beta_i - \beta_i^*\| = \epsilon_i} \mathcal{L}_i(\beta_i) - \mathcal{L}_i(\beta_i^*) \\
            & = \inf_{\|\beta_i - \beta_i^*\| = \epsilon_i} \left((\mathcal{L}_i(\beta_i) - R_i(\beta_i)) + (R_i(\beta_i) - R_i(\beta_i^*)) + (R_i(\beta_i^*) - \mathcal{L}_i(\beta_i^*))\right) \\
            & \geq \eta_i - 2\sup_{\|\beta_i - \beta_i^*\| \leq \epsilon} |\mathcal{L}_i(\beta_i) - R_i(\beta_i)| \\
            & =\eta_i - o_p(1).
        \end{split}
    \end{equation*}
    Then we consider for any $\beta_i$ such that $\|\beta_i - \beta_i^*\| \geq \epsilon_i$, fix a point $\beta^1_i = \beta_i^* + \frac{\beta_i - \beta_i^*}{\|\beta_i - \beta_i^*\|}\epsilon_i$ which is on the $\epsilon_i$-shell, we have $\beta_i = \beta_i^* + \lambda_i(\beta^1_i - \beta_i^*)$ where $\lambda_i = \frac{\|\beta_i - \beta_i^*\|}{\epsilon_i} \geq 1$. Thus, by the convexity of $\mathcal{L}_i(\beta_i)$ as it consists of convex functions, the following inequality holds:
    $$
    \mathcal{L}_i(\beta_i) - \mathcal{L}_i(\beta_i^*) \geq \lambda_i \left(\mathcal{L}_i(\beta^1_i) - \mathcal{L}_i(\beta_i^*)\right) \geq \eta_i - o_p(1).
    $$
    This inequality implies that $\{\beta_i:\|\beta_i - \beta_i^*\| \geq \epsilon_i\}$ has no minimizer to $\mathcal{L}_i(\beta_i)$, so the consistency holds:
    $$
    \mathbb{P}(\|\hat{\beta}_{\hat{M}_i} - \beta_i^*\| \geq \epsilon_i)= 0.
    $$

    Now we prove the asymptotic normality. 
    As $\E\|\hat s_i(\theta) - s_i^*(\theta)\|^2 \rightarrow 0$ holds, we have the following convergence by the \citep[Lemma 19.24]{van2000asymptotic}:
    $$
    \mathbb{G}_N\left[\hat s_i(\theta) - s_i^*(\theta)\right] \xrightarrow{P}0 \quad\text{and} \quad \mathbb{G}_{\tilde N}\left[\hat s_i(\theta) - s_i^*(\theta)\right] \xrightarrow{P}0 .
    $$
    By the local Lipschitzness and differentiability, and \citep[Lemma 19.24]{van2000asymptotic}, we also have:
\begin{equation*}
     \mathbb{G}_N\left[\nabla_{\beta_i} r_i(\theta,Z^i;\hat{\beta}_{\hat{M}_i}) - \nabla_{\beta_i} r_i(\theta,Z^i;\beta_i^*)\right] \xrightarrow{P}0.
\end{equation*}
    Denote $\mathcal{G}_i(\beta_i) = \E [\nabla_i r_i(\theta,Z^i;\beta_i)]$. Since the loss function $r_i(\theta,Z^i;\beta_i)$ is strongly convex, $\hat{\beta}_{\hat{M}_i}$ is the unique solution to the equation that
    $$
    \mathcal{F}_i(\beta_i) = \frac{1}{N}\sum_{k\in[N_i]}\bigg\{\nabla_i r_i(\theta_k,Z_k^i;\beta_i)-\frac{\tilde N_i}{N+\tilde N_i}\hat M_i\hat s_i(\theta_k)\bigg\}+\frac{1}{N+\tilde N_i}\sum_{k\in[\tilde N_i]}\hat M_i\hat s_i(\tilde\theta_k) = 0
    $$
    and also $\mathcal{G}_i(\beta_i^*) = \mathcal{F}_i(\hat{\beta}_{\hat{M}_i}) = 0$. Thus, we have its Taylor expansion as
    $$
    \sqrt{N}(\mathcal{G}_i(\hat{\beta}_{\hat{M}_i}) - \mathcal{G}_i\left(\beta_i^*)\right) = \sqrt{N}H_i(\beta_i^*)(\hat{\beta}_{\hat{M}_i} - \beta_i^*) + o_p(\sqrt{N}\|\hat{\beta}_{\hat{M}_i} - \beta_i^*\|).
    $$
    Note that $\mathcal{G}_i(\hat{\beta}_{\hat{M}_i})$ can be rewritten as
    \begin{equation*}
        \begin{split}
             \mathcal{G}_i(\hat{\beta}_{\hat{M}_i}) &= \E [\nabla_i r_i(\theta,Z^i;\hat{\beta}_{\hat{M}_i})] \\
             &= \E [\nabla_i r_i(\theta,Z^i;\hat{\beta}_{\hat{M}_i})] - \frac{\tilde N_i}{N+\tilde N_i}\hat M_i \E[\hat{s}_i] +\frac{\tilde N_i}{N+\tilde N_i}\hat M_i \E[\hat{s}_i] \\ 
             &= \E \left[\nabla_i r_i(\theta,Z^i;\hat{\beta}_{\hat{M}_i})\right] - \frac{\tilde N_i}{N+\tilde N_i}\hat M_i \E\left[\frac{1}{N}\sum_{k\in[N_i]}\hat s_i(\theta_k)\right] +\frac{1}{N+\tilde N_i}\hat M_i \E\left[\sum_{k\in[\tilde N_i]}\hat s_i(\theta_k)\right].
        \end{split}
    \end{equation*}
    Therefore we have
    \begin{equation*}
        \begin{split}
            \sqrt{N}(\mathcal{G}_i(\hat{\beta}_{\hat{M}_i}) - \mathcal{G}_i(\beta_i^*)) & = \sqrt{N}(\mathcal{G}_i(\hat{\beta}_{\hat{M}_i}) - \mathcal{F}_i(\hat{\beta}_{\hat{M}_i})) \\
            & = - \left[\mathbb{G}_N(\nabla_i r_i(\theta,Z^i;\hat{\beta}_{\hat{M}_i})) - \frac{\tilde N_i}{N+\tilde N_i}\hat M_i \left(\mathbb{G}_N(\hat s_i(\theta)) - \sqrt{\frac{N}{\tilde N_i}}\mathbb{G}_{\tilde N}(\hat s_i(\tilde\theta))\right)\right] \\
            & = - \left[\mathbb{G}_N(\nabla_i r_i(\theta,Z^i;\beta_i^*)) - \frac{\tilde N_i}{N+\tilde N_i} M_i \left(\mathbb{G}_N(s_i^*(\theta)) - \sqrt{\frac{N}{\tilde N_i}}\mathbb{G}_{\tilde N}(s_i^*(\tilde\theta))\right)\right] + o_p(1),
        \end{split}
    \end{equation*}
    where the third equation is ensured by the convergence results above.
    We apply the central limit theorem to the RHS and obtain its asymptotic normality as follows:
    \begin{equation*}
        \begin{split}
            &\qquad \mathbb{G}_N(\nabla_i r_i(\theta,Z^i;\beta_i^*)) - \frac{\tilde N_i}{N+\tilde N_i} M_i \left(\mathbb{G}_N(s_i^*(\theta)) - \sqrt{\frac{N}{\tilde N_i}}\mathbb{G}_{\tilde N}(s_i^*(\tilde\theta))\right) \\
            & = \mathbb{G}_N\left(\nabla_i r_i(\theta,Z^i;\beta_i^*)-\frac{\tilde N_i}{N+\tilde N_i} M_i s_i^*(\theta)\right) + \sqrt{\frac{N}{\tilde N_i}} \mathbb{G}_{\tilde N}\left(\frac{\tilde N_i}{N+\tilde N_i}M_i s_i^*(\tilde\theta)\right) \xrightarrow{d} \mathcal{N}(0,\Sigma_i),
        \end{split}
    \end{equation*}
    where 
    $$
    \Sigma_i = \Cov\left(\nabla_i r_i(\theta,Z^i;\beta_i^*)-\frac{\tilde N_i}{N+\tilde N_i} M_i s_i^*(\theta)\right) + \frac{N}{\tilde N_i}\Cov\left(\frac{\tilde N_i}{N+\tilde N_i} M_i s_i^*(\tilde\theta)\right).
    $$
    Since $M_i=\Cov\big(\nabla_{\beta_i} r_i(\theta,Z^i;\beta_i^*), s_i^*(\theta)\big)\Cov\big(s_i^*(\theta)\big)^{-1}$, we can further simplify the covariance as follows and find that $\Sigma_i = \Sigma_{\beta_i}$:
\begin{equation*}
    \begin{split}
        &\qquad \Cov\left(\nabla_i r_i(\theta,Z^i;\beta_i^*)-\frac{\tilde N_i}{N+\tilde N_i} M_i s_i ^*(\theta)\right) + \frac{N}{\tilde N_i}\Cov\left(\frac{\tilde N_i}{N+\tilde N_i} Ms^*(\tilde\theta)\right) \\
        & = \Cov(\nabla_i r_i(\theta,Z^i;\beta_i^*)) + \left(1+\frac{N}{\tilde N_i}\right)\Cov\left(\frac{\tilde N_i}{N+\tilde N_i} M_i s_i^*(\tilde\theta)\right) - 2\Cov\left(\nabla_i r_i(\theta,Z^i;\beta_i^*),\frac{\tilde N_i}{N+\tilde N_i} M_i s_i^*(\theta)\right)\\
        & = \Cov(\nabla_i r_i(\theta,Z^i;\beta_i^*)) - \frac{\tilde N_i}{N+\tilde N_i}\Cov (M_i s_i^*(\theta)) \\
        & = \Cov(\nabla_i r_i(\theta,Z^i;\beta_i^*)) - \frac{\tilde N_i}{N+\tilde N_i} \Cov\big(\nabla_{\beta_i} r_i(\theta,Z^i;\beta_i^*), s_i^*(\theta)\big)\Cov\big(s_i^*(\theta)\big)^{-1}\Cov\big(\nabla_{\beta_i} r_i(\theta,Z^i;\beta_i^*), s_i^*(\theta)\big) \\
        & = \Cov(\nabla_i r_i(\theta,Z^i;\beta_i^*)) - \frac{\tilde N_i}{N+\tilde N_i} \Cov\big(s_i^*(\theta)\big),
    \end{split}
\end{equation*}
where the second equation holds as $\Cov\left(\nabla_i r_i(\theta,Z^i;\beta_i^*),s_i^*(\theta)\right) = \Cov(s_i^*(\theta))$ with $s_i^*(\theta) = \E[\nabla_{\beta_i} r_i(\theta,Z^i;\beta_i)\mid\theta]$. Based on the analysis above, we have the asymptotic normality results:
$$
\sqrt{N}(\hat{\beta}_{\hat{M}_i} - \beta_i^*)  = \sqrt{N}H_i(\beta_i^*)^{-1}(\mathcal{G}_i(\hat{\beta}_{\hat{M}_i}) - \mathcal{G}_i\left(\beta_i^*)\right)+ o_p(\sqrt{N}\|\hat{\beta}_{\hat{M}_i} - \beta_i^*\|) \xrightarrow{d} N(0,\Sigma_{\beta_i}),
$$
since $\Cov(\nabla_i r_i(\theta,Z^i;\beta_i^*)) - \frac{\tilde N_i}{N+\tilde N_i} \Cov\big(s_i^*(\theta)\big) \rightarrow \Cov(\nabla_i r_i(\theta,Z^i;\beta_i^*)) - \Cov\big(s_i^*(\theta)\big)$ as $\frac{N}{\tilde N_i} \rightarrow 0$.
\end{proof}

\subsubsection{Proof of Theorem \ref{thm:CLT of theta, optimal}}

\begin{proof}[Proof of Theorem \ref{thm:CLT of theta, optimal}]

Note that $\hat{\theta}^{\hat{\beta}}_{PO} - \theta^{\beta^*}_{PO} = (\hat{\theta}^{\hat{\beta}}_{PO} - \theta^{\hat{\beta}}_{PO}) + (\theta^{\hat{\beta}}_{PO} - \theta^{\beta^*}_{PO})$ holds, our proof of both consistency and asymptotic normality will be separated into two parts similarly.

First we prove the property of consistency of $\hat{\theta}^{\hat{\beta}}_{PO}$. 
As for the $\hat{\theta}^{\hat{\beta}}_{PO}$ to $\theta^{\hat{\beta}}_{PO}$ part, the proof follows from the same argument as in the Proof for Lemma \ref{lem:consistency of z-estimation}. As for the $\theta^{\hat{\beta}}_{PO}$ to $\theta^{\beta^*}_{PO}$ part, since the map $\mathrm{sol}(\beta)$ is differentiable at $\beta^*$, it is also a continuous function at $\beta^*$.
Thus, by the Continuous Mapping Theorem, we have:
$$
\theta^{\hat{\beta}}_{PO} = \mathrm{sol}(\hat \beta) \xrightarrow{P} \theta^{\beta^*}_{PO} = \mathrm{sol}(\beta^*).
$$
Combining the results above, we can prove the consistency of $\hat{\theta}^{\hat{\beta}}_{PO}$ toward $\theta^{\beta^*}_{PO}$ that
$$
\mathbb{P}(\|\hat{\theta}^{\hat{\beta}}_{PO} - \theta^{\beta^*}_{PO}\| \geq \epsilon) \leq \mathbb{P}(\|\hat{\theta}^{\hat{\beta}}_{PO} - \theta^{\hat{\beta}}_{PO}\| \geq \epsilon/2) + \mathbb{P}(\|\theta^{\hat{\beta}}_{PO} - \theta^{\beta^*}_{PO}\| \geq \epsilon/2) = 0.
$$

Now we prove the asymptotic normality of our estimator.
The proof of the asymptotic normality of the $\hat \theta_{PO}^{\hat \beta}$ towards $\theta_{PO}^{\hat \beta}$ is similar to the proof of Theorem \ref{thm:CLT of theta, stable}, and we can obtain the result by central limit theorem and the Slutsky lemma that
\begin{equation*}
    \sqrt{n}(\hat \theta_{PO}^{\hat \beta} - \theta_{PO}^{\beta^*}) \mid \hat \beta \xrightarrow{d} N(\sqrt{n}(\theta_{PO}^{\hat \beta} - \theta_{PO}^{\beta^*}), \hat \Sigma_{\hat \theta}) \triangleq N(\hat \mu_{\hat \theta}, \hat \Sigma_{\hat \theta}),
\end{equation*}
where 
$$
\hat \Sigma_{\hat \theta} = V_{\hat \beta}(\theta_{PO}^{\hat \beta})^{-1}\E_{Z \sim q(z)}\left(G(\theta_{PO}^{\hat \beta},Z, \hat \beta)G(\theta_{PO}^{\hat \beta},Z, \hat \beta)^T\right)V_{\hat \beta}(\theta_{PO}^{\hat \beta})^{-1}.
$$
We prove the distribution of $\sqrt{n}(\hat{\theta}^{\hat{\beta}}_{PO} - \theta^{\beta^*}_{PO} )$ by characteristic function. 
Similarly, we denote $X = \sqrt{n}(\hat{\theta}^{\hat{\beta}}_{PO} - \theta^{\beta^*}_{PO})$ and $Z = \sqrt{N}(\hat{\beta} - \beta^* )$, with variance as
$$
\Sigma_{\theta} = V_{\beta}(\theta_{PO}^{\beta})^{-1}\E_{Z \sim q(z)}\left(G(\theta_{PO}^{\beta},Z, \beta)G(\theta_{PO}^{\beta},Z, \beta)^T\right)V_{\beta}(\theta_{PO}^{\beta})^{-1}.
$$
The characteristic function of the condition distribution is:
$\phi_{X \mid Z}(t) \xrightarrow{P} \exp\left\{it^T\hat \mu_{\hat \theta} - \frac{1}{2}t^T\hat \Sigma_{\hat \theta} t\right\}$, and the distribution of $Z$ can be described by characteristic function as:
\begin{equation*}
    \begin{split}
        \mathbb{P}(Z) &= \frac{1}{(2\pi)^d} \int \phi_{Z}(s)\cdot e^{-is^TZ} ds \\
        &= \frac{1}{(2\pi)^d} \int \exp\left\{-\frac{1}{2}s^T\Sigma_{\beta}s-is^TZ\right\} ds,
    \end{split}
\end{equation*}
where the covariance $\Sigma_\beta$ is given in the Lemma \ref{lem:covariance of hat beta}.
Then we have the characteristic function that
\begin{equation*}
    \begin{split}
        \phi_{X}(t) &= \mathbb{E}(e^{it^TX}) = \mathbb{E}_{Z}\left(\mathbb{E}(e^{it^TX} \mid Z)\right) \\
        &= \frac{1}{(2\pi)^d}  \int \int \exp\left\{it^T\hat \mu_{\hat \theta} - \frac{1}{2}t^T\hat \Sigma_{\hat \theta} t\right\} \exp\left\{-\frac{1}{2}s^T\Sigma_{\beta}s-is^TZ\right\} ds dZ. \\
    \end{split}
\end{equation*}
To simplify the formulation, we let $-\frac{1}{2}s^T\Sigma_{\beta}s-is^TZ =  -\frac{1}{2}(s-A_1)^TM_1(s-A_1) + B_1$, and by comparing the terms, we have:
\begin{equation*}
    \begin{split}
        &M_1 = \Sigma_{\beta}, \\
        &A_1 = iM_1^{-1}Z =  i\Sigma_{\beta}^{-1}Z,\\
        &B_1 = \frac{1}{2}A_1^TM_1A_1 = -\frac{1}{2}Z^T\Sigma_{\beta}^{-1}Z.
    \end{split}
\end{equation*}
Then the characteristic function can be rewritten as
\begin{equation*}
    \begin{split}
        \phi_{X}(t) &= \frac{1}{(2\pi)^d}  \int \int \exp\left\{it^T\hat \mu_{\hat \theta} - \frac{1}{2}t^T\hat \Sigma_{\hat \theta} t\right\} \exp\left\{-\frac{1}{2}(s-A_1)^TM_1(s-A_1) + B_1\right\} ds dZ \\
        &= \frac{\det|\Sigma_{\beta}|}{(2\pi)^{d/2}}  \int \exp\left\{it^T\hat \mu_{\hat \theta} - \frac{1}{2}t^T\hat \Sigma_{\hat \theta} t  -\frac{1}{2}Z^T\Sigma_{\beta}^{-1}Z\right\}dZ.
    \end{split}
\end{equation*}
Since we have
\begin{equation*}
        \left|\exp\left\{it^T\hat \mu_{\hat \theta} - \frac{1}{2}t^T\hat \Sigma_{\hat \theta} t  -\frac{1}{2}Z^T\Sigma_{\beta}^{-1}Z\right\}\right| = \exp\left\{ - \frac{1}{2}t^T\hat \Sigma_{\hat \theta} t  -\frac{1}{2}Z^T\Sigma_{\beta}^{-1}Z\right\}
\end{equation*}
and $t^T\hat \Sigma_{\hat \theta} t  >0$ for all $t \in \mathcal{T}$, and therefore $-\frac{1}{2}t^T\hat \Sigma_{\hat \theta} t  <0$ for all $t \in \mathcal{T}$, the exponential term is bounded:
\begin{equation*}
    \left|\exp\left\{it^T\hat \mu_{\hat \theta} - \frac{1}{2}t^T\hat \Sigma_{\hat \theta} t -\frac{1}{2}Z^T\Sigma_{\beta}^{-1}Z\right\}\right| \leq \left|\exp\left\{-\frac{1}{2}Z^T\Sigma_{\beta}^{-1}Z\right\}\right|.
\end{equation*}
Denote $J_{sol}(\beta)$ as the Jacobian matrix of the map $\mathrm{sol}(\beta)$ for finding the optimality, so by the first-term Taylor expansion, we have:
\begin{equation*}
    \sqrt{n}(\theta^{\hat{\beta}}_{PO} - \theta^{\beta^*}_{PO}) = \sqrt{n}(\mathrm{sol}(\hat \beta) - \mathrm{sol}(\beta^*))  \rightarrow \sqrt{\frac{n}{N}}J_{sol}(\beta^*) Z,
\end{equation*}
as $N \rightarrow \infty$.
Thus, by the control convergence theorem, we have:
\begin{equation*}
    \begin{split}
        \lim_{n \rightarrow \infty}\phi_{X}(t) &= \frac{\det|\Sigma_{\beta}|}{(2\pi)^{d/2}}  \int \lim_{n \rightarrow \infty}\exp\left\{it^T\hat \mu_{\hat \theta} - \frac{1}{2}t^T\hat \Sigma_{\hat \theta} t  -\frac{1}{2}Z^T\Sigma_{\beta}^{-1}Z\right\}dZ  \\
        &= \frac{\det|\Sigma_{\beta}|}{(2\pi)^{d/2}}  \int \exp\left\{it^T \sqrt{\frac{n}{N}} J_{sol}(\beta^*)Z- \frac{1}{2}t^T\Sigma_{\theta} t  -\frac{1}{2}Z^T\Sigma_{\beta}^{-1}Z\right\}dZ.
    \end{split}
\end{equation*}
Similarly, we let $it^T \sqrt{\frac{n}{N}}J_{sol}(\beta^*)Z -\frac{1}{2}Z^T\Sigma_{\beta}^{-1}Z =  -\frac{1}{2}(Z-A_2)^TM_2(Z-A_2) + B_2$, and by comparing the terms, we have:
\begin{equation*}
    \begin{split}
        &M_2 = \Sigma_{\beta}^{-1}, \\
        &A_2 = \sqrt{\frac{n}{N}}i M_2^{-1}J_{sol}(\beta^*)^T t = \sqrt{\frac{n}{N}} i\Sigma_{\beta} J_{sol}(\beta^*)^T t,\\
        &B_2 = \frac{1}{2} A_2^TM_2A_2 = -\frac{1}{2}\cdot \frac{n}{N}t^TJ_{sol}(\beta^*)\Sigma_{\beta}J_{sol}(\beta^*)^Tt.
    \end{split}
\end{equation*}
Then the limit of the characteristic function is
\begin{equation*}
    \begin{split}
        \lim_{n \rightarrow \infty}\phi_{X}(t)  &= \frac{\det|\Sigma_{\beta}|}{(2\pi)^{d/2}}  \int \exp\left\{it^T \sqrt{\frac{n}{N}}J_{sol}(\beta^*)Z- \frac{1}{2}t^T\Sigma_\theta t  -\frac{1}{2}Z^T\Sigma_{\beta}^{-1}Z\right\}dZ  \\
        &= \frac{\det|\Sigma_{\beta}|}{(2\pi)^{d/2}}  \int \exp\left\{- \frac{1}{2}t^T\Sigma_\theta t -\frac{1}{2}(Z-A_2)^TM_2(Z-A_2) + B_2\right\}dZ  \\
        &= \exp\left\{- \frac{1}{2}t^T(\Sigma_\theta +\frac{n}{N}J_{sol}(\beta^*)\Sigma_{\beta}J_{sol}(\beta^*)^T)t \right\},  \\
    \end{split}
\end{equation*}
which means the distribution function of $X = \sqrt{n}(\hat{\theta}^{\hat{\beta}}_{PO} - \theta^{\beta^*}_{PO})$ is a normal distribution with zero mean and variance $\Sigma_\theta +\frac{n}{N}J_{sol}(\beta^*)\Sigma_{\beta}J_{sol}(\beta^*)^T$, where $n$ is the sample size for estimating $\theta_{PO}^{\beta^*}$ by importance sampling, and $N$ is the sample size for estimating the distributional parameter $\beta^*$ by recalibrated method.
Therefore, by Slutsky’s Lemma, we have the asymptotic normality:
$$
\sqrt{N}(\hat{\theta}^{\hat{\beta}}_{PO} - \theta^{\beta^*}_{PO}) \xrightarrow{d} N(0,\Sigma),
$$
where 
$$
\Sigma = J_{sol}(\beta^*)\Sigma_{\beta}J_{sol}(\beta^*)^T,
$$
since $\frac{N}{n}\Sigma_\theta + J_{sol}(\beta^*)\Sigma_{\beta}J_{sol}(\beta^*)^T \xrightarrow{} J_{sol}(\beta^*)\Sigma_{\beta}J_{sol}(\beta^*)^T$ as $\frac{N}{n} \rightarrow 0$.
\end{proof}

\begin{lemma}
\label{lem:covariance of hat beta}
    Assume that Assumption \ref{assumption for beta} hold. Suppose sample sizes satisfy that $\frac{N}{\tilde N_i} \rightarrow 0$, and $\E\|\hat s_i(\theta) - s_i(\theta)\|^2 \xrightarrow{P} 0$ for some $s_i(\theta)$ for each player $i$, based on the analysis of Theorem \ref{thm:CLT of beta}, we have the asymptotic normality for $\hat \beta$:
    $$
    \sqrt{N}(\hat \beta - \beta^*) \xrightarrow{P} N(0,\Sigma_\beta).
    $$
    Let $s_i(\theta) = s_i^*(\theta)$, then we have the asymptotic covariance as 
    $$
    \Sigma_\beta = \operatorname{diag}\{H_i(\beta_i^*)^{-1}\left(\operatorname{Cov}\left( \nabla_{\beta_i} r_i(\theta_k,Z_k^i;\beta_i^*) \right) - \operatorname{Cov}\left(s_i^*(\theta_k)\right) \right)H_i(\beta_i^*)^{-1}\}.
    $$
\end{lemma}

\begin{proof}
From the Proof of Theorem \ref{thm:CLT of beta}, we know that
\begin{equation*}
    \begin{split}
        \sqrt{N}(\hat\beta_i - \beta_i^*) = &-H_i(\beta_i^*)^{-1}\left[\mathbb{G}_N(\nabla_{\beta_i} r_i(\theta,Z^i;\beta_i^*)) - \frac{\tilde N_i}{N+\tilde N_i} M_i \left(\mathbb{G}_N(s_i^*(\theta)) - \sqrt{\frac{N}{\tilde N_i}}\mathbb{G}_{\tilde N}(s_i^*(\tilde\theta))\right)\right] \\
        &+ O_P\left(\frac{1}{\sqrt{N}}\right).  
    \end{split}
\end{equation*}
For our first term, we have $\mathbb{G}_N$ as:
\begin{align*}
    \mathbb{G}_N(\nabla_{\beta_i} r_i(\theta,Z^i;\beta_i^*)) &= \sqrt{N} \left(\frac{1}{N}\sum_{k=1}^N\nabla_{\beta_i} r_i(\theta_k,Z_k^i;\beta_i^*) - \E(\nabla_{\beta_i} r_i(\theta,Z^i;\beta_i^*))\right) \\
    & = \frac{1}{\sqrt{N}}\sum_{k=1}^N \left(\nabla_{\beta_i} r_i(\theta_k,Z_k^i;\beta_i^*) - \E(\nabla_{\beta_i} r_i(\theta,Z^i;\beta_i^*))\right).
\end{align*}
Therefore, for a non-zero vector $\mathbf{c} = (c_1, ... c_m)^T \in \R^d$, we have 
\begin{align*}
\label{equ:term1}
    L_1 &= \sum_{i=1}^m c_i \left(H_i(\beta_i^*)^{-1} \mathbb{G}_N(\nabla_{\beta_i} r_i(\theta,Z^i;\beta_i^*))\right) \\
    & = \frac{1}{\sqrt{N}}\sum_{i=1}^m c_i \left(H_i(\beta_i^*)^{-1} \sum_{k=1}^N(\nabla_{\beta_i} r_i(\theta_k,Z_k^i;\beta_i^*) - \E(\nabla_{\beta_i} r_i(\theta,Z^i;\beta_i^*))) \right) \\
    & = \frac{1}{\sqrt{N}} \sum_{k=1}^N \left(\sum_{i=1}^m c_i H_i(\beta_i^*)^{-1} (\nabla_{\beta_i} r_i(\theta_k,Z_k^i;\beta_i^*) - \E(\nabla_{\beta_i} r_i(\theta,Z^i;\beta_i^*)) \right) \\
    & \triangleq \frac{1}{\sqrt{N}} \sum_{k=1}^N \left(\sum_{i=1}^m c_i \phi_i^1(Z_k^i)\right) \\
    & \triangleq \frac{1}{\sqrt{N}} \sum_{k=1}^N \psi_1(Z_k).
\end{align*}
Note that here we have zeros mean $E[\psi_1(Z)] = 0$, and the variance as follows:
$$
\operatorname{Cov}(\psi(Z)) = \operatorname{Var}\left( \sum_{i=1}^m c_i \phi_i(Z^i) \right) = \sum_{i=1}^m \sum_{j=1}^m c_i c_j \cdot \operatorname{Cov}(\phi_i^1(Z^i), \phi_j^1(Z^j)) = \mathbf{c}^T \Sigma_1 \mathbf{c}.
$$
By the central limit theorem and the Slutsky lemma, we have the asymptotic normality:
$$
L_1 \xrightarrow{d} N(0, \mathbf{c}^T \Sigma_1 \mathbf{c}),
$$
where the covariance $\Sigma_1 = \operatorname{diag}\{\operatorname{Cov}(\phi_i^1(Z^i)), i\in [m]\}$ as $Z^i$ are independent.
For the second and the third terms, we do the similar process and obtain the similar items 
\begin{align*}
    L_2 & = \sum_{i=1}^m c_i \left(\frac{\tilde N_i}{N + \tilde N_i}M_i \cdot H_i(\beta_i^*)^{-1} \mathbb{G}_N(s_i^*(\theta))\right) \\
    & = \frac{1}{\sqrt{N}} \sum_{k=1}^N \left(\sum_{i=1}^m c_i \cdot \frac{\tilde N_i}{N + \tilde N_i}M_i \cdot H_i(\beta_i^*)^{-1} \cdot (s_i^*(\theta_k) - \E(s_i^*(\theta_k))) \right) \\
    & \triangleq \frac{1}{\sqrt{N}} \sum_{k=1}^N \left(\sum_{i=1}^m c_i \phi_i^2(\theta_k)\right), \\
    L_3 & = \sum_{i=1}^m c_i \left(\sqrt{\frac{N}{\tilde N_i}}\frac{\tilde N_i}{N + \tilde N_i}M_i \cdot H_i(\beta_i^*)^{-1} \mathbb{G}_{\tilde N}(s_i^*(\tilde \theta))\right) \\
    & = \frac{1}{\sqrt{\tilde N}}\sum_{k=1}^{\tilde N} \left(\sum_{i=1}^m c_i \sqrt{\frac{N}{\tilde N_i}}\frac{\tilde N_i}{N + \tilde N_i}M_i \cdot H_i(\beta_i^*)^{-1} \cdot (s_i^*(\theta_k) - \E(s_i^*(\theta_k))) \right) \\
    & \triangleq \frac{1}{\sqrt{\tilde N}} \sum_{k=1}^{\tilde N} \left(\sum_{i=1}^m c_i \phi_i^3(\theta_k)\right),
\end{align*}
and asymptotic covariances $\Sigma_2 = \operatorname{diag}\{\operatorname{Cov}(\phi_i^2(Z^i)), i\in [m]\}$ and $\Sigma_3 = \operatorname{diag}\{\operatorname{Cov}(\phi_i^3(Z^i)), i\in [m]\}$.
Therefore, we have the asymptotic normality of the linear combination of $\sqrt{N}(\hat \beta_i - \beta_i^*)$ as follows:
\begin{align*}
    L &= \sum_{i=1}^m c_i \sqrt{N}(\hat \beta_i - \beta_i^*) \\
    & = -(L_1 - L_2 + L_3) \\
    & \xrightarrow{d}  N(0,\Sigma_\beta),
\end{align*}
where the covariance matrix is
$$
\Sigma_\beta = \operatorname{diag}\{H_i(\beta_i^*)^{-1}\left(\operatorname{Cov}\left( \nabla_{\beta_i} r_i(\theta_k,Z_k^i;\beta_i^*) \right) - \operatorname{Cov}\left(s_i^*(\theta_k)\right) \right)H_i(\beta_i^*)^{-1}\},
$$
since $\frac{N}{\tilde N_i} \rightarrow 0$ and $\operatorname{Cov}(M_is_i^*(\theta)) = \operatorname{Cov}(s_i^*(\theta))$ as we have proved in Proof \ref{proof:inter CLT for beta}.
By the Cremer-Wold Theorem (Lemma \ref{cremer-wold theorem}),we deduce the asymptotic normality of the estimator $\hat \beta$:
$$
\sqrt{N}(\hat \beta -  \beta^*) \xrightarrow{d} N(0, \Sigma_\beta).
$$

\end{proof}

\begin{lemma}[Cramer–Wold Theorem]
\label{cremer-wold theorem}
Let $\{X_n\}$ be a sequence of random vectors in $\mathbb{R}^d$, and let $X$ be a random vector in $\mathbb{R}^d$. Then:
$$
X_n \xrightarrow{d} X \quad \Longleftrightarrow \quad
a^\top X_n \xrightarrow{d} a^\top X \;\;\; \text{for all } a \in \mathbb{R}^d.
$$
\end{lemma}

\subsubsection{Proof of the consistency of estimated covariances}

\begin{proof}[Proof of Theorem \ref{thm:consistency of estimated cov, beta}]
Since the samples $(\theta_k,Z_{k}^i)$ are i.i.d. from the joint distribution $D_\theta \times D_i(\theta_k)$ with finite expectation conditions $\mathbb{E}\|\nabla^2_{\beta_i} r_i(\theta, Z^i; \beta_i^*)\|^2 \leq \infty$, $\mathbb{E}\|\nabla_{\beta_i} r_i(\theta, Z^i; \beta_i^*)\|^2 \leq \infty$, by the law of large number, we directly have the consistency:
\begin{align*}
    \hat H_i(\beta_i^*) = \frac{1}{N} \sum_{k = 1} ^N &\left[ \nabla_{\beta_i}^2 r_i(\theta_k,Z_{k}^i; \beta_i^*)\right] \xrightarrow{P} H_i(\beta_i^*), \\
    \hat V_a(\beta_i^*) = \frac{1}{N} \sum_{k =1}^N &\left(\nabla r_i(\theta_k,Z_{k}^i; \beta_i^*) - L_i^* \right)\left(\nabla r_i(\theta_k,Z_{k}^i; \beta_i^*) - L_i^* \right)^T  \xrightarrow{P} \operatorname{Cov}\left(\nabla_{\beta_i} r_i(\theta,Z^i;\beta_i^*)\right).
\end{align*}
Besides, since $\theta_k$ are i.i.d. from the distribution $D_\theta$, and $Z_{k,j}^i$ are also i.i.d. from $D_i(\theta_k)$ conditional on $\theta_k$, by the law of large number, we have the results:
\begin{align*}
    \hat V_b(\beta_i^*) = \frac{1}{N} \sum_{k =1}^N \left(\frac{1}{M}\sum_{j=1}^M\nabla r_i(\theta_k,Z_{k,j}^i; \beta_i^*) - W_i^* \right)&\left(\frac{1}{M}\sum_{j=1}^M\nabla r_i(\theta_k,Z_{k,j}^i; \beta_i^*) - W_i^* \right)^T \\
    & \xrightarrow{P} \operatorname{Cov}\left(\E\big[\nabla_{\beta_i} r_i(\theta,Z^i;\beta_i^*)|\theta\big]\right).
\end{align*}
Since $H_i(\beta_i^*)$ is nonsingular according to our previous analysis, and the three components follow their own consistency, the consistency of $\hat \Sigma_\beta$ follows directly from the continuous mapping theorem:
$$
\hat \Sigma_{\beta} = \hat H_i(\beta_i^*)^{-1} \hat V_i(\beta_i^*) \hat H_i(\beta_i^*)^{-1} \xrightarrow{P} \Sigma_{\beta}.
$$
\end{proof}

\begin{proof}[Proof of Theorem \ref{thm:consistency of estimated cov, optimal theta}]
The proof follows from the same argument as in the Proof of Theorem \ref{thm:consistency of estimated cov, beta}. Since $Z_k$ are i.i.d. from $q(z)$, by the law of large number, we obtain the consistency:
    \begin{align*}
        \hat J_{1}(\beta) &= \frac{1}{N} \sum_{k = 1} ^N \left[ \frac{\partial G(Z_k,\theta_{PO}^{\beta^*};\beta^*)}{\partial \theta^\top}\right] \xrightarrow{P} \mathbb{E}_{Z \sim q(z)} \left[\frac{\partial G(Z,\theta_{PO}^{\beta^*};\beta^*)}{\partial \theta^\top}\right],\\
        \hat J_{2}(\beta) &= \frac{1}{N} \sum_{k =1}^N \left[\frac{\partial G(Z_k,\theta_{PO}^{\beta^*};\beta^*)}{\partial \beta^\top}\right] \xrightarrow{P}  \mathbb{E}_{Z \sim q(z)} \left[\frac{\partial G(Z,\theta_{PO}^{\beta^*};\beta^*)}{\partial \beta^\top}\right] .
    \end{align*}
    Similarly, we obtain the the consistency result by the nonsingularity and the continuous mapping theorem:
    $$
    \hat \Sigma = \hat J_{sol}(\beta)^{-1}\hat \Sigma_\beta \hat J_{sol}(\beta)^{-1} \xrightarrow{P} \Sigma.
    $$
    
\end{proof}

\subsubsection{Proof of Lemma \ref{lem:EIF}}

\begin{proof}[Proof of Lemma \ref{lem:EIF}]
    The proof consists of two parts. Firstly, we show $\Psi_{\beta^*}$ is an influence function. Then, we prove $\Psi_{\beta^*}$ is in the tangent space of $\mathscr{P}_{\theta,Z}$. The results for $\Psi_{\theta_{PO}^{\beta^*}}$ then follows from the Delta method \citep{van2000asymptotic}.
    
    We start by proving $\Psi_{\beta^*}$ is an influence function. For any smooth one-dimensional parametric submodel $\{P_{\theta,Z}^u:u\in\R,|u|\le\delta\}\subset\mathscr{P}_{\theta,Z}$ with $P_{\theta,Z}^0=P_{\theta,Z}$ and score function $s$, since the marginal distribution $P_\theta^u$ of $\theta$ is fixed, we know
    \[s(\theta,Z)=\frac{d}{du}\log\frac{dP_{\theta,Z}^u}{dP_{\theta,Z}}\bigg|_{u=0}=\frac{d}{du}\log\frac{dP_{Z|\theta}^u}{dP_{Z|\theta}}\bigg|_{u=0}.\]
    Therefore, $s(\theta,Z)$ is the score of conditional sub-models and thus satisfies $\E_{P_{Z|\theta}}s(\theta,Z)=0$ $P_\theta$-almost surely. Denote
    \[\beta_i^{*(u)}=\argmin_{\beta_i}\E_{P^u_{\theta,Z}}r_i(\theta,Z^i;\beta_i),\quad i\in[m].\]
    Then it follows that
    \[\E_{P_{\theta,Z}^u}G_r(\theta,Z;\beta^{*u})=0,\quad\text{with}\quad G_r(\theta,Z;\beta)=(\nabla_{\beta_1}^\top r_1(\theta,Z^1;\beta_1),\ldots,\nabla_{\beta_m}^\top r_m(\theta,Z^m;\beta_m))^\top.\]
    Taking derivatives on both sides, we get
    \begin{align*}
        \frac{d\beta^{*(u)}}{du}\bigg|_{u=0}=&-\big\{\E_{P_{\theta,Z}}\nabla_\beta^\top G_r(\theta,Z;\beta^*)\big\}^{-1}\E_{P_{\theta,Z}}G_r(\theta,Z;\beta^*)s(\theta,Z).
    \end{align*}
    Note that
    \[\E_{P_{\theta,Z}}s(\theta,Z)\E_{P_{Z|\theta}}G_r(\theta,Z;\beta^*)=\E_{P_{Z}}\E_{P_{Z|\theta}}s(\theta,Z)\E_{P_{Z|\theta}}G_r(\theta,Z;\beta^*)=0,\]
    therefore $\Psi_{\beta^*}$ is an influence function, i.e.,
    \[\frac{d\beta^{*(u)}}{du}\bigg|_{u=0}=\E_{P_{\theta,Z}}\Psi_{\beta^*}(\theta,Z)s(\theta,Z).\]
    
    Then we show elements of $\Psi_{\beta^*}$ are in the tangent space of $\mathscr{P}_{\theta,Z}$. For $u=(u_1^\top,\ldots,u_m^\top)^\top$, we define $P_{\theta,Z}^u$ as
    \[\frac{dP_{\theta,Z}^u}{dP_{\theta,Z}}(\theta,Z)=\prod_{i\in[m]}\frac{K(u_i^\top s_i(\theta,Z^i))}{C_i^{u_i}(\theta)},\quad K(x)=\frac{2}{1+e^{-2x}},\quad C_i^{u_i}(\theta)=\E_{\D_i(\theta)}K(u_i^\top s_i(\theta,Z^i)),\]
    \[s_i(\theta,Z^i)=-\big\{\E_{P_{\theta,Z}}\nabla_{\beta_i}^2 r_i(\theta,Z^i;\beta^*_i)\big\}^{-1}\big\{\nabla_{\beta_i}r_i(\theta,Z^i;\beta_i^*)-\E_{\D_i(\theta)}\nabla_{\beta_i}r_i(\theta,Z^i;\beta_i^*)\big\},\]
    we know the score is $\Psi_{\beta^*}=(s_1^\top,\ldots,s_m^\top)^\top$ and 
    \[P_{\theta}^u(A)=\E_{P_{\theta,Z}^u}\bm{1}(\theta\in A)=\E_{P_{\theta,Z}}\prod_{i\in[m]}\frac{K(u_i^\top s_i(\theta,Z^i))}{C_i^{u_i}(\theta)}\bm{1}(\theta\in A)=\E_{P_{\theta,Z}}\bm{1}(\theta\in A)=P_{\theta}(A).\]
    Denote $\D_i^{u}(\theta)=P_{Z^i|\theta}^{u}$, we know
    \[\frac{d\D_i^{u}(\theta)}{d\D_i(\theta)}(Z^i)=\frac{K(u_i^\top s_i(\theta,Z^i))}{C_i^{u_i}(\theta)}.\]
    Then it suffices to show $\D^u_{[m]}$ satisfies Assumption \ref{assumption for beta} and \ref{assumption for theta, optimal}.

    \noindent \textbf{Assumption \ref{assumption for beta}.1: Locally Lipschitz}

    Note that $\sup_{x\in\R}|\nabla K(x)|=1$, $\sup_{x\in\R}|K(x)|=2$, and $\sup_{\theta\in\Theta,Z^i\in\Zcal_i}\|\nabla_{\beta_i}r_i(\theta,Z^i;\beta^*_i)\|_2<\infty$, then
    \[|C_i^{u_i}(\theta)-1|=\big|\E_{\D_i(\theta)}K(u_i^\top s_i(\theta,Z^i))-1\big|\le\E_{\D_i(\theta)}|u_i^\top s_i(\theta,Z)|=O(\|u\|_2),\]
    \[\E_{P_{\theta,Z}^u} \big(L_{U_i}^i(\theta,Z^i)\big)^2=\E_{P_{\theta,Z}^u} \frac{K(u_i^\top s_i(\theta,Z^i))}{C_i^{u_i}(\theta)}\big(L_{U_i}^i(\theta,Z^i)\big)^2\le(2+O(\|u\|_2)\E_{P_{\theta,Z}} \big(L_{U_i}^i(\theta,Z^i)\big)^2<\infty.\]

    \noindent \textbf{Assumption \ref{assumption for beta}.3: Positive definite}

    Since $\nabla_{\beta_i}r_i(\theta,Z^i;\beta_i^*)$ and $\nabla_{\beta_i}^2r_i(\theta,Z^i;\beta_i^*)$ are bounded on $\Theta\times\Zcal_i$ by Assumption \ref{asm:bounded_optimal}, then
    \[\bigg\|\E_{P_{\theta,Z}}\bigg\{\frac{K(u_i^\top s_i(\theta,Z^i))}{C_i^{u_i}(\theta)}-1\bigg\}\nabla_{\beta_i}^2r_i(\theta,Z^i;\beta_i^*)\bigg\|_{\rm sp}=O(\|u\|_2),\]
    \[\bigg\|\E_{P_{\theta,Z}}\bigg\{\frac{K(u_i^\top s_i(\theta,Z^i))}{C_i^{u_i}(\theta)}-1\bigg\}\nabla_{\beta_i}r_i(\theta,Z^i;\beta_i^*)\nabla_{\beta_i}^\top r_i(\theta,Z^i;\beta_i^*)\bigg\|_{\rm sp}=O(\|u\|_2),\]
    \[\bigg\|\E_{P_{\theta,Z}}\bigg\{\frac{K(u_i^\top s_i(\theta,Z^i))}{C_i^{u_i}(\theta)}-1\bigg\}\nabla_{\beta_i}r_i(\theta,Z^i;\beta_i^*)\bigg\|_2=O(\|u\|_2),\]
    \[\bigg\|\E_{\D_i(\theta)}\bigg\{\frac{K(u_i^\top s_i(\theta,Z^i))}{C_i^{u_i}(\theta)}-1\bigg\}\nabla_{\beta_i}r_i(\theta,Z^i;\beta_i^*)\bigg\|_2=O(\|u\|_2),\]
    then the assumption follows.
\end{proof}

\subsubsection{Proof of Theorem \ref{gaps, optimal}}

\begin{proof}[Proof of Theorem \ref{gaps, optimal}]
    Here we take the convexity of $\mathbf{PR}^{\beta_i^*}(\theta^i)$ as an example, and the proof based on the convexity of $\mathbf{PR}^i(\theta^i)$ is the same.
    
    We first prove that the distance between two risk functions is bounded for every $\theta^i \in \Theta_i$. Note that $\left|\ell_i(\theta^i, \theta^{-i}, Z^i)\right| \leq M_i$, so for every $\theta^i \in \Theta_i$, the distance between two risk functions is bounded as follows:
\begin{equation*}
        \begin{split}
            \left|\mathbf{PR}^{\beta_i^*}(\theta^i) - \mathbf{PR}^i(\theta^i)\right| &= \left| \int \ell_i(\theta^i, \theta^{-i}, Z^i) (p_{\beta_i^*}(z;\theta) - p_i(z;\theta)) dz\right| \\
            & \leq M_i \cdot \int \left|p_{\beta_i^*}(z;\theta) - p_i(z;\theta) \right|dz \\
            &= 2M_i \cdot \text{TV}(D_{\beta_i^*}(\theta), D_i(\theta)) \\
            &= 2M_i \cdot \sup_{\theta^i \in \Theta_i}\text{TV}(D_{\beta_i^*}(\theta), D_i(\theta)) \\
            &= 2M_i \cdot \eta_i.
        \end{split}
    \end{equation*}
Thus, we have
\begin{equation*}
\begin{split}
     -2M_i \cdot \eta_i &\leq \mathbf{PR}^{\beta_i^*}(\theta_{PO}^i) - \mathbf{PR}^i(\theta_{PO}^i) \leq 2M_i \cdot \eta_i, \\
     -2M_i \cdot \eta_i &\leq \mathbf{PR}^{\beta_i^*}(\theta^{\beta_i^*}_{PO}) - \mathbf{PR}^i(\theta^{\beta_i^*}_{PO}) \leq 2M_i \cdot \eta_i.
\end{split}
\end{equation*}
Since $\theta^{\beta_i^*}_{PO}$ is the minimizer for the risk function $\mathbf{PR}^{\beta_i^*}(\theta)$, the inequality of the strong convexity has the form
\begin{equation*}
    \mathbf{PR}^{\beta_i^*}(\theta^i) \geq \mathbf{PR}^{\beta_i^*}(\theta^{\beta_i^*}_{PO}) + \frac{\lambda_i}{2}\|\theta^i - \theta^{\beta_i^*}_{PO}\|^2.
\end{equation*}
Combining the strong convexity with inequalities above, we have
\begin{equation*}
    \begin{split}
        \mathbf{PR}^{\beta_i^*}(\theta^{\beta_i^*}_{PO}) &\leq \mathbf{PR}^{\beta_i^*}(\theta_{PO}^i) - \frac{\lambda_i}{2}\|\theta_{PO}^i - \theta^{\beta_i^*}_{PO}\|^2 \\
        & \leq \mathbf{PR}^i(\theta_{PO}^i) + 2M_i \cdot \eta_i - \frac{\lambda_i}{2}\|\theta_{PO}^i - \theta^{\beta_i^*}_{PO}\|^2, \\
    \end{split}
\end{equation*}
and
\begin{equation*}
    \begin{split}
        \mathbf{PR}^{\beta_i^*}(\theta^{\beta_i^*}_{PO}) &\geq \mathbf{PR}^i(\theta^{\beta_i^*}_{PO}) -2M_i \cdot \eta_i \\
        & \geq \mathbf{PR}^i(\theta_{PO}^i) -2M_i \cdot \eta_i,
    \end{split}
\end{equation*}
where $\theta_{PO}^i$ is the minimizer for the risk function $\mathbf{PR}^i(\theta^i)$. Thus, we have
\begin{equation*}
    \begin{split}
        \mathbf{PR}^i(\theta_{PO}^i) -2M_i \cdot \eta_i \leq \mathbf{PR}^i(\theta_{PO}^i) + 2M_i \cdot \eta_i - \frac{\lambda_i}{2}\|\theta_{PO}^i - \theta^{\beta_i^*}_{PO}\|^2. 
    \end{split}
\end{equation*}
This leads to the result for each player $i$ that
\begin{equation*}
    \|\theta_{PO}^i - \theta^{\beta_i^*}_{PO}\|_2^2  \leq \frac{8M_i \cdot \eta_i}{\lambda_i}.
\end{equation*}
Therefore, from the population level, we have
\begin{equation*}
\begin{split}
    \|\theta_{PO} - \theta^{\beta^*}_{PO}\|_2^2 = \sum_{i=1}^m \|\theta_{PO}^i - \theta^{\beta_i^*}_{PO}\|_2^2 \leq \sum_{i=1}^m \frac{8M_i \cdot \eta_i}{\lambda_i}.
\end{split}
\end{equation*}

\end{proof}

\section{Further experiments details and results} 

\subsection{Stable Points}

\subsubsection{Experiment details}

\paragraph{Verifying joint smoothness and strongly convex} We use the loss function $\ell(\theta,Z) = \frac{1}{2}\|Z-\theta\|_2^2$ here, where $Z,\theta \in \R^d$. The gradient of the loss function with respect to $\theta$ is $\nabla_\theta \ell(\theta,Z) = \theta - Z$, so we have
$$
\|\nabla\ell(\theta_1,Z) - \nabla\ell(\theta_2,Z)\|_2 = \|\theta_1 - \theta_2\|_2 \leq \beta \cdot \|\theta_1 - \theta_2\|_2,
$$
and the equality holds when $\beta=1$. Thus, the smoothness parameter is $\beta=1$.
Furthermore, the Hessian matrix of the loss function is $\nabla^2_\theta \ell(\theta, Z) = I_d$, of which the eigenvalues are all $1$, so the parameter for strong monotonicity is $\alpha = \lambda_{min} = 1$.

\paragraph{Verifying sensitivity} In the simulation, the distribution map is formed as
$$
\D(\theta) = N(\epsilon\theta, \Sigma), \quad\Sigma = diag(\sigma_1^2, ... ,\sigma_d^2),
$$ where $\epsilon, \sigma_1^2, ... ,\sigma_d^2 \in \R$. Now we verify that $\epsilon$ is the sensitive parameter for $\D(\theta)$. For any $\theta_1, \theta_2 \in \Theta$, set random variables as follows:
$$
X \sim \D(\theta_1) = N(\epsilon\theta_1, \Sigma),
$$
$$
Y = X + \epsilon(\theta_2 - \theta_1) \sim \D(\theta_2) = N(\epsilon\theta_2, \Sigma),
$$
which leads to the fact that $\E\|X - Y\|_1 = \|\epsilon\theta_1 - \epsilon\theta_2\|_1$. Since the Wasserstein-1 distance is defined as 
$$
W_1(\D(\theta_1),\D(\theta_2)) = \inf_{P \in \Gamma(\D(\theta_1),\D(\theta_1))}\E_{(X,Y) \sim P}\|X-Y\|_1,
$$ 
which is the infimum over all couplings, we have 
$$
W_1(\D(\theta_1),\D(\theta_2)) \leq \E\|X - Y\|_1 = \epsilon\|\theta_1 - \theta_2\|_1.
$$
Thus, the distribution map is $\epsilon$-sensitive.

\paragraph{Optimization details} As indicated in \cite{perdomo2020performative}, the definition of repeated risk minimization requires exact minimization of the objective at every iteration, and the authors used gradient descent with tolerance $10^{-8}$ and backtracking line search to decide step size at each iteration. Here, the definition of empirical repeated risk minimization also requires exact minimization for finding estimators at each iteration. However, we do not need to use optimization algorithms here as our problem can be simplified to the mean estimation as follows:
\begin{equation*}
        \theta_{t+1} = \arg\min _{\theta \in \Theta}\E_{Z \sim \mathcal{D}(\theta_{t})}\frac{1}{2} \|Z-\theta\|_2  = \E_{Z \sim \mathcal{D}(\theta_{t})} Z = \epsilon \cdot \theta_t.
\end{equation*}

\paragraph{Coverage rate} According to the update procedure, it is easy to see that the stable point for the performative problem is $\theta_{PS} = (0,0)^T$ in this problem. First, we compute the coverage rate of the confidence interval for each $\theta_t$, we do $1000$ independent experiments with $N=5000$ samples at each experiment. At each experiment, we construct the confidence interval for $\theta_t$ with estimators $\hat{\theta}_t$ and numerically estimated covariance $\hat{\Sigma}_t$ as
\begin{equation*}
    \left[\hat{\theta}_{t,(i)} - z_{1-\alpha/2} \cdot \sqrt{\frac{\hat{\Sigma}_{t,(ii)}}{N}}, \hat{\theta}_{t,(i)} + z_{1-\alpha/2} \cdot \sqrt{\frac{\hat{\Sigma}_{t,(ii)}}{N}} \right], \quad z_{1-\alpha/2} = \Phi^{-1}(0.975),
\end{equation*}
where $i = [d]$ denotes entries of $\theta$, $N$ is the number of samples and $\Phi$ is the quantile function of standard normal distribution. Besides, we can compute the coverage rate of the same confidence interval for stable point $\theta_{PS}$ with estimators $\hat{\theta}_t$ and estimated covariance $\hat{\Sigma}_t$. 

\subsubsection{Additional results}
\label{appdix:subsub:add}

According to the work \cite{li2025statisticalinferenceperformativity}, the magnitude of the distributional shift caused by the distribution map $D(\theta)$ can influence the iteration required for our estimations to reach a valid level, and here we further examine this by detecting the effect of sensitivity $\epsilon$ on the inferential performance of the true stable point $\theta_{PS}$ over time. In Figure \ref{fig:coverage-rate}, we compare the coverage rates of each coordinate of $\theta_{PS}$ under sensitivity levels $\epsilon = 0.01,\ 0.05$, and $0.2$ across time steps $t = 0, \ldots, 10$. Regardless of the value of $\epsilon$, the coverage rates of both coordinates converge to the target level $\alpha = 0.95$ as time progresses. However, as $\epsilon$ increases, more iterations are required for the coverage rate to reach the target level.

\begin{figure}[htbp]
    \centering
    \includegraphics[width=0.9\textwidth]{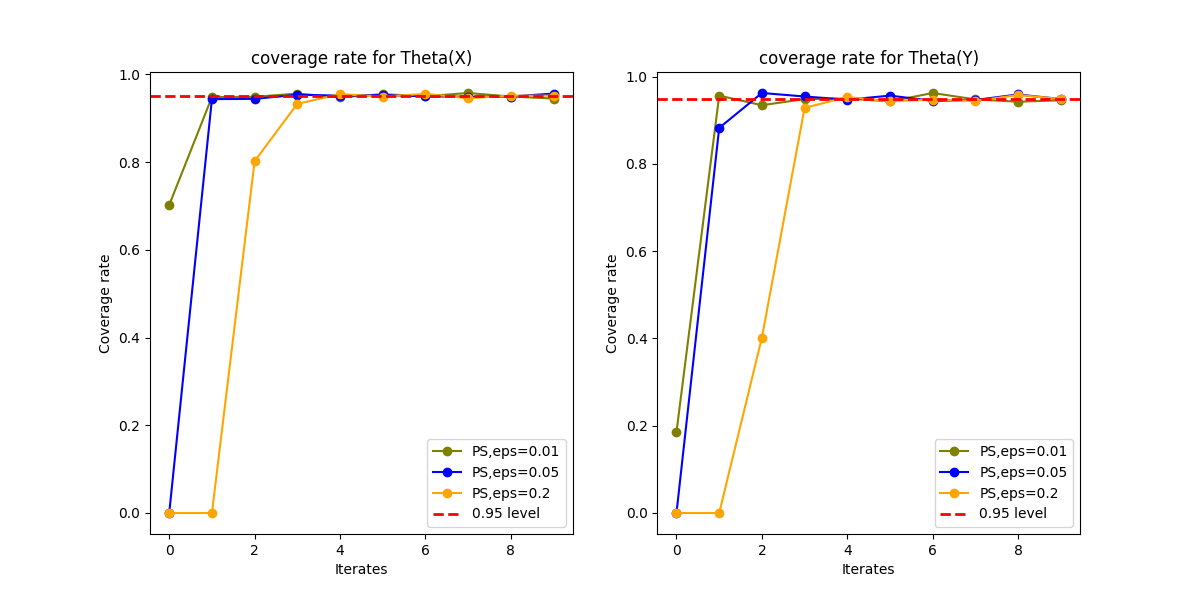}
    \caption{Coverage Rate for two entries of $\theta_{PS}$ vs. Misspecification}
    \label{fig:coverage-rate}
\end{figure}

\subsection{Optimal Points}

\subsubsection{Experiment Details}

\paragraph{Verifying misspecification and smoothness} The true distribution map is 
$$
\mathcal{D}(\theta): b + M_1 *\theta + \epsilon M_2 *\theta^2 + Z_0, \quad Z_0 \sim N(0,\sigma^2I_d),
$$
and the distribution atlas is 
$$
\mathcal{D}_M(\theta): b + M *\theta + Z_0, \quad Z_0 \sim N(0,\sigma^2I_d).
$$
We first prove its misspecification in total variation distance. We can calculate directly that $M^* = M_1$, and since $\mathcal{D}(\theta) \triangleq N(\mu_1,\sigma^2I_d)$ and $\mathcal{D}_{M^*}(\theta) \triangleq N(\mu_2,\sigma^2I_d)$ are Gaussian distributions with the same covariance, their total variation distance follows the inequality:
\begin{equation*}
    TV(\mathcal{D}_{M^*}(\theta),\mathcal{D}(\theta)) \leq \frac{1}{2}\|\Sigma^{-1/2}(\mu_1 - \mu_2)\|_2 = \frac{1}{2\sigma}\|\mu_1 - \mu_2\|_2.
\end{equation*}
Since $\mu_1 - \mu_2 = (M_1 - M^*)\theta + \epsilon M_2\theta^2 = \epsilon M_2\theta^2$, in this experiment, the distance has the upper bound
$$
TV(\mathcal{D}_{M^*}(\theta),\mathcal{D}(\theta)) \leq \frac{\epsilon M_2 \theta^2}{2\sigma}.
$$
Therefore, the distribution map is $\frac{\epsilon M_2 \theta^2}{2\sigma}$-misspecified. 
As for the smoothness in total variation distance, we have the following inequality:
\begin{equation*}
    TV(\mathcal{D}_{M}(\theta),\mathcal{D}_{M'}(\theta)) \leq \frac{1}{2\sigma}\|\mu - \mu'\|_2 = \frac{\theta}{2\sigma}\|M - M'\|_2.
\end{equation*}
Thus the distribution atlas is $\frac{\theta}{2\sigma}$-smoothness.

\paragraph{Optimization details} According to our simulation setting, we can calculate the closed form of the target distributional parameter and the target plug-in optimum that $\beta^* = \beta_1$ and $\theta_{PO}^{\beta^*} = \frac{-b}{\beta^*-1}$. We explain them in detail. 
The true distribution map $\D$ and the distribution atlas $\D_\beta$ for fitting the true map in our performative problem are 
$$
Z \sim \D(\theta) = N(b + \beta_1 \theta + \epsilon \beta_2 \theta^2,\sigma^2) \triangleq N(\mu(\theta),\sigma^2),
$$
$$
Z \sim \mathcal{D}_\beta(\theta) = N(b + \beta \theta,\sigma^2),
$$
and the distribution of $\theta$ for fitting the distribution map is uniform $\theta \sim U(-1,1)$. The true distributional parameter is $\beta^* = \arg\min_{\beta \in \mathcal{B}}\E_{\theta,Z}(Z-\beta\theta)^2$, and we can extend the expectation as follows:
\begin{equation*}
\begin{split}
    \E_{\theta,Z}(Z-\beta\theta)^2 &= \E_{\theta,Z}(Z^2- 2Z\beta\theta + \beta^2\theta^2) \\
    &= \E_{\theta}\left\{\E_{Z \mid \theta}(Z^2- 2Z\beta\theta + \beta^2\theta^2) \right\} \\
    &= \E_{\theta}\left\{\sigma^2 + \mu(\theta)^2 - 2 \beta\theta\mu(\theta) + \beta^2\theta^2\right\} \\
    &= \sigma^2 + b^2 + \frac{2}{3}b\epsilon\beta_2 + \frac{1}{3}(\beta_1-\beta)^2 + \frac{1}{5}\epsilon^2\beta_2^2.
\end{split}
\end{equation*}
Therefore, differentiating the expectation with respect to $\beta$ and equating it to zero yields the true distributional parameter $\beta^* = \beta_1$. Besides, the plug-in optimum is $\theta_{PO}^{\beta^*} = \arg\min_{\beta \in \mathcal{B}}\E_{Z \sim \D_{\beta^*}(\theta)}(Z-\theta)^2$, and similarly we can extend the expectation as follows:
\begin{equation*}
\begin{split}
    \E_{Z \sim \D_{\beta^*}(\theta)}(Z-\theta)^2 &= \E_{Z \sim \D_{\beta^*}(\theta)}(Z^2 - 2Z\theta + \theta^2) \\
    &= \sigma^2 + (b + \beta^* \theta)^2 - 2\theta(b + \beta^* \theta) + \theta^2.
\end{split}
\end{equation*}
Then we take its first derivatives with respect to $\theta$ and equate it to zero
\begin{equation*}
\begin{split}
    & \qquad 2\beta^*(b + \beta^*\theta) - 2b - 4\beta^* \theta + 2\theta\\
    &= 2[(\beta^*)^2 -2\beta^* + 1] \theta + 2(\beta^* -1)b\\
    &= 2(\beta^*-1)^2 \theta + 2(\beta^* -1)b \\
    &= 0.
\end{split}
\end{equation*}
Thus, the true plug-in optimum is $\theta_{PO}^{\beta^*} = \frac{-b}{\beta^*-1}$.

\paragraph{Coverage rate} We compute the coverage rate of the confidence interval for the plug-in optimum $\theta_{PO}^{\beta^*}$ using $1000$ independent experiments. In each experiment, we use $N = 15000$ samples to estimate $\hat{\beta}$, $\tilde{N} = 1000000$ Monte Carlo samples to estimate the integral, and $n = 1000000$ Monte Carlo samples to generate $\hat{\theta}_{PO}^{\hat{\beta}}$. The confidence interval for $\theta_{PO}^{\beta^*}$ is constructed using the estimator $\hat{\theta}_{PO}^{\hat{\beta}}$ and the numerically estimated covariance matrix $\hat{\Sigma}_\theta$, as follows:
\begin{equation*}
    \left[\hat \theta_{PO}^{\hat \beta} - z_{1-\alpha/2} \cdot \sqrt{\frac{\hat{\Sigma}_\theta}{N}}, \hat \theta_{PO}^{\hat \beta} + z_{1-\alpha/2} \cdot \sqrt{\frac{\hat{\Sigma}_\theta}{N}} \right], \quad z_{1-\alpha/2} = \Phi^{-1}(0.975),
\end{equation*}
where $N$ is the sample size used for simulating the plug-in estimator, and $\Phi(\cdot)$ denotes the quantile function of the standard normal distribution. Note that in each experiment, the $N = 15000$ samples are evenly divided across three steps, with $N/3 = 5000$ samples allocated per step, which is sufficiently large for reliable estimation. Furthermore, the ratios of sample sizes $\frac{N}{\tilde{N}} = \frac{N}{n} = 0.015$ are small enough to satisfy the requirements for theoretical covariance.

\end{document}